\documentclass[12pt,a4paper]{article}

\usepackage[utf8]{inputenc}
\usepackage[T1]{fontenc}
\usepackage{amsmath,amssymb,amsthm}
\usepackage{hyperref}
\usepackage{geometry}
\usepackage{setspace}
\usepackage{natbib}
\newtheorem{theorem}{Theorem}[section]
\newtheorem{definition}{Definition}[section]
\newtheorem{axiom}{Axiom}[section]
\newtheorem{corollary}{Corollary}[theorem]
\theoremstyle{remark}

\begin{document}

\title{Predictive Set Theory: A Generative Framework for Cognitive Architecture with Operationalized Core Mechanisms}
\author{YiYang Yu\thanks{Email: sitlas182056@gmail.com}}
\date{}
\maketitle

\begin{abstract}
Predictive processing theories portray the brain as a hierarchical prediction engine that minimizes prediction error, yet they lack operational definitions for the structure of a ``prediction,'' the standardized response to a prediction error, and the mechanism that maintains consistency across successive updates. Bayesian cognitive science attempts to subsume all uncertainty under probabilistic belief updating, but it presupposes a closed hypothesis space and provides no generative account of how the objects over which probabilities are distributed become discrete, identifiable referents in the first place. This paper introduces Predictive Set Theory (PST), a formal generative framework that reconstructs cognitive architecture from first principles. PST anchors cognition in a minimal set of operations---a sensor formalized as an identity function, set-theoretic state refresh, and three fundamental forms of reference chains (reference, counter-reference, and semi-reference)---and rigorously derives core cognitive functions including state sequences, demand, comparison, efficiency, and finite-horizon probabilistic planning. Rather than modeling neural mechanisms, PST constitutes a design specification for any system that must maintain internal consistency while acting under incomplete information and irreversible risk. The framework offers novel resolutions to classical problems such as Russell's paradox, the cognitive status of G\"{o}delian incompleteness, the grounding of negative feedback, and the comprehension of film editing. The primary purpose of this paper is to establish, through the public academic record, the originality and completeness of the Predictive Set Theory framework.
\end{abstract}

\section{Introduction}

A fundamental challenge in cognitive science is to explain how a bounded system, confronted with incomplete information and irreversible risk, can safely accumulate knowledge and continuously revise itself. Predictive processing theories \cite{Friston2010,Clark2013,Hohwy2013} have offered a powerful answer: the brain is a hierarchical prediction engine that continuously generates top-down predictions and minimizes prediction error through bottom-up signaling. Despite its explanatory appeal, the framework faces a structural difficulty when its core concepts are examined at an operational level. How should a ``prediction'' be formalized---as a probability distribution, a vector, or a symbolic proposition? When a prediction fails, what standardized operation should the system perform---updating the model, marking the unknown, or triggering avoidance? And how is the ``consistency'' of the internal model maintained across successive updates so that knowledge does not dissolve into noise? These questions remain largely unformalized within the predictive processing literature.

Bayesian cognitive science \cite{Knill2004,Doya2007} attempts to subsume all uncertainty under probabilistic belief updating, but its successful operation depends on a frequently overlooked precondition: the hypothesis space must be pre-enumerated and closed. It requires not only that the system already knows all possible hypotheses, but also that the objects involved have already been standardized as discrete, identifiable units to which probabilities can be assigned. However, the more fundamental question of how an object becomes a referable discrete unit in the first place---how the cognitive system itself carves boundaries out of the continuous sensory stream and maintains their identity---receives no generative account within the Bayesian framework. It is precisely at this point that our theory intervenes: before probabilities can be assigned, the discreteness, identity, and referential relations of objects must be established, and this is the operational foundation that set theory and the formalization of reference chains provide.

This paper introduces Predictive Set Theory (PST), a formal generative framework that reconstructs cognitive architecture from first principles. PST is neither an empirical model of neural processing nor a direct competitor to specific Bayesian algorithms; it is a \textit{generative design specification} for a cognitive agent---a set of axioms, definitions, and derived operations that collectively answer the question: What minimal architecture must a system possess to perceive, predict, decide, learn, and self-correct while maintaining internal consistency?

In answering this question, PST accomplishes two foundational tasks.

First, it rigorously formalizes the core concepts of the predictive processing framework. A ``prediction'' is redefined as a \textbf{sequence}---a chain of discrete state-refresh events. Based on the current set of known states, the system constructs candidate subsequent states by traversing transformation templates within its experiential network. Prediction is thus no longer an amorphous probability distribution but a concretely specifiable inferential path that can be expressed as a sequence, verified against future sensory input, and corrected. A ``prediction error'' triggers one of three standardized computational responses, which correspond to three fundamental forms of reference chains: \textbf{reference} (establishing or confirming a positive directional link between objects), \textbf{counter-reference} (declaring the impossibility of a specific referential path, thereby blocking a particular predictive direction), and \textbf{semi-reference} (marking a position where a referential intention has been initiated but the target content has not yet been filled by perception---the system's acknowledgment that ``there is something unknown here''). This triadic system of reference chains constitutes a minimally complete set of cognitive correction operations.

Second, the entire derivation of PST proceeds from a single architectural prerequisite: the \textbf{consistency requirement}. Consistency is not an externally imposed norm but a necessary condition for a cognitive system to generate any coherent decision: an internally contradictory knowledge network would deprive the action-selection function of a determinate output. From this starting point, the framework unfolds in a strictly generative order. Through a convergence analysis of infinite reference chains, it is first proven that any reference chain must converge to a self-referential fixed point in finitely many steps, thereby grounding identity at the most fundamental level---every operable object must be self-identical by default. The sensor is then formalized as an identity function \(S(i)=i\), becoming the sole independent variable in the cognitive architecture. The action-planning function is defined as a conditioned mapping from known states to behavior and is proven not to be an independent variable. On this basis, the fundamental asymmetry between known and unknown states is rigorously demonstrated: known states cannot derive unknown states through internal computation, while the sensor is the only channel through which unknown states are converted into known states. The operation of state refresh---whenever the current sensor output differs from the register object, the system expands the known-state set and appends the previous state to the sequence---is strictly defined. The resulting sequence of known-state sets exhibits a set-inclusion structure that is strictly isomorphic to von Neumann's construction of the ordinals \cite{vonNeumann1923}, thereby generating a before-after order without any external time parameter. Furthermore, any deterministic action-planning function is shown to necessarily embed a minimal demand constraint that takes the function itself as the optimality standard---demand is not an add-on module but a mathematical inevitability of the function's existence. When demand operates in sequence, the dimension of efficiency (shorter sequences are better) is inevitably introduced, causing the choice set to expand and making comparison an unavoidable cognitive operation. Finally, feedback is quantified as an efficiency value based on sequence distance and demand weights, and short-sighted versus far-sighted behaviors are unified under a single parametric model of efficiency polylines. Finite-horizon probabilistic planning then integrates all the above concepts into a complete decision model. The full derivation chain is presented in the Appendix.

Because PST is a design specification rather than an empirical model, its validity rests on internal consistency, conceptual precision, and the explanatory unification it achieves. The framework yields novel resolutions to several classical problems. Russell's paradox is shown to correspond to a non-convergent sequence oscillation, which the system can recognize and suspend by detecting the non-convergence of the reference chain, without logical collapse. G\"{o}delian incompleteness is reinterpreted as the formal boundary that necessarily arises when a counter-reference chain takes itself as its own target---the system can formulate such propositions but cannot complete their truth-value determination without destroying consistency. The quantitative grounding of negative feedback is traced to the combination of sequence distance with demand weights that originate in somatic signals. The cognitive comprehension of film editing---a phenomenon that challenges classical continuity theories---is explained as the viewer's active deployment of reference chains and sequence-verification processes to maintain cross-shot object identity, rather than a passive inheritance of physical continuity.

The primary purpose of this paper is to establish, through the public and citable academic record, the originality and completeness of the Predictive Set Theory framework. Its secondary purpose is to provide a rigorous theoretical foundation for future computational implementations and for the design of artificial cognitive systems with intrinsic safety constraints.

The paper is organized as follows. Chapter~1 establishes the consistency requirement as the architectural prerequisite and grounds identity through a convergence analysis of infinite reference chains. Chapter~2 provides the axiomatic definition of the cognitive agent, strictly distinguishes the sensor from the action-planning function, and proves the asymmetry theorem of known and unknown states. Chapter~3 derives the operational mechanism of state refresh and the set-theoretic construction of sequences. Chapter~4 proves that any deterministic action-planning function necessarily embeds a minimal demand constraint and that the sequential operation of demand inevitably forces the emergence of comparison. Chapter~5 quantifies feedback through efficiency polylines and finite-horizon probabilistic planning, and unifies short-sighted and far-sighted behaviors under a single parametric model. The Appendix presents the complete derivation chain.

\section{Glossary of Core Terms}

\begin{description}
\item[Action-planning function] \(f(K_t) = A_t\). The decision core of the cognitive system, a conditioned mapping from the current set of known states \(K_t\) to a behavioral plan \(A_t\). Its output must strictly depend on known states; it cannot be an independent variable (see Sensor).

\item[Alignment (process)] The sequence-level operation by which two distinct referential nodes are merged so that they share the same sequence position, thereby being treated as a single object. Alignment is not a primitive tag but a functional product of reference chains and sequence operations.

\item[Comparison] The cognitive operation of ordering candidate actions or sequences by their efficiency values when the choice set contains two or more distinguishable elements. Comparison is triggered inevitably when demand operates in sequence and efficiency differences emerge.

\item[Consistency requirement] The architectural prerequisite that a cognitive system must maintain an internally non-contradictory knowledge network. Without consistency, the action-planning function cannot produce a determinate output.

\item[Counter-reference (chain)] A referential relation of the form \(X \not\to Y\) or \(X \mathrel{\text{excludes}} Y\), which declares that a particular referential path is impossible or forbidden. It is one of the three basic forms of reference chains.

\item[Demand] A constraint on the output of the action-planning function, formally expressed as an evaluation function \(\Phi(K_t, A)\) whose maximization yields the selected action. Every deterministic function embeds a minimal demand constraint that takes the function itself as the optimality standard.

\item[Efficiency] The dimension along which shorter sequences are strictly preferred to longer ones when both satisfy the same final demand. Efficiency is quantified as the rate of change of demand value with respect to sequence distance.

\item[Efficiency polyline] A function \(f_d(L)\) mapping sequence distance \(L\) to an efficiency multiplier in \([0,1]\). It satisfies \(f_d(0)=1\), monotonic decrease, and asymptotic approach to zero. Its shape determines whether behavior is short-sighted or far-sighted.

\item[Feedback efficiency value] The quantified evaluation of a state \(s\) with respect to a demand state \(d\): \(\eta(s,d) = w_d \cdot f_d(L(s,d))\), where \(w_d\) is the demand weight and \(L(s,d)\) the shortest sequence distance.

\item[Finite-horizon probabilistic planning] The decision procedure that maximizes expected cumulative efficiency over a bounded number of future state-refresh steps, using the recursive formula \(V_h(s) = \max_a \sum_{s'} P(s'|s,a)[\eta(s') + V_{h-1}(s')]\).

\item[Known-state set] \(K_t\). The cumulative set of all internal objects output by the sensor from the initial moment up to time \(t\), constructed recursively as \(K_{t+1} = K_t \cup \{i_{t+1}\}\).

\item[Minimal demand constraint] The demand function \(F_f(K,A) = -\|A - f(K)\|^2\) that takes the pure-computation output \(f(K)\) as its unique maximum. It is the mathematically inevitable constraint embedded in any deterministic function.

\item[Reference (chain)] The basic cognitive operation by which a set node points to another node, encapsulating the content of ``what this object is.'' Formally, a mapping \(R: \mathcal{O} \to \mathcal{O}\).

\item[Register object] \(E_t\). The internal storage unit that holds the most recent sensor output, serving as the baseline for detecting change via the comparison operation.

\item[Semi-reference (chain)] An incomplete referential link that marks a position where a referential intention has been initiated but the target has not yet been filled by sensory input. It expresses ``there is something unknown here'' without presupposing a closed hypothesis space.

\item[Sensor] \(S(i) = i\). The sole unconditioned variable in the cognitive architecture, formalized as an identity function. It unconditionally outputs internal objects, constituting the only channel through which unknown states are converted into known states.

\item[Sequence] An ordered accumulation of state-refresh events. The sequence of known-state sets \(K_0 \subset K_1 \subset \cdots\) is strictly isomorphic to von Neumann's construction of the ordinals, generating temporal order without an external time parameter.

\item[State refresh] The atomic operation triggered when the current sensor output differs from the register object. It consists of appending the previous known-state set to the sequence, expanding the known-state set with the new internal object, and updating the register object.

\item[Unknown state (relative)] The set difference \(U(K_b \mid K_a) = K_b \setminus K_a\) between two known-state sets. A state is ``unknown'' only relative to a specific earlier known-state set; it has no absolute unknown property.
\end{description}

\newpage

% ==========================================
% CHAPTER 1: THE CONSISTENCY REQUIREMENT
% ==========================================

\section{The Consistency Requirement as an Architectural Prerequisite}

This chapter establishes the foundational architectural prerequisite of Predictive Set Theory. It proceeds in two parts. Section~1.1 argues from the possibility of decision-making that the ``consistency requirement'' is a necessary condition for any cognitive system capable of generating effective behavior. Section~1.2 anchors this requirement at the most fundamental computational level through a convergence analysis of infinite reference chains, proving that a reference chain must default to self-reference as an absolute precondition for any cognitive operation to complete---thereby establishing identity as the irreducible starting point of cognition.

\subsection{Why a Cognitive System Must Maintain Consistency}

Before any specific cognitive operation is introduced, a more fundamental question must be answered: why must a cognitive system maintain internal consistency? This section argues that consistency is neither an accidental virtue nor an externally imposed normative requirement, but an architectural prerequisite that any cognitive system capable of generating coherent decisions must satisfy.

\subsubsection{The Formal Condition of Decision: From Contradiction to Paralysis}

Consider a cognitive system at a given moment facing a set of behavioral options \(A = \{a_1, a_2, \ldots, a_n\}\). The system must select an action \(a^* \in A\) based on its internal knowledge network \(K\). The core of this process is a value-evaluation function \(V_K: A \to \mathbb{R}\), where \(V_K(a)\) represents the system's expected valuation of action \(a\) based on knowledge network \(K\).

A system is said to be capable of effective decision-making if and only if there exists a unique optimal action:
\[
\exists! a^* \in A: V_K(a^*) = \max_{a \in A} V_K(a)
\]
When multiple actions are tied as optimal, the system may employ an arbitrary tie-breaking rule to produce a determinate output without affecting the core argument.

Now suppose the knowledge network \(K\) contains an irreducible contradiction. Specifically, there exists a proposition \(P\) such that the system simultaneously holds both \(P\) and \(\neg P\) as valid beliefs within \(K\). The effect on the value-evaluation function is direct: for any action \(a \in A\) whose valuation depends on \(P\) or \(\neg P\), the system simultaneously faces \(V_K(a) = v_1\) (under the premise that \(P\) is true) and \(V_K(a) = v_2\) (under the premise that \(\neg P\) is true), where \(v_1\) and \(v_2\) need not be equal. Because the system cannot adjudicate between \(P\) and \(\neg P\) in a principled manner---both enjoy equal cognitive standing in \(K\)---it cannot determine the true value of \(V_K(a)\). The value function fractures at the point of contradiction, losing its capacity to provide a unified basis for action selection.

More fundamentally, this fracture is not local. If the contradiction involves causal beliefs about action consequences---for example, ``eating this mushroom causes death'' and ``eating this mushroom does not cause death'' both residing in \(K\)---the value-evaluation function loses all determinate output on safety-critical options. The system falls into a functional paralysis---not a physical shutdown, but an inability to generate any action that can be self-justified as optimal. This is the generative root of ``prediction oscillation'': mutually contradictory beliefs drive mutually contradictory behavioral predictions, depriving the system of behavioral direction.

\subsubsection{The Architectural Status of the Consistency Requirement}

Based on the above analysis, the ``consistency requirement'' acquires a clear architectural definition: \textbf{The consistency requirement is the architectural condition that a cognitive system must satisfy, requiring that its knowledge network \(K\) contain no irreducible contradictions at any moment---or more precisely, that any detected contradiction be processed and resolved by some cognitive operation (such as alignment processes, counter-reference chains, or semi-reference chains), thereby ensuring that the decision function \(f(K)\) can produce a determinate output.}

Three key points must be emphasized.

\textbf{First, the consistency requirement is a formal condition, not a content constraint.} It does not prescribe what the system should believe, only that it cannot simultaneously believe \(P\) and \(\neg P\). What specific content the system believes depends on its perceptual experience, somatic feedback, and social learning, not on the consistency requirement itself. In this sense, the consistency requirement is analogous to the law of non-contradiction in logic---it does not tell thought where to go, but only marks the boundary that thought must not cross.

\textbf{Second, the consistency requirement is a generative starting point, not a derived conclusion.} Within the generative construction of Predictive Set Theory, the consistency requirement does not need to be proven---it is the posited starting point. The only way to ``prove'' why consistency is necessary is to demonstrate that without it, determinate decision-making is impossible, which is precisely what the preceding analysis has accomplished. In logical status, the consistency requirement is comparable to the parallel postulate in geometry or the axiom of extensionality in set theory---it is the initial condition that enables the entire theoretical edifice to unfold.

\textbf{Third, the consistency requirement is not equivalent to the positive-negative feedback principle, but is its condition of possibility.} The positive-negative feedback principle---``pursue positive feedback, avoid negative feedback''---is a content-based decision criterion that guides the system, on the basis of consistency, to select \textit{what} is beneficial to itself. But the absolute prerequisite for this criterion to function is that the system's internal causal knowledge about ``what action brings what feedback'' is consistent and coherent. If the knowledge network contradicts itself, the positive-negative feedback principle loses its viability as a decision criterion---because, amidst contradictory information, it is impossible to determine which option is genuinely ``better.'' Thus, the consistency requirement is logically prior to the positive-negative feedback principle and constitutes the architectural foundation upon which the latter can operate.

\subsubsection{Distinction from Classical Positions}

Positioning consistency as an architectural prerequisite for cognition resonates with several historical positions in philosophy, yet also fundamentally differs from them.

In Kantian philosophy, the unity of transcendental apperception---the requirement that the ``I think'' must be capable of accompanying all my representations---is established as the supreme condition of the possibility of experience. Kant argued that if representations could not be unified under a single self-consciousness, they ``could not be my representations'' and thus could not constitute any intelligible cognitive object. The consistency requirement occupies an analogous structural position in our framework: it is not an empirically discovered fact, but a transcendental condition that makes cognitive experience possible. However, our consistency requirement is strictly confined to the level of cognitive operations---it is a formal condition for ``decisions to be generable,'' not a metaphysical thesis about the ``subject'' itself.

In classical cognitive science and artificial intelligence, consistency is typically treated as an external evaluative standard---a good theory or a good cognitive system \textit{ought} to be logically consistent. This ``ought'' implicitly presupposes the perspective of an external observer. Our position is fundamentally different: consistency is not an external evaluation but an internal necessary condition. An inconsistent system is simply not a functioning decision system at all. Just as a machine with interlocking gears that jam cannot be called an ``engine,'' a system with mutually exclusive internal beliefs cannot be called a ``cognitive agent.''

In predictive processing theory, the brain is understood as a hierarchical prediction engine whose fundamental goal is to ``minimize prediction error'' \cite{Friston2010,Clark2013}. However, predictive processing theory has not explicitly elevated ``consistency'' to the status of an architectural prerequisite independent of ``error minimization.'' If prediction error minimization can proceed under local inconsistency---for example, two distinct high-level prediction modules each minimizing their own prediction error while producing mutually contradictory outputs---then the system as a whole may still produce chaotic behavior. Our position is that the consistency requirement is logically prior to prediction error minimization, providing the latter with a unified, non-contradictory model of the perceived world within which prediction errors can be coherently evaluated and responded to.

This distinction can be presented with greater precision through direct contrast with the predictive processing framework. In Friston's (2010) free-energy principle, the brain is understood as maintaining coupling with the environment through the minimization of variational free energy, and ``consistency'' is merely an implicit constraint operating during the hierarchical transmission of prediction errors---it manifests as the degree of match between high-level predictions and low-level sensations, not as an independent architectural principle. Clark (2013) pushes predictive processing to its extreme, advocating ``perception as inference''---the thesis that all perceptual activity is essentially the interplay between top-down predictions and bottom-up errors. However, within this framework, if two high-level predictive modules each minimize their own prediction errors but produce mutually contradictory outputs---for instance, one module predicting leftward motion and another predicting rightward motion---the system as a whole possesses no mechanism to adjudicate this contradiction, because ``error minimization'' per se does not include any requirement of global consistency. Hohwy (2013) explicitly acknowledges this tension, noting that predictive processing requires a ``unified self'' to integrate cross-modal predictions, yet provides no formal definition of this integrative mechanism.

Predictive Set Theory remedies precisely this gap by elevating the consistency requirement to a more fundamental architectural prerequisite than prediction error minimization. In PST, contradictions are not ``minimized''---they are not permitted to enter the knowledge network in the first place. Alignment processes (forcibly merging or refusing to merge two potentially contradictory referential nodes) and counter-reference chains (blocking referential paths that lead to contradiction) operate continuously within the tag-determination process, ensuring that the action-planning function always operates on a coherent set of known states. This means that consistency is not an asymptotic state achieved through optimization, but a transcendental condition that must be satisfied before any optimization can occur. Just as Kant's (1781/1787) transcendental unity of apperception provides the formal condition for the possibility of experience, the consistency requirement provides the formal condition for prediction error minimization---without the former, the latter loses its operational stage.

\subsection{Grounding Consistency at the Computational Level: Convergence Analysis of Infinite Reference Chains}

Section~1.1 established the ``consistency requirement'' as the architectural prerequisite of any cognitive system. However, this establishment remains at an abstract level---it demonstrates \textit{why} consistency is necessary, but has not yet provided a precise mathematical anchor for consistency at the most fundamental computational level. This section accomplishes that task.

The central question is: how is consistency formally grounded at the lowest level of cognitive computation? The answer lies in the convergence analysis of the ``reference operation'' and its infinite iteration sequence. We shall demonstrate that the cognitive system's reference operation on any object---the process of asking ``what is this?''---must converge to a self-referential fixed point in finitely many steps; otherwise, the cognitive computation can never be completed. This establishes the generative principle that ``reference chains default to self-reference,'' which in turn is the most fundamental formal expression of ``identity''---\(A\) is \(A\)---within the cognitive architecture.

\subsubsection{Formal Definition of the Reference Operation}

In Predictive Set Theory, ``reference'' is one of the most fundamental operations of cognition. When the system confronts an object, it must answer a primordial question: ``What is this object?'' This interrogation is formalized internally as a reference operation.

\begin{definition}[Reference Operation]
Let there exist within the cognitive system a reference operation function:
\[
R: \mathcal{O} \to \mathcal{O}
\]
where \(\mathcal{O}\) is the set of all objects operable by the cognitive system. For any cognitive object \(x \in \mathcal{O}\), \(R(x)\) denotes ``the object to which \(x\) refers''---that is, the system's answer to the question ``what is \(x\)?''

It must be emphasized that this definition does not presuppose whether the referent is \(x\) itself or some other object. \(R(x)\) may be \(x\) itself (self-reference), or another object distinct from \(x\) (other-reference). \(R\) is the fundamental operation invoked whenever the cognitive system attempts to understand any object---it is the generative starting point of all cognitive activity.
\end{definition}

\begin{definition}[Reference Chain]
For any initial object \(x_0 \in \mathcal{O}\), its reference chain is defined as the following infinite sequence:
\[
x_0 \mapsto R(x_0) \mapsto R(R(x_0)) \mapsto R(R(R(x_0))) \mapsto \cdots
\]
Denoting \(R^0(x) = x\) and \(R^{n+1}(x) = R(R^n(x))\), the reference chain can be formally expressed as:
\[
\mathcal{C}(x_0) = \{R^n(x_0)\}_{n=0}^{\infty}
\]

The reference chain depicts the cognitive system's process of repeatedly interrogating the object \(x_0\): the system first knows \(x_0\), then asks ``what is \(x_0\),'' obtaining the answer \(R(x_0)\); it then asks ``what is \(R(x_0)\),'' obtaining the answer \(R(R(x_0))\); and so on. In principle, this interrogation can proceed indefinitely---each answer produces a new object, which itself can become the starting point for the next interrogation.
\end{definition}

\begin{definition}[Convergence of a Reference Chain]
A reference chain \(\mathcal{C}(x_0)\) is said to be \textbf{convergent} if and only if there exists some natural number \(N \in \mathbb{N}\) such that:
\[
R^{N+1}(x_0) = R^N(x_0)
\]
That is, the mapping reaches a fixed point in finitely many steps. The fixed point \(x^* = R^N(x_0)\) satisfies \(R(x^*) = x^*\). At this point, the answer to ``what is \(x^*\)'' is \(x^*\) itself---the interrogation self-closes at \(x^*\) and produces no new object.

If a reference chain does not converge, it is said to be \textbf{divergent}. Divergence has two fundamental modes, to be analyzed in detail below.
\end{definition}

\subsubsection{Necessary and Sufficient Condition for Convergence}

\begin{theorem}[Necessary and Sufficient Condition for Reference Chain Convergence]\label{thm:convergence}
A cognitive system's reference chain \(\mathcal{C}(x)\) for any object \(x \in \mathcal{O}\) converges if and only if there exists some \(N \in \mathbb{N}\) such that \(R^N(x)\) is a self-referential object, i.e., satisfies:
\[
R(R^N(x)) = R^N(x)
\]
\end{theorem}

\begin{proof}
This follows directly from Definition~1.3, but a bidirectional proof is provided for completeness.

(\(\Rightarrow\) Necessity) If the reference chain converges, then by Definition~1.3 there exists \(N\) such that \(R^{N+1}(x) = R^N(x)\). Let \(x^* = R^N(x)\). Substituting: the left side is \(R^{N+1}(x) = R(R^N(x)) = R(x^*)\), and the right side is \(R^N(x) = x^*\). Hence \(R(x^*) = x^*\). Thus \(x^*\) is a self-referential object. This shows that a convergent reference chain necessarily terminates at a self-referential fixed point.

(\(\Leftarrow\) Sufficiency) If there exists \(N\) such that \(R^N(x)\) is self-referential, i.e., \(R(R^N(x)) = R^N(x)\), then \(R^{N+1}(x) = R^N(x)\). By Definition~1.3, the reference chain converges at step \(N\).
\end{proof}

Theorem~\ref{thm:convergence} reveals a deep structural property of the reference operation: the convergence of a reference chain is strictly equivalent to the emergence of a self-referential object. There is no mode of convergence in which a reference chain terminates at a non-self-referential object---because if the terminating object \(y\) were not self-referential, then \(R(y) \neq y\), leaving the system with room to further ask ``what is \(y\),'' meaning the reference chain would not yet be truly closed. The only possible mode of convergence is that at some point, the interrogation is answered by itself---the object points to itself, and the loop closes.

\subsubsection{Analysis of Divergence Modes}

To comprehensively assess the consequences of non-convergence, we analyze the two fundamental modes of divergence.

\paragraph{(i) Infinitely Non-Repetitive Divergence}

\begin{definition}[Infinitely Non-Repetitive Divergence]
A reference chain \(\mathcal{C}(x_0)\) is \textbf{infinitely non-repetitively divergent} if and only if for all \(m, n \in \mathbb{N}\) with \(m \neq n\), we have \(R^m(x_0) \neq R^n(x_0)\). That is, every term in the sequence differs from all previous terms; the reference chain extends through infinitely many distinct objects, never repeating, never terminating.
\end{definition}

\begin{theorem}[Meaninglessness of Infinitely Non-Repetitive Divergence]\label{thm:infinite}
If a reference chain \(\mathcal{C}(x_0)\) is infinitely non-repetitively divergent, then the cognitive interrogation ``what is \(x_0\)'' can never be completed. Specifically, for any \(n \in \mathbb{N}\), the answer ``\(x_0\) is \(R^n(x_0)\)'' merely postpones the question to the next level without ever providing a definitive, final answer. The object \(x_0\) thus cannot be established as a stably retrievable cognitive unit.
\end{theorem}

\begin{proof}
Under infinitely non-repetitive divergence, for any \(n\), \(R^n(x_0)\) is an object distinct from all preceding terms. When the system obtains the answer ``\(x_0\) is \(R(x_0)\),'' it immediately faces the follow-up question ``then what is \(R(x_0)\)?'' Upon receiving the answer ``\(R(x_0)\) is \(R^2(x_0)\),'' it faces ``what is \(R^2(x_0)\)?''---and so on, ad infinitum.

The crucial point is this: because each term in the sequence is entirely new (distinct from all previous terms), there exists no step \(n\) at which the system can say ``I now know---\(R^n(x_0)\) is itself, no further questioning is needed.'' Every answer only generates a new question. Hence, cognition of \(x_0\)---a stable understanding of ``what, ultimately, \(x_0\) is''---remains perpetually incomplete. The object \(x_0\) has no fixed coordinates in cognitive space; it is merely an arrow forever pointing elsewhere.
\end{proof}

In mathematical analysis, this situation is analogous to a sequence that diverges in infinite space---it converges to no limit point whatsoever. Just as a divergent sequence cannot be regarded as ``equaling'' any determinate value, an infinitely non-repetitively divergent reference chain cannot be regarded as ``defining'' any determinate cognitive object. The existence of a cognitive object depends on the convergence of its reference chain.

\paragraph{(ii) Cyclic Divergence}

\begin{definition}[Cyclic Divergence]
A reference chain \(\mathcal{C}(x_0)\) is \textbf{cyclically divergent} if and only if there exist \(N \in \mathbb{N}\) and \(k \geq 2\) such that for all \(m \geq 0\), \(R^{N+m+k}(x_0) = R^{N+m}(x_0)\), but the cycle contains no fixed point---i.e., for all \(0 \leq i < k\), \(R(R^{N+i}(x_0)) \neq R^{N+i}(x_0)\).

In other words, after an initial finite segment of \(N\) steps, the reference chain enters a periodic cycle of length \(k\), but none of the objects in the cycle refers to itself---each refers to the next object in the cycle. The system spins infinitely within the cycle, never touching a self-closing point.
\end{definition}

\begin{theorem}[Cyclic Divergence and the Analogy to Non-Convergent Infinite Series]\label{thm:cyclic}
If a reference chain enters a periodic cycle without a fixed point, then the reference operation is functionally equivalent to a non-convergent infinite series---its partial results oscillate among multiple values without converging to any determinate limit.
\end{theorem}

\begin{proof}
Consider a cognitive judgment function \(J: \mathcal{O} \to \{0, 1\}\) defined on reference chain states, where \(J(y) = 1\) if and only if the system judges that ``\(y\) is the final referent of \(x_0\).'' Under cyclic divergence, for each state within the cycle \(y_0, y_1, \ldots, y_{k-1}\) (where \(y_i = R^{N+i}(x_0)\) and \(R(y_i) = y_{i+1 \bmod k}\)), the system, as it traverses the cycle, will sequentially judge:
\[
J(y_0), J(y_1), \ldots, J(y_{k-1}), J(y_0), J(y_1), \ldots
\]
This sequence oscillates among \(k\) values infinitely without converging to any stable truth value. From any state within the cycle, the system cannot arrive at a self-affirming conclusion.

This situation is structurally identical, in mathematical terms, to a conditionally convergent but not absolutely convergent infinite series. The classic example is Grandi's series:
\[
\sum_{n=0}^{\infty} (-1)^n = 1 - 1 + 1 - 1 + 1 - 1 + \cdots
\]
whose partial sum sequence is:
\[
s_0 = 1, s_1 = 0, s_2 = 1, s_3 = 0, \ldots
\]
This partial sum sequence oscillates infinitely between \(1\) and \(0\), converging to no real limit whatsoever. Just as Grandi's series does not exist as a real-number ``sum,'' a cyclically divergent reference chain does not exist as a determinate cognitive ``final referent.''

It is worth noting that in mathematics, one may assign a ``value'' to Grandi's series by introducing special summation methods (such as Ces\`{a}ro summation, yielding \(\frac{1}{2}\)). But such an assignment is not the result of the series itself converging; it is the imposition of an external rule that alters the definition of ``summation.'' Analogously, at the cognitive level, a system may \textit{decide} to treat some object within the cycle as ``final'' (by forcibly merging two objects in the cycle via alignment, thereby breaking the cycle), but such a decision is a cognitive operation external to pure computation---it belongs to the jurisdiction of the tag-determination process, not to the reference operation itself. Within the internal logic of the reference operation, cyclic divergence means the computation can never spontaneously complete.
\end{proof}

\begin{corollary}
Infinitely non-repetitive divergence and cyclic divergence are the only two modes of non-convergence. Proof: Consider the reference chain \(\mathcal{C}(x_0)\). If the sequence does not converge, then either it contains infinitely many mutually distinct terms (infinitely non-repetitive divergence), or it must exhibit repetition after finitely many terms (by the pigeonhole principle or its cognitive analogue---under finite cognitive resources, an infinite sequence must reuse finite states). If repetition occurs, let \(R^a(x_0) = R^b(x_0)\) with \(a < b\) being the earliest such pair. If this cycle contains a fixed point, the sequence converges; if it contains no fixed point, it is cyclically divergent.
\end{corollary}

\subsubsection{Reference Chain Default Self-Reference as a Generative Necessity}

Integrating the above analyses, we can now establish one of the most fundamental architectural axioms of Predictive Set Theory.

\begin{axiom}[Reference Chain Default Self-Reference]\label{axiom:selfref}
A cognitive system must presuppose that for any cognitive object \(x \in \mathcal{O}\) capable of being processed and manipulated by the system, its reference chain converges to a self-referential fixed point in finitely many steps. That is, by default there exists some \(N \in \mathbb{N}\) such that:
\[
R(R^N(x)) = R^N(x)
\]
\end{axiom}

\textbf{Justification:} The generative necessity of this axiom is jointly provided by Theorems~\ref{thm:convergence}--\ref{thm:cyclic}.

If the reference chain does not converge to a self-referential fixed point, then it is either (a) infinitely non-repetitively divergent, or (b) cyclically divergent. In case (a), Theorem~\ref{thm:infinite} proves that the cognitive interrogation can never be completed and the object cannot be established as a stable cognitive unit. In case (b), Theorem~\ref{thm:cyclic} proves that the reference results oscillate among multiple values infinitely, structurally equivalent to a non-convergent infinite series, with no determinate limit.

In both cases, the cognitive computation fails to produce a determinate conclusion---it either extends infinitely or oscillates infinitely. For a cognitive system to be operable, it must be able to \textit{complete} the reference operation on objects---that is, it must be able to provide a final, no-longer-pointing-elsewhere answer to the question ``what is \(x\).'' Without this completion, \(x\) cannot enter the operational space of the action-planning function, because the latter requires determinate inputs.

Thus, ``reference chain default self-reference'' is not an externally imported metaphysical preference, but rather \textbf{the internal formal condition for cognitive computation to be completable}. Just as in mathematical analysis one must first ascertain that an infinite series converges to some limit before using its value---a divergent series cannot be assigned a determinate value---so too in cognitive computation, to use a cognitive object, the system must by default assume that its reference chain converges to a self-referential fixed point---a divergent reference chain cannot define a determinate cognitive object.

\textbf{The pivotal status of this axiom must be understood precisely:} It is not asserting that ``all objects in the world are self-identical''---that would be a metaphysical assertion about realism. Rather, it is asserting: ``For a cognitive system to be capable of producing determinate, retrievable cognitive objects, it must, at the architectural level, set self-reference as the default terminus of reference chains.'' Identity---i.e., \(x = x\)---thereby acquires its generative necessity. It is not an empirical discovery, nor a logical consequence, but the transcendental condition that makes the basic fact ``there are objects available for cognition'' possible in the first place.

\begin{corollary}[Generative Origin of Identity]
Every object \(x\) capable of being processed and manipulated by a cognitive system is, by default, endowed with identity at the architectural level---that is, \(x\) is \(x\). This identity is not a perceived empirical property, not a reasoned conclusion, and not a socially constructed convention. It is the direct expression, within the cognitive architecture, of the convergence condition of the reference operation. At the lowest level, it is the identity operation \(S(i) = i\) executed by the sensor hardware---a point to be rigorously demonstrated in Chapter~2.
\end{corollary}

\subsubsection{The Relation between Consistency and Identity}

At this point, we can clarify the relationship between the consistency requirement (Section~1.1) and reference chain default self-reference (Section~1.2).

Reference chain default self-reference anchors the consistency requirement at the computational level. The consistency requirement demands that the system's knowledge network contain no irreducible contradictions; reference chain default self-reference demands that every cognitive object be established as self-identical at the most fundamental level---it guarantees that ``\(x = x\)'' is the most basic unit of consistency, beyond questioning. If even ``\(x = x\)'' could be negated---that is, if a counter-reference chain were to point to itself---then consistency could gain no foothold at any level of cognition, because the most fundamental identity itself would have collapsed. This is precisely the generative root of the ``consistency collapse'' discussed previously.

Thus, reference chain default self-reference and the consistency requirement form a hierarchical foundational relationship:

\begin{itemize}
\item \textbf{Consistency requirement:} A global condition at the architectural level---the knowledge network must be coherent for decision-making to be generable.
\item \textbf{Reference chain default self-reference:} A local anchoring at the computational level---every cognitive object must be self-referential for the interrogation to close. The global demand for consistency is expressed at the lowest level as the default self-identity of every object.
\end{itemize}

\subsection{A PST-Based Reinterpretation of Classical Paradoxes}

The analysis of reference chain convergence in Predictive Set Theory provides a unified cognitive account of two classical logical paradoxes.

Russell's (1903) paradox---arising from ``the set of all sets that do not contain themselves''---is reduced, within PST, to a case of cyclically divergent reference chains. When the system constructs the reference rule ``all nodes that do not refer to themselves'' and attempts to apply this rule to itself, the reference operation enters a fixed-point-free cycle: if the node refers to itself, it violates the rule; if it does not refer to itself, it satisfies the rule and therefore must refer to itself. This oscillation is mathematically strictly equivalent to a non-convergent infinite series (such as Grandi's series). PST's method of handling this is not to prevent the paradox through type-theoretic restrictions (as in Russell's ramified type theory), but rather to recognize the non-convergence of the reference chain and suspend it as a ``syntactic construction that cannot be stably mapped''---the system can formulate it, but does not assign it a truth value. The paradox thereby ceases to be a threat to the logical substrate and becomes merely a special, manageable sequence pattern within the cognitive system.

G\"{o}del's (1931) incompleteness theorems revealed that formal arithmetic systems necessarily contain undecidable propositions. PST provides an operational reinterpretation of this result: the essence of the G\"{o}del sentence \(G\) (``\(G\) is unprovable'') is the self-reference of a counter-reference chain. The system simultaneously establishes a reference chain (\(G\) is a proposition with a truth value) and a counter-reference chain (\(G\) is unprovable) toward this proposition, and these two chains trigger a contradiction when merged at the meta-linguistic level. The difference between PST's interpretation and G\"{o}del's original proof is this: PST does not view incompleteness as a ``defect'' of formal systems, but rather reveals it as a structural feature necessarily possessed by any system that \textit{permits counter-reference chains to self-refer}. A cognitive system can perfectly well accommodate such constructions---it simply cannot complete their truth-value determination without destroying consistency. Incompleteness is thus redefined as the necessary formal boundary within which counter-reference chains operate at the limit of consistency.

% ==========================================
% CHAPTER 2: AXIOMATIC DEFINITION OF THE COGNITIVE AGENT
% ==========================================

\section{Axiomatic Definition of the Cognitive Agent}

This chapter provides the axiomatic definition of the cognitive agent based on the ``consistency requirement'' and ``reference chain default self-reference'' established in Chapter~1. The central task is to strictly distinguish two fundamentally different cognitive functions---the sensor (an unconditioned independent variable) and the action-planning function (a conditioned dependent variable)---and to derive from this distinction the fundamental asymmetry between ``known states'' and ``unknown states.''

\subsection{The Unknowability Postulate of the External World}

Before defining any internal functions of the cognitive agent, a more primordial question must first be clarified: what, at the architectural level, is the relationship between the cognitive agent and the external world in which it is situated? This section establishes the basic position of Predictive Set Theory on this issue: the external world is \textbf{unknowable} to the cognitive agent---this is not an empirical discovery but an architectural stipulation.

\subsubsection{Why This Postulate is Necessary}

Cognitive theory faces a fundamental philosophical difficulty: if the cognitive system is a part of the external world, how can it ``correctly'' cognize that external world? This problem---appearing in various forms throughout the history of philosophy, from Cartesian skepticism to Kant's transcendental philosophy---is intractable because any attempt to verify ``whether cognition accurately reflects the external world'' inevitably presupposes some independent access to that external world, thus falling into circularity.

Classical representationalism attempts to respond to this difficulty by positing some isomorphism between internal representations and external objects. But this strategy faces an in-principle insurmountable difficulty: the cognitive system can never step outside itself to directly compare internal representations with external objects in order to verify whether the former ``accurately represent'' the latter. All that the system can directly access is the internal objects presented by its own sensory interface. The external world itself---if it exists---is, for the cognitive system, an impenetrable black box.

Predictive Set Theory adopts a radically different strategy for handling this difficulty. We do not attempt to ``solve'' it, but rather \textit{dissolve} it through an architectural stipulation: we openly and explicitly set the unknowability of the external world as part of the definition of the cognitive agent. This position is philosophically closest to Kant's thesis of the ``thing in itself'' (\textit{Ding an sich})---there exists an unbridgeable gulf between the phenomenal world (the world we can experience) and the thing in itself (the world as it is), and the cognitive agent can only access the former, never the latter. Unlike Kant's transcendental idealism, however, our stipulation is strictly confined within the axiomatic definition of the cognitive architecture and does not involve any metaphysical claims about a ``transcendental subject'' or ``transcendental apperception.''

\subsubsection{Formal Stipulations}

\begin{definition}[External World]
There exists an external world whose causal action upon the cognitive agent can be formalized as an external-object function:
\[
W: \Sigma \times A \to \Omega
\]
where:
\begin{itemize}
\item \(\Sigma\) is the world-state space. Its internal structure is unknowable to the cognitive agent.
\item \(A\) is the action space output by the cognitive agent.
\item \(\Omega\) is the external-object space. Its internal structure is likewise unknowable.
\item For a given world state \(\sigma \in \Sigma\) and action \(a \in A\), \(W(\sigma, a)\) outputs an external object \(O \in \Omega\).
\end{itemize}
\end{definition}

\begin{definition}[External Object]
An external object \(O \in \Omega\) is the product of the external world in a particular state, after being influenced by the cognitive agent's actions. It is that which the cognitive agent indirectly contacts through the sensor, but itself---as part of the external world---does not enter the internal operational space of the cognitive system.
\end{definition}

\begin{axiom}[Unknowability of the External-Object Function]
The external-object function \(W\) and its internal mechanisms---including the structure of its domain (\(\Sigma\) and the internal constitution of \(A\)), its mapping rules, and the internal structure of its codomain \(\Omega\)---are all unknowable to the cognitive agent. The agent cannot in any way access, invoke, query, or modify \(W\).
\end{axiom}

\begin{axiom}[Direct Inaccessibility of External Objects]
External objects \(O \in \Omega\) themselves cannot be directly accessed by the cognitive agent. The agent can never use \(O\) as a direct operand of its internal computations. For the agent, \(O\) is something \textit{before} the sensory interface---it triggers perception but does not enter into perception.
\end{axiom}

These two axioms jointly constitute the ``unknowability postulate'' of Predictive Set Theory. They demarcate the fundamental boundary between the cognitive agent and the external world: the external world exists and exerts causal influence upon the agent (otherwise there would be no perception), but the external world itself---its structure, properties, and mode of operation---can never become an object of the cognitive system's internal computations.

\subsubsection{Theoretical Significance of this Postulate}

\textbf{First, it avoids the circularity predicament of representationalism.} Representationalism implicitly assumes that internal representations can be compared with external objects to verify the accuracy of the former. But such comparison is in principle impossible for the cognitive system to perform on its own---it has no external channel independent of its sensory interface. Our stipulation, by canceling the very possibility of such comparison---formally declaring external objects to be directly inaccessible---refrains from entering the representationalist circle from the outset.

\textbf{Second, it demarcates a clear internal boundary for cognitive theory.} The cognitive system is strictly enclosed within its own sensory interface. All its computational objects are internal objects---sensory atoms output by the sensor and endowed with self-identity. The ``world'' of the cognitive system is precisely the set of these internal objects and their relational network. This is not an answer to the metaphysical question ``does the world really exist''---that question is suspended---but an answer to the methodological question ``on what basis can a cognitive system operate.''

\textbf{Third, it resonates with the grounding of consistency in Chapter~1.} In Section~1.2, we demonstrated that reference chains must default to self-reference---identity is the absolute starting point of cognition. The unknowability postulate provides the complementary half of this argument: since the external world is unknowable, the cognitive system cannot attribute the origin of identity to the ``accurate representation'' of external objects. Identity must be internally conferred by the cognitive system---it is an architectural stipulation forced by the convergence condition of the reference operation, not a factual description of the external world. Unknowability externally closes off the possibility of appealing to ``objective reality,'' while reference chain default self-reference internally provides the generative source of identity.

\subsubsection{Relation to Relevant Academic Positions}

The philosophical background of this postulate deserves further elaboration to clarify its position in the academic lineage.

In \textbf{Kantian philosophy}, the distinction between phenomena and things in themselves is the core of transcendental idealism. Kant argued that we can only cognize phenomena---things as they appear to us under the forms of intuition (space and time) and the categories of the understanding---and cannot cognize things in themselves. Our unknowability postulate is highly consonant in spirit with Kant's position, but differs in two important respects: first, our stipulation is purely formal---it involves only function definitions and axiomatic declarations, without detailed analysis of ``forms of intuition'' or ``categories of the understanding''; second, our stipulation is ``constructive'' rather than ``critical''---its purpose is not to limit the pretensions of reason, but to provide clear premises for an operable cognitive architecture.

In the \textbf{analytic philosophy} tradition, Putnam's ``brain in a vat'' thought experiment and more general anti-skeptical arguments also touch upon similar issues. Putnam argued that if we were ``brains in a vat,'' our words ``vat'' and ``brain'' would refer not to real vats and brains but to illusions generated by electronic signals. Hence the statement ``we are brains in a vat,'' if true, could not be coherently asserted by us. The insight of this argument lies in revealing the referential gulf between the external world and internal representations. Our postulate accepts the existence of this gulf but adopts a different strategy: we do not make any assertion about ``whether the external world is real,'' but instead inscribe its unknowability directly into the axioms of the cognitive architecture.

In \textbf{theoretical computer science}, the relationship between a Turing machine and its external input (the tape) provides a useful analogy. A Turing machine cannot ``check'' whether the symbols on the tape ``correctly represent'' the external world---the tape itself is the Turing machine's input, the only thing it can access. For the Turing machine, the symbols on the tape are simply ``given.'' In our cognitive architecture, the sensor plays an analogous role---it is the sole information input channel of the cognitive system, and its computational status vis-\`{a}-vis the system's internals is strictly isomorphic to that of the Turing machine's tape vis-\`{a}-vis the Turing machine.

It is necessary to further clarify the essential differences between PST's definition of the sensor and the views of perception in classical cybernetics and predictive processing. In Wiener's (1948) cybernetic framework, the system's operation revolves around an externally given ``set point''---the target temperature of a thermostat, the target speed of cruise control---and the function of feedback is to eliminate the deviation between the current state and the set point. From whence the set point itself comes, cybernetics does not inquire; it belongs to the engineer outside the system or to natural selection. PST's sensor is radically different: as the sole independent variable in the cognitive architecture, the sensor refers to no external standard whatsoever. Its output \(i\) is not an approximation to some ``correct value'' but a pure self-presentation \(S(i)=i\). Goals---i.e., the anchoring of demands---are not externally supplied set points, but are endogenously rooted in somatic feedback (see Chapter~5). The philosophical consequences of this difference are profound: a cybernetic system is a passive executor of goals; a PST system is an active constructor of demands.

Contrasting with the ``perception as inference'' dogma in predictive processing theory, PST's position is even more sharply delineated. Clark (2013) and Hohwy (2013) maintain that perception is essentially an inferential process driven by high-level predictions---the brain actively constructs hypotheses about the external world and tests them against sensory signals. Within this framework, perceptual content is always permeated by top-down predictions; there is no ``pure'' perception, only ``predicted'' perception. PST's definition of the sensor, by contrast, requires a strict architectural separation between perception and prediction. As an identity function, the sensor outputs ``given'' internal objects whose content is not modified by the system's memories, beliefs, or expectations. This is not to deny that top-down attentional modulation can influence \textit{which} perceptual contents enter the attention pool; rather, it insists that once an internal object has been presented by the sensor, the identity of its content---the fact that ``it is itself''---is impermeable to higher-level cognition. This ``impermeability'' is precisely the precondition for the system to maintain a stable perceptual substrate in a changing world---if even the sensor's output could be revised by beliefs, the system would lose the ability to distinguish ``real change'' from ``imagined change,'' and the consistency requirement would be rendered inoperative.

\subsection{Axiomatic Definition of the Sensor}

With the unknowability of the external world established, the axiomatic definition of the primary component of the cognitive architecture---the sensor---follows naturally. The sensor is the sole interface between the cognitive system and the external world. Its function is to receive the action of external objects and output internal objects---sensory atoms that can be computationally processed within the cognitive system. This section provides a rigorous formal definition of the sensor and demonstrates its core property: from the internal perspective of the cognitive system, the effective behavior of the sensor is equivalent to an identity function.

\subsubsection{Functional Form of the Sensor}

\begin{definition}[Sensor]
The sensor is the sole perceptual input interface of the cognitive system, formalized as a function:
\[
S: \Omega \to \mathcal{I}
\]
where:
\begin{itemize}
\item \(\Omega\) is the external-object space (given by Definition~2.1).
\item \(\mathcal{I}\) is the internal-object space---the set of all data atoms that can be computationally processed within the cognitive system.
\item For a given external object \(O \in \Omega\), \(S(O)\) outputs an internal object \(i \in \mathcal{I}\).
\end{itemize}
\end{definition}

The functional definition of the sensor captures a crucial structural relationship: the sensor straddles the ``external--internal'' boundary. Its domain \(\Omega\) belongs to the unknowable external world, while its codomain \(\mathcal{I}\) belongs to the internal space operable by the cognitive system. The sensor itself is the sole causal channel between these two domains.

\begin{axiom}[Unknowability of the Sensor's Internal Operation]
The internal transformation mechanism of the sensor \(S\)---i.e., the concrete mapping rules from \(O \in \Omega\) to \(i \in \mathcal{I}\)---is unknowable to the cognitive system. The system cannot access the internal computational process of \(S\), cannot modify the transformation rules of \(S\) in any way, and cannot obtain the ``source code'' or ``design parameters'' of \(S\).
\end{axiom}

This axiom complements Axioms~2.1 and~2.2. The unknowability of the external world (Axioms~2.1--2.2) ensures that the cognitive system cannot access external objects \(O\) through any other channel. Axiom~2.3 further ensures that the system cannot even obtain additional information about \(O\) by ``inspecting how the sensor works.'' The sensor is, for the cognitive system, a completely transparent interface---it only provides outputs, disclosing no internal information about its inputs.

At the neuroscientific level, the analogue of this axiom is evident: a single neuron or brain region cannot ``inspect'' its own synaptic transmission mechanisms or sensory transduction processes to obtain information about external stimuli beyond the pattern of firing. All that the system possesses is the spatiotemporal pattern of neural firing itself---our ``internal objects.''

\subsubsection{Equivalent Form of the Sensor at the Cognitive Level: The Identity Function}

Based on the above stipulations, we can now derive the most central property of the sensor.

\begin{theorem}[Equivalent Form of the Sensor at the Cognitive Level]\label{thm:sensor_identity}
From the internal perspective of the cognitive system, the effective behavior of the sensor is equivalent to an identity function:
\[
S(i) = i
\]
where \(i \in \mathcal{I}\) is an internal object. That is, the sensor is functionally equivalent to ``outputting itself''---its unconditional presentation of its input is equivalent to a self-confirmation of the internal object.
\end{theorem}

\begin{proof}
The proof proceeds in three steps.

\textbf{Step 1: Identify all facts confirmable by the system.}

When the sensor \(S\) receives an external object \(O\) and outputs an internal object \(i = S(O)\), what facts can the cognitive system, from its internal perspective, confirm?

\begin{itemize}
\item The system can confirm: the internal object \(i\) has been presented (\(i\) appears in \(\mathcal{I}\)).
\item The system \textbf{cannot} confirm: whether there exists an external object \(O\) such that \(i = S(O)\) (because \(\Omega\) is inaccessible, Axiom~2.2).
\item The system \textbf{cannot} confirm: what computation \(S\) internally performed (because the internal mechanism of \(S\) is unknowable, Axiom~2.3).
\item The system \textbf{cannot} confirm: whether \(i\) ``accurately represents'' some external object \(O\) (because comparison is impossible---comparison requires independent access to \(O\), but \(O\) is inaccessible).
\end{itemize}

Hence, in the event of the sensor outputting \(i\), the sole positive fact confirmable by the cognitive system is: ``The internal object \(i\) has been presented.'' Everything else---the external cause of \(i\), the representational accuracy of \(i\), the internal computation of \(S\)---falls outside the cognitive boundary of the system.

\textbf{Step 2: Limit-case analysis---the external object does not exist.}

Consider a limit case: suppose the external object \(O\) does not exist---i.e., the sensor receives no external causal input whatsoever. In this situation, can the sensor still output an internal object \(i\)?

The answer depends on how we understand the definition of the sensor. If the sensor were defined as a causal transducer strictly dependent on external input, then no input would mean no output. But the problem is: from the internal perspective of the cognitive system, the system \textbf{cannot determine} whether external input exists. If the sensor, in the absence of external input, still outputs \(i\) (e.g., due to internal noise, spontaneous activity, or other causes indistinguishable by the system), the system will receive the same signal ``the internal object \(i\) has been presented'' and will be unable to differentiate it from the signal ``when external input is present.''

This is not a purely thought experiment. In neurobiology, sensory systems continuously produce spontaneous activity even in the absence of external stimuli---retinal ganglion cells fire action potentials in complete darkness, and auditory nerve fibers have basal firing rates in absolute silence. From the ``internal perspective'' of the nervous system, spontaneous activity and stimulus-driven activity are completely indistinguishable in signal format. The system cannot, through introspection, determine whether a particular neural firing pattern originates from the external world or from random fluctuations internal to the system.

Thus, in the limit case---the external object does not exist---the sensor may still output an internal object \(i\). This shows that the existence of the sensor's output \(i\) does not depend on the existence of an external object \(O\) (in the sense confirmable by the cognitive system).

\textbf{Step 3: Derivation of the equivalent form.}

Combining Step~1 and Step~2: all that the cognitive system can confirm is that ``\(i\) has appeared.'' Whether this \(i\) ``represents'' some external object \(O\) is unknowable to, and has no effect upon, the system's internal computations. From the perspective of the system's internal computations, the role of the sensor is simply ``to provide \(i\)''---nothing more, nothing less.

This function can be precisely characterized as: the sensor is a device that unconditionally presents an internal object \(i\) to the system. Its output \(i\) is its output itself---it points to no ``input'' that is confirmable within the system. This is precisely the semantics of the identity function \(S(i) = i\): the sensor's ``input'' (in terms of its effective behavior at the cognitive level) is simply its output \(i\) itself.

Note that this equivalent form is not a description of the physical mechanism of the sensor---the physical mechanism involves a complex causal chain from the external world to neural signals. It is a characterization of the \textbf{functional role of the sensor within the cognitive architecture}. At the architectural level, the role the sensor plays is ``to provide internal objects to the system''---which is equivalent to the function of an identity function: the output is the output itself.

From another angle, this conclusion can also be seen as a direct corollary of a more fundamental principle: a cognitive system must anchor its data stream ``unconditionally'' somewhere. Tracing the origin of every internal object will eventually hit a non-traceable point---the sensor hardware itself. At this point, the question ``where does this internal object come from'' is unanswerable within the system. The system's internal answer can only be: ``It comes from the sensor.'' And from the system's internal perspective, ``coming from the sensor'' is equivalent to ``it is itself''---because the sensor is an impenetrable black box. Hence, the internal object output by the sensor, at the lowest level, is simply the self-presentation of itself.
\end{proof}

\subsubsection{Theoretical Status of the Sensor}

Theorem~\ref{thm:sensor_identity} endows the sensor with a precise theoretical status within the cognitive architecture.

\textbf{First, the sensor is the sole independent variable in the cognitive architecture.}

In mathematics, an independent variable is a quantity whose value does not depend on the other variables in the function under consideration. In the cognitive architecture, the sensor's output \(i\) does not depend on the output of the action-planning function, does not depend on memory contents, and does not depend on any beliefs or expectations modifiable within the system. It is ``given''---its occurrence and content do not depend on the system's internal states.

The significance of this status as an independent variable is that it provides the cognitive system with an absolute factual anchor that cannot be internally revised. No matter what the system has learned, what it believes, or what it expects, the internal object \(i\) output by the sensor will not thereby change its content. This realizes, at the architectural level, a strict separation between ``perception'' and ``belief''---perception is not permeated by belief.

\textbf{Second, \(S(i) = i\) is the most primitive ``identity'' operation.}

In Section~1.2, we established, through the convergence analysis of infinite reference chains, that ``reference chain default self-reference'' is a generative necessity of cognitive computation. Theorem~\ref{thm:sensor_identity} reveals the hardware counterpart of this necessity: the sensor, at every perceptual pixel, every millisecond, every frequency channel, faithfully executes the identity operation \(S(i) = i\). This execution does not ``verify'' identity---it requires no comparison, no judgment---but directly \textit{constitutes} identity. The sensor, at the hardware level, forcibly endows internal objects with irreducible self-identity, making ``this perception is itself'' the most unshakeable fact within the cognitive system.

Thus, reference chain default self-reference (Axiom~\ref{axiom:selfref}) and sensor identity operation (Theorem~\ref{thm:sensor_identity}) form a perfect structural correspondence:

\begin{itemize}
\item \textbf{Axiom~\ref{axiom:selfref}}: At the logical and computational level, the tracing of a reference chain must terminate at a self-referential fixed point---this is the transcendental formal condition for cognitive interrogation to be completable.
\item \textbf{Theorem~\ref{thm:sensor_identity}}: At the architectural and hardware level, the effective behavior of the sensor is equivalent to the identity function \(S(i) = i\)---this is the hardware mechanism through which identity is enforced at the sensory interface.
\end{itemize}

Axiom~\ref{axiom:selfref} provides the ``logical necessity'' of identity; Theorem~\ref{thm:sensor_identity} provides the ``hardware execution'' of identity. Together they anchor the most fundamental consistency of Predictive Set Theory---undeniable at the conceptual level, and unbypassable at the hardware level.

\textbf{Third, the sensor is the sole source of ``known states.''}

Since the external world is unknowable and the sensor is the sole perceptual input interface (Definition~2.3), the internal objects \(i\) output by the sensor are the source of all initial data obtainable by the cognitive system. All known states---whether current perceptions, historical memories, or abstract knowledge obtained through reasoning---ultimately trace back to some sensor output moment.

This property will be combined, in Section~2.4, with the definition of the action-planning function to derive the asymmetry theorem of known and unknown states. At the present stage, it suffices to confirm: the privileged position of the sensor in the architecture---as the sole independent variable and the hardware executor of identity---makes it the absolute starting point of the cognitive data stream.

\subsection{Axiomatic Definition of the Action-Planning Function and Its Strict Distinction from the Sensor}

Having defined the sensor (Section~2.2) and established its status as the sole independent variable in the cognitive architecture, the next step in defining the cognitive agent is to formalize its decision-making core---the action-planning function. This section provides the axiomatic definition of the action-planning function and rigorously demonstrates its fundamental distinction from the sensor: the action-planning function cannot be an independent variable, but must be a conditioned dependent variable whose output strictly depends on the known states currently possessed by the system.

\subsubsection{Formal Definition of the Action-Planning Function}

\begin{definition}[Action-Planning Function]
The action-planning function \(f\) is the decision-making core of the cognitive system. Its function is to output a behavioral plan based on all known states currently accessible to the system. Formally:
\[
A_t = f(K_t)
\]
where:
\begin{itemize}
\item \(K_t \subseteq \mathcal{I}\) is the \textbf{known-state set} of the system at time \(t\)---i.e., the set of all internal objects that the sensor has output and the system has recorded up to time \(t\). The rigorous construction of \(K_t\) will be elaborated in Section~2.4 and Chapter~3.
\item \(A_t \in \mathcal{A}\) is the \textbf{behavioral plan} output by the system at time \(t\). \(\mathcal{A}\) is the behavioral-plan space, whose internal structure depends on the effector type of the cognitive agent and is not further specified here.
\item \(f: \mathcal{P}(\mathcal{I}) \to \mathcal{A}\) is a function mapping from the known-state set to a behavioral plan. \(\mathcal{P}(\mathcal{I})\) denotes the power set of the internal-object set---i.e., the set of all possible known-state sets.
\end{itemize}
\end{definition}

This definition explicitly presents the action-planning function as a mapping from the system's internal cognitive state to a behavioral output. \(f\) does not directly contact the external world, nor does it directly contact the current raw output of the sensor (in the definition, the sensor's current output only becomes operational material for \(f\) after being incorporated into \(K_t\)). All information input to \(f\) comes from \(K_t\)---that is, what the system already ``knows.''

\subsubsection{The Action-Planning Function Cannot Be an Independent Variable}

In Section~2.2, we demonstrated the core property of the sensor in the cognitive architecture: the sensor is an independent variable---its output does not depend on any causal chain externally confirmable by the system, equivalent to the identity function \(S(i) = i\). A natural question arises: could the action-planning function \(f\) also be an independent variable? If \(f\) could also unconditionally produce behavioral outputs, then the system would possess two independent information sources, rendering the cognitive architecture redundant and uncontrollable.

\begin{axiom}[Non-Independent-Variability of the Action-Planning Function]\label{axiom:action_nonindep}
The action-planning function \(f\) cannot be equivalent to an independent variable. Concretely, there exists no valid mode of computation such that the output of \(f\) can be produced without depending on its input \(K_t\). \(f\) must be a conditioned function---its output must strictly depend on the content of \(K_t\).
\end{axiom}

\textbf{Justification:} The justification of this axiom rests on three levels of argument: structural, functional, and generative.

\textbf{Structural argument: the principle of the unique independent variable.} In Section~2.2, we have already established that the sensor is the sole independent variable in the cognitive architecture. This ``uniqueness'' is not accidental, but a necessary condition for the cognitive architecture to maintain self-consistency. If \(f\) were also an independent variable---i.e., its output did not depend on \(K_t\)---the system would face the following dilemma:

\begin{itemize}
\item If the output of \(f\) were independent of the sensor's output, the system would simultaneously possess two mutually independent information sources: the sensor providing perceptual input, and \(f\) providing behavioral output. But this would mean that there is no necessary causal link between behavioral output and perceptual input---\(f\) could output any behavioral plan regardless of what the sensor just input. Such a system could not be described as ``responding'' to the environment---it would merely be spontaneously generating behavior, with no traceable feedback loop between its behavior and the environment. This violates the basic definition of a cognitive agent as ``an adaptive system coupled to its environment.''
\item If the output of \(f\) attempted to maintain some relationship with the sensor's output (e.g., if \(f\) were also some kind of identity function directly outputting the sensor's current value), then \(f\) would functionally degenerate into a duplicate of the sensor---the system's behavioral output would be merely a copy of perceptual input. In this case, the system's behavior would contain no ``planning'' component whatsoever; it would be merely a passive sensor relay. This contradicts the definition of the action-planning function in Section~2.3.1---that its output is a ``behavioral plan''---since planning necessarily involves processing and transforming input, not merely copying it.
\end{itemize}

Hence, to maintain the structural self-consistency of the cognitive architecture, \(f\) must be strictly distinguished from the sensor in functional nature: the sensor is an independent variable, and \(f\) is a dependent variable.

\textbf{Functional argument: the condition of possibility of planning.} The semantic core of ``planning'' is: selecting not-yet-determined future actions based on known information. This semantics presupposes a causal-logical link between known information and behavioral output---i.e., the content of the behavioral output is constrained and determined by the content of known information. If the output of \(f\) did not depend on known information, then the concept of ``planning'' would lose its operational meaning. The system would not be ``planning'' behavior, but ``emitting'' behavior---its behavior would not be explained by any internal cognitive state.

More concretely, consider the system facing a behavioral choice at time \(t\). If \(f\) were an unconditioned function, then no matter what the content of \(K_t\) was---no matter what the system had just experienced, learned, or remembered---the output of \(f\) would be the same. Such a system could not learn from experience, because the very essence of learning is to modify \(K_t\) so as to change the output of \(f\). If the output of \(f\) were independent of \(K_t\), then no change in \(K_t\) would affect behavior. The system would lose the basis of adaptability.

Thus, the conditioned nature of the action-planning function---its output depending on known states---is the transcendental condition for the two core cognitive concepts of ``planning'' and ``learning'' to be possible.

\textbf{Generative argument: behavioral output cannot be ``given.''} In Section~2.2, the internal objects output by the sensor were demonstrated to be ``given''---their occurrence does not depend on the system's internal states, and they are unmodifiable cognitive atoms. Behavioral output is fundamentally different in nature: behavioral output is \textit{actively produced} by the system, not passively received. The system bears causal responsibility for its behavioral output---it is the system that ``decides'' what behavior to output.

If \(f\) were also an independent variable, its output would be equivalent to ``given,'' and behavioral output would be ontologically homogeneous with perceptual input---both would merely be signals passively received by the system. This would completely obliterate the agency of the cognitive agent, reducing it to a dual-channel passive receiver: one channel receiving perceptions, another receiving ``behavioral commands.'' But whence would these ``behavioral commands'' come? If they also came from outside the system, then the system would not be an autonomous cognitive agent, but a puppet manipulated externally. If they came from within the system but were not constrained by any internal state, then they would be random---contradicting the semantics of ``planning.''

Thus, from the generative perspective, the action-planning function must be a conditioned dependent variable. Its output is not given, but selected by the system based on already available information. This distinction---perception is given, behavior is chosen---is one of the most fundamental architectural distinctions in the definition of the cognitive agent.

\subsubsection{The Closure Property of the Action-Planning Function}

Based on Axiom~\ref{axiom:action_nonindep} (the non-independent-variability of the action-planning function) and the unknowability postulates of Sections~2.1 and~2.2, we can derive a crucial property of the action-planning function: \textbf{closure}.

\begin{theorem}[Closure of the Action-Planning Function]\label{thm:closure}
The action-planning function \(f\) is strictly enclosed within the known-state set \(K_t\). Concretely:
\begin{enumerate}
\item The input space of \(f\) is restricted to \(K_t\)---it cannot directly access internal objects not yet output by the sensor, nor can it directly access external objects.
\item \(f\) cannot in any way ``query'' or ``invoke'' the external world to obtain information beyond \(K_t\).
\end{enumerate}
\end{theorem}

\begin{proof}
\begin{enumerate}
\item By Definition~2.4, the input of \(f\) is solely \(K_t\). \(K_t\) is the accumulation of internal objects output by the sensor up to time \(t\) (its rigorous construction is given in Section~2.4 and Chapter~3).
\item By Axiom~2.2, external objects \(O\) are not directly accessible.
\item By Axioms~2.1 and~2.3, the external-world function \(W\) and the internal mechanism of the sensor are both unknowable.
\item Hence, \(f\) has no other information channel. It is strictly confined within the known-state set \(K_t\).
\end{enumerate}
\end{proof}

The closure theorem characterizes the action-planning function as an ``internal-model-driven'' decision unit: it can only plan based on information the system already possesses, and cannot ``peek'' at the external world's answer. This property provides the foundation for subsequent theories of prediction and learning---because \(f\) cannot directly obtain future states, it must \textit{predict} them; because \(f\) cannot directly verify the external world, it must \textit{self-correct} through feedback.

\subsection{The Asymmetry Theorem of Known and Unknown States}

Having defined the sensor (Section~2.2) and the action-planning function (Section~2.3), we can now combine the two to derive one of the most important theorems of Predictive Set Theory: the asymmetry theorem of known and unknown states. This theorem rigorously establishes the absolute asymmetry of two cognitive statuses within the cognitive system, and serves as the foundation for all subsequent derivations concerning prediction, learning, demand evaluation, and behavioral choice.

\subsubsection{Rigorous Definition of the Known-State Set}

Before formally stating the theorem, we must first rigorously define the concept of the ``known-state set.''

\begin{definition}[Known-State Set \(K_t\)]
At any time \(t\), the known-state set \(K_t\) of the cognitive system is defined as the cumulative set of all internal objects output by the sensor from the initial moment \(t=0\) to the current moment \(t\). Formally:
\[
K_t = \bigcup_{\tau=0}^{t} \{i_\tau\}
\]
where \(i_\tau = S(O_\tau)\) is the internal object output by the sensor at time \(\tau\). In particular, when \(t=0\), \(K_0\) contains the internal object output by the sensor at the initial moment (if the sensor produces no output at the initial moment, then \(K_0 = \emptyset\), but this does not affect the generality of the subsequent reasoning).

This definition is strictly isomorphic in mathematical structure to von Neumann's construction of the ordinals:
\[
K_{t+1} = K_t \cup \{i_{t+1}\}
\]
where \(i_{t+1}\) is the internal object newly output by the sensor at time \(t+1\). This recursive construction endows the known-state set with the properties of cumulativity and monotonic increase:
\[
K_0 \subseteq K_1 \subseteq K_2 \subseteq \cdots
\]
Known states only increase, never decrease. This is fully consistent with the principle of the ``unrewritability of sequences'' and will be further elaborated in Chapter~3.
\end{definition}

\begin{definition}[Cognitive Set \(C_t\)]
To distinguish ``the accumulation of all known states'' from ``the immediate perception of the current moment,'' we introduce the concept of the Cognitive Set:
\[
C_t = \{i_t\}
\]
i.e., the singleton set containing only the internal object most recently output by the sensor at time \(t\). \(C_t\) is the new information that the action-planning function directly confronts ``at the present moment''---it is the incremental part of the known-state set at time \(t\), i.e., \(K_t \setminus K_{t-1}\) (when \(t \geq 1\)).
\end{definition}

\subsubsection{Set-Theoretic Definition of Relative Unknown State}

In Section~3.4, we will discuss the set-theoretic definition of the relative unknown state in a more general form. Here, a basic version is given in advance to serve the statement of the asymmetry theorem.

\begin{definition}[Relative Unknown-State Set]
For any two known-state sets \(K_a\) and \(K_b\), the unknown-state set of \(K_b\) relative to \(K_a\) is defined as their set difference:
\[
U(K_b \mid K_a) = K_b \setminus K_a
\]
i.e., the set of all internal objects that belong to \(K_b\) but do not belong to \(K_a\).

A crucial special case of this definition is when \(K_b = K_{t+1}\) and \(K_a = K_t\):
\[
U(K_{t+1} \mid K_t) = K_{t+1} \setminus K_t = \{i_{t+1}\}
\]
That is, the ``unknown increment'' between adjacent moments is precisely the single internal object newly output by the sensor at time \(t+1\).

This definition reveals the set-theoretic essence of ``unknown'': it is not an absolute property of an internal object, but a set-difference relation between two sets. The same internal object \(i\) is unknown with respect to \(K_t\) (\(i \notin K_t\)), but known with respect to \(K_{t+1}\) (\(i \in K_{t+1}\)). The relativity of unknownness is thereby given a precise formal expression.
\end{definition}

\subsubsection{Statement and Proof of the Asymmetry Theorem}

\begin{theorem}[Asymmetry Theorem of Known and Unknown States]\label{thm:asymmetry}
Let \(K_t\) be the known-state set of the system at time \(t\), and let \(U_t\) be the totality of unknown states of the system at time \(t\) relative to \(K_t\)---i.e., the set of all internal objects not yet output by the sensor but possibly output in the future. Then the following two asymmetry rules strictly hold:

\textbf{Rule 1 (Known cannot derive unknown):} The action-planning function \(f\) cannot derive any element of \(U_t\) from \(K_t\). That is, there exists no effective computation \(g: \mathcal{P}(\mathcal{I}) \to \mathcal{I}\) such that \(g(K_t) \in U_t\).

\textbf{Rule 2 (Unknown can be transformed into known):} When the sensor refreshes, some \(u \in U_t\) is transformed into an element of \(K_{t+1}\). That is, the sensor \(S\) is the sole channel through which unknown states can be transformed into known states.
\end{theorem}

\begin{proof}
\textbf{Proof of Rule 1:}

Assume, for contradiction, that there exists an effective computation \(g\) such that \(u = g(K_t) \in U_t\). Since \(u \in U_t\), by the definition of unknown state (Definition~2.7 and its context), \(u \notin K_t\)---i.e., \(u\) is an internal object that the system does not yet know at time \(t\).

But \(g(K_t)\) is the result of applying \(g\) to the known-state set \(K_t\). Since \(K_t \subseteq \mathcal{I}\) is known to the system (the system does possess the content of \(K_t\)), and \(g\) is a computation executable by the system, the result of \(g(K_t)\)---namely \(u\)---becomes known to the system the moment the computation is completed. This means that \(u\) can be added to \(K_t\), forming \(K_t \cup \{u\}\), which is precisely the construction method of \(K_{t+1}\) (when \(u = i_{t+1}\)).

But this yields a contradiction: if \(u\) can be obtained through the internal computation \(g\), then it is no longer an ``unknown state''---it enters the known-state set at the moment the computation completes. If one insists that \(u \in U_t\) (i.e., before the computation completes, it is not in \(K_t\)), then the existence of \(g\) would mean that the system can unilaterally expand its known-state set without relying on sensor refresh. But according to Theorem~\ref{thm:closure} (closure of the action-planning function), \(f\) cannot directly obtain information beyond \(K_t\); according to Section~2.2, the sensor is the sole source of known states. The existence of \(g\) would endow the system with an information-acquisition channel independent of the sensor, contradicting the sensor's status as the sole independent variable.

Hence, the assumption fails. No such \(g\) exists. Known states cannot derive unknown states through any internal computation.

\textbf{Proof of Rule 2:}

At time \(t+1\), the sensor \(S\) receives the external object \(O_{t+1}\) and outputs the internal object \(i_{t+1} = S(O_{t+1})\). At time \(t\) (i.e., before the sensor refreshes), \(i_{t+1} \notin K_t\); therefore \(i_{t+1} \in U_t\)---it is an unknown state of the system at time \(t\).

After the sensor refreshes, \(K_{t+1} = K_t \cup \{i_{t+1}\}\) (Definition~2.5). Hence, \(i_{t+1} \in K_{t+1}\)---the unknown state has been transformed into a known state.

The sensor is the sole channel effecting this transformation, because:
\begin{itemize}
\item By Axioms~2.1--2.3, information from the external world can only enter the system through the sensor.
\item By Theorem~\ref{thm:closure}, the action-planning function \(f\) is enclosed within \(K_t\) and cannot directly obtain external information.
\item By Rule~1, internal computation cannot derive unknown states.
\end{itemize}
Thus, the sensor is the sole channel through which unknown states enter the known-state set.
\end{proof}

\subsubsection{Corollaries of the Theorem}

\begin{corollary}[Cognitive Closure of the Action-Planning Function]
The action-planning function \(f\) is strictly enclosed within the known-state set \(K_t\). It cannot access, cannot predict, and cannot in any way utilize information from the unknown-state set \(U_t\) for its computations. All operational material of \(f\) consists of already presented known states.
\end{corollary}

\begin{corollary}[Perpetual Fallibility of Prediction]
Since the known cannot derive the unknown (Rule~1), any prediction about future states generated by \(f\) on the basis of \(K_t\) is, in essence, hypothetical and fallible. The correctness of a prediction can never be guaranteed in advance---it can only be verified or falsified by future sensor refresh. This establishes the fundamental uncertainty and conjectural nature of prediction as a cognitive activity.
\end{corollary}

This corollary directly echoes the philosophy of science of Popper (1959). Popper argued that scientific theories can never be verified, only falsified---no matter how many successful predictions have been made, the next prediction may still fail. PST derives the same conclusion from the internal architecture of cognition: the asymmetry that ``the known cannot derive the unknown'' entails that any extrapolation from current known states to future states cannot be guaranteed correct in advance, logically speaking. The correctness of a prediction can only be adjudicated by future sensor refresh, and the content of the sensor refresh is, before the refresh, absolutely unknown to the system.

This position also reveals the fundamental limitation of Bayesian cognitive science in handling the problem of ``unknown.'' In the Bayesian framework, the unknown is handled by assigning a uniform prior over the hypothesis space---the system assigns equal probabilities to all hypotheses and then gradually converges to high-probability hypotheses through Bayesian updating \cite{Knill2004,Doya2007}. But this operation presupposes that the hypothesis space has been fully enumerated and closed. PST's ``known/unknown asymmetry'' reveals a more fundamental problem: genuinely unknown possibilities---those not yet included in the hypothesis space---cannot be covered by a probability value. PST handles this situation through semi-reference chains: at positions where a referential intention has been initiated but the content is not yet clear, the system inserts an incomplete reference, acknowledging ``there is something unknown here,'' without disguising it as a uniform gamble over known hypotheses.

\begin{corollary}[Sensor as the Sole Bottleneck of Information Input]
The sensor is the sole channel through which the cognitive system acquires new information. The system cannot generate new perceptual content through ``thinking''---thinking can only reorganize and process known states, and cannot create, out of nowhere, new information from the external world. This corollary provides an architectural basis for the strict separation between perception and thought.
\end{corollary}

\begin{corollary}[Generative Explanation of Time at the Cognitive Level]
At the cognitive level, the passage of time is equivalent to the irreversible process by which unknown states are continuously transformed into known states through sensor refresh. The known cannot revert to the unknown (because \(K_t \subseteq K_{t+1}\) is monotonic increasing; the system cannot ``forget'' an element already entered into the known-state set). Hence, the arrow of cognitive time is unidirectional---it is immanent in the cumulative expansion of the known-state set, requiring no appeal to any external concept of physical time.
\end{corollary}

% ==========================================
% CHAPTER 3: STATE REFRESH AND THE GENERATIVE CONSTRUCTION OF SEQUENCES
% ==========================================

\section{State Refresh and the Generative Construction of Sequences}

In Chapter~2, we completed the axiomatic definition of the cognitive agent: the sensor was established as the sole independent variable (\(S(i) = i\)), the action-planning function was established as a conditioned dependent variable (\(f(K_t) = A_t\)), and the asymmetry theorem of known and unknown states was rigorously proved. This work laid the static foundation of the cognitive architecture---it defined ``what components constitute the cognitive system'' and ``what fundamental constraints hold among these components.''

However, a complete cognitive architecture requires not only static component definitions, but also dynamic operational mechanisms. The cognitive system is not a motionless statue, but a process that continuously operates in time. The sensor continuously outputs new internal objects, the action-planning function continuously outputs new behavioral plans, and the system is continuously changing. \textbf{State refresh} and \textbf{sequence} are precisely the core concepts that characterize this dynamic process.

This chapter proceeds from the cumulative expansion of the known-state set to rigorously define the operational mechanism of state refresh and, on this basis, to provide the generative construction of sequences. The chapter will reveal that sequences are not passive recordings of an external temporal dimension, but direct expressions of the fact of monotonic increasing expansion of the cognitive system's internal known-state set. The order, unrewritability, and discreteness of sequences can all be rigorously derived from the inclusion relations among known-state sets, without presupposing any external concept of time.

\subsection{The Register Object and the Comparison Operation}

Before formally defining state refresh, we must first introduce a crucial internal component---the \textbf{register object}. The register object plays a mediating role between the sensor and the known-state set: it holds the most recent value output by the sensor, provides the system with a working memory of the ``current world state,'' and is compared against the ongoing output of the sensor to trigger the state refresh operation.

\subsubsection{Definition of the Register Object}

\begin{definition}[Register Object]
The register object \(E\) is a dedicated storage unit within the cognitive system. Its function is to hold the value of the internal object most recently output by the sensor and already confirmed by the system as ``current.'' At any time \(t\), the content of the register object is denoted \(E_t\), satisfying:
\[
E_t = i_t
\]
where \(i_t\) is the internal object output by the sensor at time \(t\) (Definition~2.3), and this output has already been processed and confirmed by the system.

The relationship between the register object \(E\) and the known-state set \(K_t\) (Definition~2.5) must be precisely delineated:

\begin{itemize}
\item \textbf{The known-state set \(K_t\)} is the accumulation of all sensor outputs from the initial moment to the current moment. It is a continuously growing, non-shrinkable set that contains the system's entire perceptual history.
\item \textbf{The register object \(E_t\)} is the most recent element of \(K_t\)---namely \(i_t\)---held as an independent copy or pointer. It is a snapshot of that particular internal object on which the system is ``currently focusing.''
\end{itemize}

Functionally, \(K_t\) provides the \textbf{breadth} of the system's cognition---how much the system knows; \(E_t\) provides the \textbf{current focus} of the system's cognition---what the system is taking, at this moment, as the reference point for ``the state of the world.'' The two play different roles in the cognitive architecture: \(K_t\) is the input domain of the action-planning function \(f\) (Definition~2.4), while \(E_t\) is the comparison baseline for the state-refresh triggering mechanism (see Section~3.2).
\end{definition}

\begin{definition}[Initial State of the Register Object]
At the initial moment \(t=0\) of the cognitive system, the content of the register object can be either:
\begin{enumerate}
\item The first internal object \(i_0\) output by the sensor at \(t=0\) (if the sensor produces output immediately at the initial moment); or
\item A special ``null'' marker \(\emptyset_E\), indicating that the system has not yet received any perceptual input.
\end{enumerate}
In the following exposition, for simplicity of presentation, we assume \(E_0 = i_0\). This does not affect the generality of the reasoning---if the null initial state is adopted, the first sensor output will automatically trigger state refresh, and all subsequent processes are entirely identical.
\end{definition}

\subsubsection{Definition and Function of the Comparison Operation}

\begin{definition}[Comparison Operation]
The comparison operation is a fundamental cognitive computation continuously executed by the cognitive system. At any moment, the system judges the identity between the sensor's current output \(i_{current}\) and the content of the register object \(E\). Formally, the comparison operation is a predicate function:
\[
\text{Compare}(i_{current}, E) = \begin{cases}
\text{True} & \text{if } i_{current} = E \\
\text{False} & \text{if } i_{current} \neq E
\end{cases}
\]
where ``\(=\)'' denotes the identity relation between internal objects. The basis for judging this identity relation is the self-identity of internal objects---forcibly endowed by the sensor's identity operation \(S(i) = i\) (Theorem~\ref{thm:sensor_identity}). The system requires no extra ``comparison criterion'' to determine whether two internal objects are identical---the identity judgment is directly based on their identity in the space \(\mathcal{I}\): if \(i_{current}\) and \(E\) are the same internal object (indistinguishable within the system), the comparison result is True; otherwise, it is False.

\textbf{The continuous nature of the comparison operation:} The comparison operation is not a one-off event, but a process continuously executed by the cognitive system in time. Whenever the sensor produces a new output, the system automatically performs a comparison operation. This enables the system to monitor in real time ``whether the world has changed''---i.e., whether the sensory interface is presenting content that differs from what was previously there.
\end{definition}

\subsubsection{Cognitive Functions of the Register Object}

The combination of the register object \(E\) and the comparison operation performs three crucial functions in the cognitive architecture.

\textbf{First, the register object is the detection baseline for ``invariance.''} For the system to recognize ``change,'' it must have a stable reference point. If the system had no memory at all---if each new perceptual input immediately overwrote the record of the previous input---then the system would forever be unable to judge ``whether the world has changed,'' because there would no longer be any ``previous world'' available for comparison. The register object provides this reference point: it holds the value last output by the sensor, anchoring for the system a stable hypothesis of ``the current world state.'' Only when a new input is inconsistent with this hypothesis does the system recognize the occurrence of change.

\textbf{Second, the register object realizes, at the architectural level, the instantaneous presentation of ``known states.''} The known-state set \(K_t\) contains the system's entire history from the initial moment to the present. But at the moment of decision, the system does not need to retrieve the entire history---what it most urgently needs is to know ``what is happening right now.'' The register object \(E_t\) is precisely the operational realization of this need: it provides the action-planning function with an instantaneous snapshot of the current state at minimal retrieval cost. Functionally, \(E_t\) is a ``cache'' of \(K_t\)---it is not a replacement for \(K_t\), but a fast access point into \(K_t\).

\textbf{Third, the register object provides a clear logical condition for triggering state refresh.} State refresh---writing the old state into the sequence and updating the known-state set---is an irreversible operation. The system must have a clear, unambiguous triggering condition to determine ``when'' to execute this operation. The output of the comparison operation provides exactly this condition: state refresh is triggered if and only if the comparison result is False. The binary nature (True/False) of this logical condition ensures that the triggering of state refresh is determinate and can be unambiguously recognized by the system itself.

\subsubsection{Relation between the Register Object and the Action-Planning Function}

It must be clarified that the content of the register object \(E_t\) is not a direct input to the action-planning function \(f\). According to Definition~2.4, the input of \(f\) is the entire known-state set \(K_t\). \(E_t\) is the most recent element in \(K_t\), and \(K_t\) is the complete information store accessible to \(f\). \(f\) can access \(E_t\) (since it belongs to \(K_t\)), but \(f\) does not only access \(E_t\)---it can also access earlier elements in \(K_t\) for cross-temporal comparisons, trend analyses, and memory retrieval.

The special status of the register object \(E_t\) lies in this: it is the ``gate'' between the sensor and the known-state set. A new internal object output by the sensor, before being incorporated into \(K_t\), is first compared with \(E_t\). The result of the comparison determines whether the new object is absorbed ``without change'' (if identical) or incorporated in a manner that ``triggers state refresh'' (if not identical). This gating mechanism ensures that the triggering of state refresh is strictly based on perceptual change, not on arbitrary products of internal inference.

\subsection{Operational Definition of State Refresh}

With the concepts of the register object and the comparison operation established, the rigorous definition of state refresh follows naturally. State refresh is the core operation by which the cognitive system responds to perceptual change---it is the crucial link connecting ``sensor output'' and ``expansion of the known-state set,'' and it is the minimal constitutive unit of sequences.

\subsubsection{Definition of State Refresh}

\begin{definition}[State Refresh]
State refresh is the set of atomic operations executed by the cognitive system when it detects that the sensor's new output differs from the register object. Concretely, at time \(t+1\), the sensor outputs \(i_{t+1}\), and the comparison operation detects \(i_{t+1} \neq E_t\) (where \(E_t = i_t\) is the content of the register object at time \(t\)); the system then executes the following operations:

\textbf{Operation 1 (Write to sequence):} The current known-state set \(K_t\)---i.e., the complete cognitive state before the refresh---is written, as a whole, into a dedicated storage structure called the \textbf{state-refresh sequence}. What is written is \(K_t\) itself, not some subset or summary thereof.

\textbf{Operation 2 (Update known-state set):} A new known-state set \(K_{t+1}\) is constructed, satisfying:
\[
K_{t+1} = K_t \cup \{i_{t+1}\}
\]
That is, the new set contains all elements of the old set, and additionally contains the internal object \(i_{t+1}\) newly output by the sensor.

\textbf{Operation 3 (Update register object):} The content of the register object \(E\) is updated from \(i_t\) to \(i_{t+1}\):
\[
E_{t+1} = i_{t+1}
\]

These three operations constitute an indivisible atomic transaction. They are logically completed simultaneously---there is no intermediate state in which one of the operations has been completed while another has not. This ensures that the system never finds itself in an inconsistent state such as ``the known-state set has been updated but the register object has not'' or ``the sequence has been written but the known-state set has not been updated.''
\end{definition}

\subsubsection{When State Refresh Is Not Triggered}

To fully define state refresh, we must also clarify the circumstance in which it is \textit{not} triggered. When the sensor output \(i_{t+1}\) is identical to the register object \(E_t\) (i.e., \(i_{t+1} = i_t\)), the system judges that the perceived world has not changed. In this case:

\begin{itemize}
\item The state-refresh sequence is not appended with a new entry.
\item The known-state set is \textbf{not} updated---i.e., \(K_{t+1} = K_t\). Since \(i_{t+1} = i_t\) and \(i_t\) is already in \(K_t\), we have \(K_t \cup \{i_{t+1}\} = K_t\). The set has not expanded.
\item The content of the register object \(E\) remains unchanged (or, equivalently, is ``updated'' with the same content as before).
\end{itemize}

This ``no-op'' path is an integral part of the state-refresh mechanism, not an exception. It ensures that the system does not misjudge ``the world remaining unchanged'' as ``a new state.'' In the perceived world, invariance is often the norm, and change the exception. The design of the state-refresh mechanism reflects this basic fact: the system maintains, by default, the hypothesis of ``world stability,'' and updates its cognitive state only when explicitly contradicted by evidence.

\subsubsection{The Discreteness Theorem of States}

Based on the operational definition of state refresh, we can derive a fundamental property of cognitive states: discreteness.

\begin{theorem}[Discreteness of Cognitive States]\label{thm:discreteness}
The states that a cognitive system can recognize and record are necessarily discrete---i.e., there exist clear, discontinuous boundaries between states. The system cannot possibly recognize a ``continuously varying'' stream of states.
\end{theorem}

\begin{proof}
The triggering condition for state refresh is the comparison operation outputting False---i.e., \(i_{t+1} \neq E_t\). This is a binary judgment: two internal objects are either identical or not identical. There is no intermediate condition of ``partially identical'' or ``unequal to some degree.''

The root of this binary logic can be traced back to the self-identity of internal objects (Theorem~\ref{thm:sensor_identity}, \(S(i) = i\)). Internal objects, as cognitive atoms, admit precise and unambiguous identity judgments within the system: \(i_1\) and \(i_2\) are either the same internal object or they are not. This ``atomicity'' precludes the possibility of continuous perception---because continuous perception would require that ``the difference between two internal objects can be arbitrarily small,'' yet there is no ``distance metric'' between internal objects. They are simply themselves.

Hence, state refresh either occurs or does not occur. The system will not execute a ``partial state refresh'' because ``perception changed a little bit.'' The existence and transition of states is binary and discrete. Every state refresh produces a determinate new state that is clearly distinguishable from the preceding state.
\end{proof}

\begin{corollary}[Generative Explanation of Continuity]
The cognitive system's experience of ``continuity''---for example, the smooth motion of objects in the visual world---is not a direct recording of a continuous perceptual stream, but rather an active construction produced by higher-level cognitive mechanisms (such as sequence templates and attention-free prediction) through interpolation and filling, on the basis of discrete states. Continuity is a cognitive achievement, not a perceptual given. This resonates with the discussion of continuity and discreteness in Chapter~1 (the discrete convergence of reference chains, the discrete partial sums in the infinite series analogy): cognition, at its lowest level, is discrete; continuity is a constructed product of higher-level cognitive operations.
\end{corollary}

\subsubsection{Relation between State Refresh and the Known/Unknown Asymmetry}

The operational definition of state refresh forms a tight connection with the known/unknown asymmetry theorem (Theorem~\ref{thm:asymmetry}) established in Section~2.4.

\begin{itemize}
\item Rule~1 (the known cannot derive the unknown): The system cannot, through internal computation, predict the concrete content of \(i_{t+1}\). The triggering of state refresh---\(i_{t+1} \neq E_t\)---is itself a confirmation of ``unknown'': before the refresh, the system cannot know the concrete content of \(i_{t+1}\); it can only passively discover its existence after the refresh, through the comparison operation.
\item Rule~2 (the unknown can be transformed into the known): Operation~2 of state refresh---\(K_{t+1} = K_t \cup \{i_{t+1}\}\)---is precisely the operational realization of ``the unknown transformed into the known.'' \(i_{t+1}\) is unknown before the refresh (not belonging to \(K_t\)), and becomes known after the refresh (belonging to \(K_{t+1}\)).
\end{itemize}

Thus, state refresh is the dynamic execution mechanism of Theorem~\ref{thm:asymmetry}. Theorem~\ref{thm:asymmetry} establishes the static asymmetry between the known and the unknown; state refresh defines how this asymmetry concretely unfolds in time---how it is triggered through the comparison of the register object, and how it dissolves the unknown through the expansion of the known-state set. Together, they constitute the core logic of the cognitive system's temporal operation.

\subsection{Strict Distinction between the Known-State Sequence and the State Sequence}

In Section~3.2, we defined the operational mechanism of state refresh: when the sensor output differs from the register object, the system writes the current known-state set \(K_t\) into the sequence and constructs a new set \(K_{t+1} = K_t \cup \{i_{t+1}\}\). This mechanism produces two kinds of sequences that must be strictly distinguished at the conceptual level---the \textbf{known-state sequence} and the \textbf{state sequence}. Confusing the two would lead to a category error, making it impossible to cleanly separate the time-stamp function from the behavior-planning function of sequences. This section provides precise definitions of both and elucidates their distinct functional roles within the cognitive architecture.

\subsubsection{Rigorous Definition of the Known-State Sequence}

\begin{definition}[Known-State Sequence]
The known-state sequence \(\mathcal{K}\) is the sequence of known-state sets produced by the successive state-refresh operations of the cognitive system from the initial moment onward. Formally:
\[
\mathcal{K} = (K_0, K_1, K_2, \ldots, K_t, K_{t+1}, \ldots)
\]
where each \(K_\tau\) is the known-state set at time \(\tau\), given by Definition~2.5: \(K_\tau = \bigcup_{j=0}^{\tau} \{i_j\}\), satisfying the recurrence relation \(K_{\tau+1} = K_\tau \cup \{i_{\tau+1}\}\).

The core property of the known-state sequence derives from the inclusion relations among its constituent sets:
\[
K_0 \subset K_1 \subset K_2 \subset \cdots \subset K_t \subset K_{t+1} \subset \cdots
\]
This is a strictly monotonic increasing chain of proper inclusions: each subsequent set properly includes the previous one (because it contains a newly added internal object \(i_{\tau+1} \notin K_\tau\)). This property is strictly isomorphic, in mathematical structure, to von Neumann's construction of the ordinals:

\textbf{Von Neumann ordinals:} \(0 = \emptyset\), \(1 = \{\emptyset\}\), \(2 = \{\emptyset, \{\emptyset\}\}\), \(\ldots\), \(n+1 = n \cup \{n\}\). The order relation among ordinals is defined by the set-inclusion relation: \(n < m \iff n \in m \iff n \subset m\).

\textbf{Known-state sequence:} \(K_{t+1} = K_t \cup \{i_{t+1}\}\). The sequence relation is defined by the set-inclusion relation: \(K_t\) is ``before'' \(K_{t+1}\) \(\iff\) \(K_t \subset K_{t+1}\).

This isomorphism reveals the essential function of the known-state sequence: \textbf{it is a ``time-stamp generator'' immanent in the cognitive system.} Just as von Neumann ordinals require no external time parameter to define ``before and after''---the order relation is immanent in the set construction itself---the known-state sequence likewise requires no appeal to any external clock to order states. The proper-inclusion relation among sets itself provides an irreversible, unambiguous temporal order: if \(K_a \subset K_b\), then \(K_a\) must have been constructed before \(K_b\), because \(K_b\) contains all elements of \(K_a\) and additionally contains sensor outputs from subsequent moments.
\end{definition}

\subsubsection{The Cognitive Function of the Known-State Sequence: Defining ``Unknown State''}

The most fundamental cognitive function of the known-state sequence is not to provide direct material for behavior planning, but rather to \textbf{provide a set-theoretic foundation for the definition of ``unknown state.''} This function is realized through the operation of set difference.

At time \(t\), the system possesses the known-state set \(K_t\). At time \(t+1\), after sensor refresh, the system possesses \(K_{t+1}\). At this point, the system can, through set subtraction, precisely identify the internal object that ``was unknown at time \(t\) and became known only at time \(t+1\)'':
\[
U(K_{t+1} \mid K_t) = K_{t+1} \setminus K_t = \{i_{t+1}\}
\]

More generally, for any two times \(a < b\):
\[
U(K_b \mid K_a) = K_b \setminus K_a = \{i_{a+1}, i_{a+2}, \ldots, i_b\}
\]

This is the set-theoretic definition of the ``relative unknown state'' (Definition~2.7). Its theoretical significance is this: \textbf{``unknown'' is not an intrinsic property of an internal object, but a set-difference relation between two known-state sets.} \(i_{t+1}\) is ``unknown'' not because it carries some mysterious ``unknown marker,'' but because it belongs to \(K_{t+1}\) yet does not belong to \(K_t\). Change the reference set---for example, looking back from time \(t+2\)---and \(i_{t+1}\) is no longer unknown, because it already belongs to \(K_{t+1}\), and \(K_{t+1} \subset K_{t+2}\).

This construction reveals a crucial generative fact: \textbf{without the known-state sequence, there is no concept of ``unknown state.''} For ``unknown'' is rigorously defined as the set difference between known-state sets, and the set-difference operation requires two known-state sets as operands. The known-state sequence provides exactly these operands---it is a temporally ordered sequence of known-state sets; any two sets at any two positions can be extracted and subjected to set-difference, thereby producing an ``unknown-state set'' relative to the earlier set. Thus, the known-state sequence is logically prior to the definition of the unknown state: there must first be accumulation and preservation of known-state sets before one can retrospectively identify, through set comparison, ``what was previously not known.''

\subsubsection{Rigorous Definition of the State Sequence and Its Distinction from the Known-State Sequence}

\begin{definition}[State Sequence]
The state sequence \(\mathcal{S}\) is the sequence of contents written into the ``state-refresh sequence'' storage structure at each state-refresh event. According to Operation~1 of Definition~3.4, when state refresh is triggered at time \(t+1\), what is written is the known-state set \textit{before the refresh} \(K_t\)---i.e., the complete cognitive snapshot of the system before recognizing the new input. Hence:
\[
\mathcal{S} = (K_0, K_1, K_2, \ldots, K_t, \ldots)
\]
where each entry \(K_\tau\) is the known-state set that was written at the state-refresh event at time \(\tau\).
\end{definition}

Extensionally, the content of the state sequence is identical to that of the known-state sequence---both are sequences of known-state sets. However, this extensional coincidence conceals a fundamental intensional difference. The two kinds of sequences play entirely distinct functional roles in the cognitive architecture, as summarized in the following table:

\begin{center}
\begin{tabular}{p{3cm} p{5.5cm} p{5.5cm}}
\hline
\textbf{Dimension} & \textbf{Known-State Sequence \(\mathcal{K}\)} & \textbf{State Sequence \(\mathcal{S}\)} \\
\hline
Mode of definition & Defined directly by the recursive construction of known-state sets & Defined by the ``write to sequence'' act within the state-refresh operation \\
\hline
Generative mechanism & The natural expansion process of known-state sets & The active write operation at each state refresh \\
\hline
Core function & Provides operands for defining ``unknown state'' (basis of set-difference) & Provides the action-planning function with retrievable, comparable state history \\
\hline
Time-stamp capability & Possesses it---set-inclusion relations provide unambiguous before-after order & Does \textit{not} possess independent time-stamp capability---its order depends on the inclusion relations of the known-state sequence \\
\hline
Relation to behavior planning & Indirect---not directly used by \(f\), but provides set-theoretic foundations for meta-cognition & Direct---retrieved and invoked by \(f\) as historical material for prediction and decision \\
\hline
\end{tabular}
\end{center}

The crux of this distinction lies in the source of the ``time stamp.'' Each entry \(K_\tau\) in the state sequence \(\mathcal{S}\), taken as an isolated data structure, \textbf{carries no information whatsoever about its own position in the sequence.} \(K_\tau\) itself is merely a set---it contains all internal objects up to time \(\tau\), but it has no internal marker capable of indicating ``\(\tau\) is the \(\tau\)-th moment.'' An external observer (or the system's own retrieval mechanism), looking only at an isolated entry in \(\mathcal{S}\), cannot tell whether it comes before or after another entry.

The temporal order of the state sequence---which entry is ``earlier'' and which ``later''---is provided by the set-inclusion relations of the known-state sequence \(\mathcal{K}\). Concretely, for two entries \(K_a\) and \(K_b\) in \(\mathcal{S}\):

\begin{itemize}
\item If \(K_a \subset K_b\), then \(K_a\) is before \(K_b\) (because \(K_b\) contains all elements of \(K_a\) and more);
\item If \(K_b \subset K_a\), then \(K_b\) is before \(K_a\);
\item If neither is included in the other, then they do not come from the same known-state sequence of the same cognitive agent (a situation to be discussed in detail in Section~3.4).
\end{itemize}

Thus, the set-inclusion relations of the known-state sequence \(\mathcal{K}\) are the \textbf{sole source} of the temporal structure of the state sequence \(\mathcal{S}\). Without the inclusion chain of the known-state sequence, the entries in the state sequence could not be ordered. In this sense, the known-state sequence is the ``infrastructure of time''---it provides the topological structure of time (the before-after relation); the state sequence is the ``filler of time''---it provides the concrete cognitive snapshots distributed within that topological structure.

\subsubsection{The Relation between the Two Kinds of Sequences: Time Stamps Provided and Received}

Based on the above distinction, the relation between the two kinds of sequences can be precisely stated as: \textbf{the known-state sequence provides time stamps for the state sequence.} The state sequence itself does not possess the capacity of ``distinguishing its own empty-set-construction-style correspondence to the natural numbers''---this is precisely the key point emphasized throughout.

The so-called ``empty-set-construction-style correspondence to the natural numbers'' refers to the mechanism in von Neumann ordinals by which each ordinal immanently carries its position through set-inclusion relations: \(3 = \{0, 1, 2\}\) is ``3'' precisely because it contains all ordinals smaller than it. This construction makes the ``numerical value'' of an ordinal require no external label---the inclusion relation itself provides the order information.

The known-state sequence possesses exactly this mechanism: \(K_{t+1}\) contains \(K_t\), and \(K_t\) contains \(K_{t-1}\), and so on. Hence, \(K_{t+1}\) ``knows,'' in the set-theoretic sense, that it is ``greater'' than \(K_t\)---because \(K_t\) is a proper subset of it.

But an entry \(K_t\) in the state sequence---when viewed in isolation as a retrieval target of the action-planning function---does not naturally carry this information. The system can store \(K_t\) as an independent data structure without recording ``through which state-refresh event it was written.'' In this case, the temporal position of \(K_t\) must be determined by tracing back to the known-state sequence and comparing its inclusion relations with other \(K\)'s.

This division of labor has profound significance in the cognitive architecture. It means that \textbf{time is an internally constructed product of the cognitive system, not an externally given dimension.} Concretely: the ``before-after'' structure of time is generated by the inclusion relations of known-state sets (a purely logical relation requiring no physical clock); the ``content'' of time---i.e., ``what the world was like at a given moment''---is provided by the cognitive snapshots in the state sequence. Together, they constitute the complete temporal experience of the cognitive agent.

\subsection{Formalization of the Relative Unknown State and Its Architectural Constraints}

In Section~3.3, we distinguished the known-state sequence from the state sequence and clarified that the former provides time stamps for the latter. Building on this distinction, this section provides a complete formal treatment of the ``relative unknown state'' and argues for two important architectural constraints: first, that within a single cognitive agent, the unknown must be relative (i.e., can only be defined by comparing known-state sets at different times); second, that the unknown states of different cognitive agents are incomparable.

\subsubsection{General Formalization of the Relative Unknown State}

\begin{definition}[General Form of the Relative Unknown-State Set]
Let a cognitive agent possess a known-state sequence \(\mathcal{K} = (K_0, K_1, \ldots, K_t, \ldots)\). For any two times \(a\) and \(b\) (\(a < b\)), the \textbf{relative unknown-state set} of \(K_b\) with respect to \(K_a\) is defined as:
\[
U(K_b \mid K_a) = K_b \setminus K_a
\]
i.e., the set of all internal objects belonging to \(K_b\) but not belonging to \(K_a\).

This definition depends on a crucial precondition: \(K_a \subset K_b\). This inclusion relation is guaranteed by the monotonic increasing property of the known-state sequence (the recursive construction of Definition~2.5)---since \(a < b\), the sensor continuously outputs new internal objects from time \(a+1\) to time \(b\), and each new object is cumulatively added to the known-state set; therefore, \(K_b\) necessarily contains all elements of \(K_a\) plus additional elements. If \(K_a \not\subset K_b\), then the set difference \(K_b \setminus K_a\), while still computable in a purely set-theoretic sense, does not represent ``the unknown of \(K_b\) relative to \(K_a\)''---because it would mean that there are some elements in \(K_a\) not in \(K_b\), which is impossible within the known-state sequence of a single cognitive agent (known-state sets only grow, never shrink).
\end{definition}

Thus, the definition of the relative unknown state implies an \textbf{integrity constraint}:

\begin{axiom}[Integrity Constraint of the Relative Unknown State]\label{axiom:integrity}
For any two times \(a < b\) within the same cognitive agent, to define the relative unknown state in \(K_b\) with \(K_a\) as baseline, the condition \(K_a \subset K_b\) must be satisfied---i.e., all known states in \(K_a\) must be completely contained in \(K_b\). If this constraint is not satisfied, then \(U(K_b \mid K_a)\) is not recognized as a valid ``relative unknown state''---it is merely a set-theoretic set difference devoid of cognitive meaning.
\end{axiom}

This constraint ensures the cognitive fidelity of ``unknown'': the system cannot ``arbitrarily'' compare two unrelated sets and then claim that one contains the ``unknown'' relative to the other. The unknown must be an expansion of the known---it must occur within the monotonic increasing historical process of the known-state set.

\subsubsection{The Cognitive Mechanism of Internal Unknown: Comparing Earlier and Later Known States}

Based on the above definitions and constraints, the mechanism by which a cognitive agent internally recognizes ``the unknown'' can be precisely described.

At time \(t\), the system possesses the known-state set \(K_t\). At this moment, the system does not yet know \(i_{t+1}\)---it is not in \(K_t\). At time \(t+1\), the sensor refreshes, the state-refresh operation completes, and the system possesses \(K_{t+1} = K_t \cup \{i_{t+1}\}\).

At this point, the system can execute the set-difference operation:
\[
U(K_{t+1} \mid K_t) = K_{t+1} \setminus K_t = \{i_{t+1}\}
\]

The result of this operation tells the system, unequivocally: ``\(i_{t+1}\) is the internal object that I did not know when I only had \(K_t\), and have only now come to know through sensor input.'' This is the cognitive-generative mechanism of ``the unknown''---it is not directly perceived (because what is unknown cannot be directly perceived), but is \textbf{retrospectively} confirmed by comparing the earlier and later known-state sets and identifying the newly appeared element.

The deeper significance of this mechanism is this: \textbf{a cognitive system can never ``directly know'' what it does not know.} To confirm the existence of an unknown state, the system must first possess it (through sensor refresh transforming it into a known state), and then recognize, through comparison, that ``this thing was not previously in my known set.'' Hence, knowledge about ``ignorance''---i.e., meta-cognition---is, in the generative order, posterior to knowledge about ``knowledge.'' The system first knows what the world is like, and only afterwards can it discover, through comparison of the known-state sequence, that its knowledge was once incomplete.

\subsubsection{Incomparability of Unknown States across Different Cognitive Agents}

The above analysis naturally leads to an important corollary: the ``unknown states'' of different cognitive agents are incomparable.

\begin{theorem}[Incomparability of Cross-Agent Unknown]\label{thm:cross_agent}
Let there be two cognitive agents \(\mathcal{A}\) and \(\mathcal{B}\), possessing respectively their own known-state sequences \(\mathcal{K}^{\mathcal{A}}\) and \(\mathcal{K}^{\mathcal{B}}\). Then there exists no common frame of reference such that a relative unknown state \(U^{\mathcal{A}}(K_b^{\mathcal{A}} \mid K_a^{\mathcal{A}})\) of \(\mathcal{A}\) can be meaningfully compared or equated with a relative unknown state \(U^{\mathcal{B}}(K_d^{\mathcal{B}} \mid K_c^{\mathcal{B}})\) of \(\mathcal{B}\).
\end{theorem}

\begin{proof}
The proof rests on the following three premises, each rigorously established earlier in this paper.

\textbf{Premise 1:} All known states of any cognitive agent derive from the output of its sensor (Theorem~2.6---the sole source of known states). Every element in \(\mathcal{A}\)'s known-state set \(K^{\mathcal{A}}\) is an internal object output by \(\mathcal{A}\)'s sensor \(S^{\mathcal{A}}\) at some moment. Likewise for \(\mathcal{B}\).

\textbf{Premise 2:} The sensor is the sole causal channel between the external world and internal objects (Definition~2.3), and its internal operation is unknowable to the cognitive agent (Axiom~2.3). This means that \(\mathcal{A}\) cannot access the sensor output of \(\mathcal{B}\), and vice versa. Each cognitive agent is strictly enclosed within its own sensory interface.

\textbf{Premise 3:} The unknowability of the external world (Axioms~2.1--2.2) precludes the possibility of bridging the two agents by appealing to a ``common external object.'' Even if one assumes that \(\mathcal{A}\) and \(\mathcal{B}\) operate in ``the same external world,'' this very assumption is unverifiable within the cognitive architecture---no cognitive agent can step outside itself to compare its internal objects with external objects.

From these three premises, the incomparability of cross-agent unknown follows.

First, \(U^{\mathcal{A}}(K_b^{\mathcal{A}} \mid K_a^{\mathcal{A}})\) is a set constituted by \(\mathcal{A}\)'s internal objects---all its elements come from the sensor output of \(\mathcal{A}\). Likewise, \(U^{\mathcal{B}}(K_d^{\mathcal{B}} \mid K_c^{\mathcal{B}})\) comes entirely from the sensor output of \(\mathcal{B}\). The two belong to distinct internal-object spaces \(\mathcal{I}^{\mathcal{A}}\) and \(\mathcal{I}^{\mathcal{B}}\).

Second, there exists no common mapping or function capable of judging the identity between an element of \(\mathcal{I}^{\mathcal{A}}\) and an element of \(\mathcal{I}^{\mathcal{B}}\). For any such mapping would require access to the outputs of both sensors, whereas each cognitive agent can only access the output of its own sensor (Premise~2). The unknowability of the external world (Premise~3) further rules out indirect comparison via external objects.

Finally, even if, from the perspective of an external observer (such as a theorist), \(\mathcal{A}\) and \(\mathcal{B}\) might be ``responding to the same external event,'' this observation cannot be converted into an internal cognitive operation of either cognitive agent. For \(\mathcal{A}\), \(\mathcal{B}\)'s unknown state---\(U^{\mathcal{B}}(K_d^{\mathcal{B}} \mid K_c^{\mathcal{B}})\)---is inaccessible, and therefore incomparable.
\end{proof}

\subsubsection{Consequences of the Inability to Confirm ``Complete Inclusion of Known States''}

Theorem~\ref{thm:cross_agent} reveals a deeper architectural constraint: when two known-state sets come from different cognitive agents (or cannot be confirmed to come from the same agent), \textbf{it is impossible to confirm ``whether known states are completely included.''} This directly entails the incomparability of cross-agent time stamps.

In Section~3.3.3, we argued that the temporal order of the state sequence depends on the set-inclusion relations of the known-state sequence. Two known-state sets \(K_a\) and \(K_b\), if satisfying \(K_a \subset K_b\), place \(K_a\) before \(K_b\). The absolute precondition for this judgment to hold is that \(K_a\) and \(K_b\) come from the same known-state sequence of the same cognitive agent. Only under this condition does the inclusion relation equate to temporal priority.

If \(K^{\mathcal{A}}\) and \(K^{\mathcal{B}}\) come from different cognitive agents, then even if, by coincidence, \(K^{\mathcal{A}} \subset K^{\mathcal{B}}\) (e.g., \(\mathcal{A}\)'s sensor happened to output only a subset of what \(\mathcal{B}\)'s sensor output), this inclusion relation cannot be interpreted as temporal priority---because it was not generated by the sensor-refresh sequence within a single cognitive agent. The inclusion relation loses its temporal semantics, degenerating into a purely extensional set-theoretic fact devoid of any cognitive meaning.

This analysis confirms the crucial point emphasized throughout: \textbf{``temporal uniqueness'' is guaranteed by the decidable condition of ``whether known states are completely included.''} Within a single cognitive agent, the monotonic increasing property of the known-state sequence guarantees that the known-state sets at any two moments necessarily satisfy the proper-inclusion relation, so temporal order is unique and unambiguous. But once the boundary of a single cognitive agent is crossed, this guarantee disappears. Set-inclusion relations can no longer be automatically assumed, and temporal order can no longer be uniquely determined.

\subsubsection{Provisional Distinction between Virtual State Sequences and Virtual Known-State Sequences}

In the cognitive architecture, the system processes not only its own direct perceptual experience, but also information coming from others---social cognition, linguistic communication, substitution-based simulation, etc. How are these ``experiences of others'' internally represented by the system? This involves an important conceptual distinction---that between ``virtual state sequences'' and ``virtual known-state sequences.'' Although the detailed exposition belongs to the domain of the social-cognition subsystem (to be treated in subsequent chapters), the basic delineation of the two can be given in advance based on the conclusions of this chapter.

\textbf{Virtual known-state sequences} are known-state sequences that the cognitive agent \textbf{simulatively constructs} for another cognitive agent (or a hypothetical cognitive agent). For example, when the system observes another individual gazing at some object, the system can infer that this individual's sensor is receiving external input related to that object, and accordingly add the corresponding internal object to that individual's ``virtual known-state set.'' The crucial property of virtual known-state sequences is this: they are simulations of \textbf{the internal states of another}, and therefore they inherit the core function of ordinary known-state sequences---in particular, the \textbf{time-stamp function}. The system can, through the set-inclusion relations of virtual known-state sequences, confer temporal order on the simulated experience of the other.

\textbf{Virtual state sequences} are different. They are sequences of state changes that the cognitive agent \textbf{imagines or inferentially unfolds} within itself, without any direct link to actual or simulated sensor input. For example, when the system plans future behavior, it internally infers ``if I execute action \(a\), what will happen next''---the series of state snapshots produced by this inference constitutes a virtual state sequence. The characteristic feature of virtual state sequences is this: they \textbf{carry no time stamps}---their entries have no set-inclusion relations among them, because each entry is a hypothetical construct that does not accumulate actual sensor outputs. The temporal order of a virtual state sequence depends entirely on external labeling (e.g., ``this is step one, this is step two'') or on its embedding position within the cognitive agent's own known-state sequence (e.g., ``I performed this inference at time \(t\)'').

The cognitive significance of this distinction is that \textbf{time-stamp capability belongs exclusively to known-state sequences} (whether real or virtual), because it depends on the inclusion relations of known-state sets---which, in turn, depend on the cumulativity of sensor output. The state sequence itself does not possess this capability, and virtual state sequences, as extensions of the state sequence, likewise do not possess it. When the system compares information from different sources---for example, checking its own experience against another's report---it must rely on the time stamps of known-state sequences (or virtual known-state sequences) to align event orders, and cannot directly compare isolated entries in state sequences. This point will be further developed in subsequent chapters on social cognition and language comprehension.

% ==========================================
% CHAPTER 4: DEMAND AS THE INTRINSIC CONSTRAINT OF THE ACTION-PLANNING FUNCTION
% ==========================================

\section{Demand as the Intrinsic Constraint of the Action-Planning Function}

In the preceding three chapters, we completed the generative derivation from the consistency requirement to state sequences. The cognitive system was delineated as a closed architecture possessing a sensor (independent variable) and an action-planning function (dependent variable), capable of cumulatively expanding its known-state set through state refresh in time, and outputting behavioral plans on the basis of known states. Yet this architecture has so far lacked a crucial dimension: \textbf{demand}. Why does the action-planning function select one behavior rather than another? What constrains the direction of its output?

This chapter provides the formal definition of demand and demonstrates that demand is not an add-on module external to the action-planning function, but a constraint structure that any action-planning function necessarily embeds as a mathematical inevitability. The argumentative strategy is as follows. In Section~4.1, we prove rigorously, from a mathematical standpoint, that any deterministic action-planning function is equivalent to an optimization process that takes itself as the optimality standard---i.e., it embeds a minimal constraint. Any attempt to remove this constraint and construct a ``demand-free'' function will, mathematically, inevitably degenerate into a pure random process or an indefinable arbitrary mapping, thus forfeiting its cognitive qualification as ``behavioral planning.'' Then, in Section~4.2, based on this mathematical conclusion, we provide the generative definition of demand and elucidate its theoretical status within the cognitive architecture---demand is not an added value preference, but the formal condition for ``a function to exist as a function.''

\subsection{The Constraint Inevitability of the Action-Planning Function: A Formal Demonstration}

This section aims to prove the following core proposition: \textbf{any deterministic mapping capable of serving as the action-planning function of a cognitive agent necessarily embeds a minimal constraint structure that takes its own output as the unique optimal solution. If one attempts to remove this constraint, the function either ceases to be deterministic (degenerating into a random process) or ceases to be ``planning'' (losing the stable mapping from input to output).} In other words, a ``demand-free action-planning function'' is a mathematical impossibility---it is not something that a designer accidentally omitted, but something that is excluded by the very mathematical nature of a function.

\subsubsection{Mathematical Description of the Action-Planning Function}

From Definition~2.4, the action-planning function \(f\) is a mapping from the known-state set \(K_t\) to a behavioral plan \(A_t\):
\[
A_t = f(K_t)
\]
where \(K_t \subseteq \mathcal{I}\) is the known-state set up to time \(t\), \(\mathcal{I}\) is the internal-object space; \(A_t \in \mathcal{A}\) is a behavioral plan, and \(\mathcal{A}\) is the action space. The determinacy of \(f\) means: for the same input \(K_t\), \(f\) always outputs the same \(A_t\). This is a standard function---given an element in the domain, there is exactly one corresponding element in the codomain.

The crucial question is: does this mapping \(f\) itself imply some kind of ``constraint''? More precisely, does there exist a mathematical structure that equivalently redescribes the mode of operation of \(f\) as ``selecting the optimal output under some evaluative standard,'' without altering the input-output behavior of \(f\)?

\subsubsection{Objectifying a Function as a Standard: Construction of the Minimal Constraint}

We first prove that, for any given deterministic function \(f\), there exists an evaluation function that takes \(f\) itself as the optimal solution.

\begin{theorem}[Standard Construction Theorem for Functions]\label{thm:standard_construction}
Let \(f: \mathcal{X} \to \mathcal{Y}\) be any deterministic function, where \(\mathcal{X}\) is the input space and \(\mathcal{Y}\) is the output space. Then there exists an evaluation function \(F: \mathcal{X} \times \mathcal{Y} \to \mathbb{R}\) such that, for any \(x \in \mathcal{X}\), \(f(x)\) is the unique global maximum point of \(F(x, \cdot)\) over \(\mathcal{Y}\). Concretely, one can construct:
\[
F(x, y) = -\| y - f(x) \|^2
\]
where \(\|\cdot\|\) is any norm on \(\mathcal{Y}\) (if \(\mathcal{Y}\) is a discrete space, one may define \(\| y - f(x) \|^2\) as the indicator function: \(0\) when \(y = f(x)\), and \(1\) otherwise).
\end{theorem}

\begin{proof}
For any given \(x \in \mathcal{X}\), examine the function \(y \mapsto -\| y - f(x) \|^2\). This function attains its maximum value of \(0\) at, and only at, \(y = f(x)\) (in the discrete case, the maximum \(0\) is attained if and only if \(y = f(x)\); in the continuous case, the negative of the squared norm attains its maximum at the unique minimum point). Hence, \(f(x)\) is the unique global maximum point of \(F(x, \cdot)\) over \(\mathcal{Y}\).
\end{proof}

This construction has an extremely simple geometric intuition: it ``stretches'' the input-output relation of the function \(f\) into a ``pit'' centered at \(f(x)\). In the output space \(\mathcal{Y}\), points closer to \(f(x)\) score higher, with \(f(x)\) itself being the highest point. Hence, the operation of finding the maximum of \(F\) is, in behavioral terms, completely equivalent to directly computing \(f(x)\). Both produce the same output.

Theorem~\ref{thm:standard_construction} is mathematically not deep---it is almost a trivial tautology. Yet its theoretical significance for the cognitive architecture is profound: \textbf{it proves that ``constraint'' is not an alien addition to a function, but an equivalent redescription of the function's operational logic.} ``Objectifying'' \(f\) as an evaluative standard---i.e., constructing a landscape function that takes \(f\) itself as the optimal solution---requires no injection of external information and no modification of the internal structure of \(f\). It is a mathematical identity: the mapping relation of \(f\) itself defines an optimization problem that takes itself as the standard.

\begin{definition}[Minimal Demand Constraint]
For an action-planning function \(f\), the evaluation function \(F_f(K, A) = -\| A - f(K) \|^2\) constructed by Theorem~\ref{thm:standard_construction} is called the \textbf{minimal demand constraint} of \(f\).
\end{definition}

The meaning of ``minimal'' is this: this constraint adds no preference beyond the original mapping relation of \(f\). It only requires that ``the output should be \(f(K)\),'' without providing any external reason for ``why it should be.'' It is a purely formal constraint---it only stipulates that a function ought to be itself, without stipulating what a function ought to be.

\subsubsection{Demonstration of the Inconstructibility of an Unconstrained Function}

Theorem~\ref{thm:standard_construction} proves from the positive side that any deterministic function can be equivalently redescribed as embedding a minimal constraint. We now argue from the negative side: what would happen if one attempted to construct an action-planning function ``completely devoid of any demand constraint''?

The phrase ``completely devoid of any demand constraint'' can be understood in two ways.

\textbf{Understanding 1: The function \(f\) is deterministic, but refuses to accept the minimal constraint that takes itself as the standard.} This understanding is mathematically untenable. For the minimal constraint is an equivalent redescription of the function's operational logic---it is mathematically necessarily derived from the mapping relation of \(f\). Refusing this constraint is equivalent to refusing that ``the output of \(f\) should equal the output of \(f\)''---a logical self-referential contradiction. A function cannot ``refuse'' to be itself.

\textbf{Understanding 2: The function \(f\) itself is not deterministic---i.e., for the same input \(K_t\), the system may output different \(A_t\), and this variability is not determined by any internal state or constraint, but is pure, unpatterned randomness.} In this case, there exists no stable mapping relation \(f\) that can be called an ``action-planning function.'' The system's behavioral output is a random process, not a function.

Let us rigorously examine the situation of Understanding~2. Suppose the system's behavioral output is given not by any deterministic function, but by some random mechanism: for input \(K_t\), the output is a random variable \(\mathbf{A}_t\) with distribution \(P(\mathbf{A}_t = A \mid K_t)\). The existence of this distribution is itself indisputable---any behavioral output process can be described as a probability distribution. But the crucial question is: whence comes this distribution? How is it internally determined by the system?

If this distribution is completely arbitrary---i.e., it is not constrained by any internal system parameters, nor generated by any optimization process---then its existence is equivalent to ``there is no action-planning function.'' For the semantic core of the concept ``action-planning function'' is a \textit{stable rule} mapping inputs to outputs. A completely arbitrary distribution does not constitute a rule---it is merely a random number generator.

If this distribution itself is determined by some internal mechanism---e.g., the system possesses an internal probabilistic model that computes the probabilities of different actions based on \(K_t\)---then this internal mechanism itself constitutes a mapping from \(K_t\) to a distribution, and this mapping can in turn be subjected to the minimal constraint (Theorem~\ref{thm:standard_construction} holds for this mechanism). Thus, randomness does not eliminate the constraint, but merely shifts it onto the meta-mechanism that determines the shape of the distribution.

We therefore conclude: \textbf{any cognitive mechanism that can be characterized as an ``action-planning function'' necessarily embeds a minimal demand constraint. If there is completely no constraint, then either one falls into a logical self-referential contradiction (refusing to be itself), or one degenerates into a pure random process that cannot be characterized as a ``function.''} Demand is not something that a cognitive system can ``choose to have'' or ``choose to give up''---it is the mathematically inevitable product of the very fact of ``having an action-planning function.''

\subsubsection{The Cognitive Meaning of the Minimal Constraint}

What does the minimal constraint \(F_f(K, A) = -\| A - f(K) \|^2\) express at the cognitive level? It expresses the most basic consistency requirement of the action-planning function: \textbf{the behavior output by the system ought to be consistent with the system's internal decision logic.} This is not an external moral norm, nor a heuristically effective strategy proven by experience. It is the logical precondition for a ``behavioral output'' to be called a ``function output.''

Under the minimal constraint, the system pursues no particular goal---it only pursues ``being itself.'' This, of course, does not mean that the system's actual output \(A_t\) at time \(t\) necessarily equals \(f(K_t)\)---since \(f\) itself may be defective, or the system may err due to noise. The minimal constraint is not a description of the system's actual performance, but a formalization of the system's operational norm: the system's operational logic is to select that \(A\) which maximizes \(F_f(K, A)\). If the system deviates from this logic, then it is no longer the system described by ``this \(f\).''

The ``minimality'' of this constraint ensures that it presupposes no particular value content. It provides a purely formal substrate for the concept of demand---a blank structure capable of bearing any concrete demand content. In subsequent chapters, we will see how somatic feedback superimposes positive and negative weights onto this blank structure, how social cognition superimposes admiration and hatred annotations, and how language decoding superimposes exclusion markers. All these concrete demands are realized by adding extra evaluative terms on top of the minimal constraint, while the minimal constraint itself ensures that the basic integrity of the action-planning function is not compromised.

Positioning demand as ``a mathematical inevitability embedded in the action-planning function'' allows a clear contrast between PST's concept of demand and the ``reward function'' in reinforcement learning, as well as the ``free energy'' in the free-energy principle. In reinforcement learning, reward is a scalar signal externally supplied by the environment, and the agent's goal is to maximize cumulative reward \cite{Sutton2018}. From whence reward comes, and why it has positive or negative polarity---these questions belong to the designer's stipulations, not to the agent's internal generative process. PST's demand is fundamentally different: the minimal demand constraint \(F_f(K, A) = -\|A - f(K)\|^2\) is not externally conferred, but necessarily derived from the determinacy of the action-planning function. Even ``pure computation''---i.e., faithful computation that involves no somatic feedback whatsoever---embeds this constraint. External reward (corresponding, in PST, to the assignment of demand weights by somatic feedback) is content added later on top of this minimal substrate, not the origin of demand.

The contrast with the free-energy principle is equally illuminating. Friston (2009, 2010) proposed that all adaptive systems can be understood as minimizing variational free energy---a concept that unifies perception, action, and learning under an information-theoretic framework. However, whether free-energy minimization, as a unifying explanatory principle, is descriptive or normative in its theoretical status has often been a matter of debate in the literature. By reducing ``demand'' to a mathematical inevitability of the existence of a function, PST offers a new perspective on this debate: any deterministic system necessarily follows an optimality condition that takes itself as the standard, and free-energy minimization---or an equivalent form thereof---can be viewed as a particular manifestation of this inevitability in probabilistic systems. PST does not deny the explanatory power of the free-energy principle, but provides it with a deeper, more universal generative foundation.

\subsection{Generative Definition and Cognitive Status of Demand}

Section~4.1 proved mathematically that any action-planning function necessarily embeds a minimal constraint. This section, based on that mathematical conclusion, provides the generative definition of ``demand'' within Predictive Set Theory and elucidates its theoretical status within the cognitive architecture.

\subsubsection{Generative Definition of Demand}

\begin{definition}[Demand]
In Predictive Set Theory, \textbf{demand} is the formal expression of the constraint condition that the action-planning function must satisfy. Concretely, demand is defined as an evaluation function:
\[
\Phi: \mathcal{P}(\mathcal{I}) \times \mathcal{A} \to \mathbb{R}
\]
where \(\mathcal{P}(\mathcal{I})\) is the power set of the known-state set, and \(\mathcal{A}\) is the behavioral-plan space. The output of the action-planning function \(f\) is given by the following optimization process:
\[
f(K_t) = \arg\max_{A \in \mathcal{A}} \Phi(K_t, A)
\]
That is, the system selects, among all possible behavioral plans, the one that maximizes \(\Phi\).

Under this definition, the minimal constraint \(F_f\) of Section~4.1 constitutes a special case of demand---namely, \textbf{minimal demand}. It guarantees that the action-planning function maintains self-consistency, but imposes no preference beyond its own logic. More general demand functions \(\Phi\) can add extra evaluative terms on top of minimal demand:
\[
\Phi(K_t, A) = -\| A - f_0(K_t) \|^2 + \sum_{d \in D} w_d \cdot \phi_d(K_t, A)
\]
where \(f_0\) is the system's ``basic computational logic'' (pure computation), \(D\) is the set of demand states, \(w_d\) are demand weights, and \(\phi_d\) are evaluation functions associated with particular demands. This form explicitly places demand in the role of ``constraint on behavioral output,'' while retaining minimal demand as the substrate.
\end{definition}

\subsubsection{The Generative Inevitability of Demand}

The generative inevitability of demand comes from two sources: mathematical and cognitive.

\textbf{Mathematical inevitability} has been demonstrated in detail in Section~4.1: any deterministic function embeds a minimal constraint. The action-planning function, as a mapping from known states to behavioral plans, cannot possibly be an exception to this rule. Hence, demand is not an independent module added on top of the action-planning function, but a formal condition of the action-planning function's existence---just as ``having boundaries'' is a formal condition of ``having a geometric figure.''

\textbf{Cognitive inevitability} comes from the functional role of the action-planning function within the cognitive architecture. The task of the action-planning function is to output a behavioral plan---i.e., to select one among many possible behaviors. Selection presupposes preference---without preference, selection would be arbitrary, and arbitrary selection is not planning, but randomness. Preference must have some formal expression, and this expression is precisely the demand function \(\Phi\). Even if the system only prefers ``to output behavior consistent with its own computational logic,'' this preference already constitutes a demand---namely, minimal demand. Hence, if a cognitive system is to output behavioral plans, it must necessarily possess demand.

These two sources of inevitability jointly establish a crucial conclusion: \textbf{demand and the action-planning function are two sides of the same coin.} There is no generative sequence of ``first there is an action-planning function, and then it is equipped with demand.'' At the very moment the action-planning function is defined, demand---at least minimal demand---is already embedded within it as its mathematical property. The action-planning function is the structured expression of demand; demand is the operational form of the action-planning function.

\subsubsection{Distinction between Demand and Positive/Negative Feedback}

It is essential to strictly distinguish ``demand'' from ``positive/negative feedback.'' This distinction has foundational status in Predictive Set Theory; conflating the two would lead to a category error.

\textbf{Demand} is the constraint form of the action-planning function---it is the formal condition of ``what the system ought to output.'' It does not answer ``why ought,'' but only ``ought what.'' In the case of minimal demand, the answer to ``ought what'' is ``ought to output behavior consistent with itself''---which is almost a logical truth.

\textbf{Positive/negative feedback} is, by contrast, the \textbf{content anchoring} of demand. It answers ``why ought this rather than that.'' Positive feedback marks events that bring the system closer to an ideal state; negative feedback marks events that deviate the system from an ideal state. The ultimate source of positive/negative feedback is somatic feedback---pain, hunger, pleasure, and other primordial signals that cannot be revised by higher-order cognition. These signals provide the concrete numerical values and signs for the demand weights \(w_d\) in the demand function.

Thus, demand is form, and positive/negative feedback is content. Demand stipulates that behavioral selection must satisfy some optimality condition; positive/negative feedback stipulates the concrete parameters of that optimality condition. Without the form of demand, positive/negative feedback could find no expression---because there would be no carrier of ``ought''; without the content of positive/negative feedback, demand would be empty---it would be mere formal self-consistency, lacking the capacity to point toward any concrete goal.

This distinction also helps to understand the relationship between the ``consistency requirement'' and ``demand.'' The consistency requirement (Chapter~1) is the architectural prerequisite that the cognitive system must maintain a self-consistent internal knowledge network---it is the condition that makes decision-making possible. Demand is the concrete operational form of the action-planning function---it is the mechanism by which the system selects behavior on the basis of an already-ensured consistency. The consistency requirement is logically prior to demand: an inconsistent system cannot possess a determinate demand function, because its evaluative standards would contradict themselves. But once consistency is established, demand necessarily emerges as a mathematical property of the action-planning function.

\subsubsection{The Relation between Demand and Pure Computation}

The mathematical demonstration of Section~4.1 has already clearly shown the relation between demand and ``pure computation.'' Pure computation---i.e., deterministic mapping that involves no external values---can be viewed as a degenerate case of demand: its demand function is exactly the minimal constraint, containing no additional evaluative terms.

In this degenerate case, the system's behavioral planning is equivalent to ``faithfully executing its internal computational logic.'' It pursues no goal beyond its own self-consistency. Such a system can be seen as a ``pure logical engine''---it performs computations, but does not care about the effects of its computational results on the external world, and receives no value signals from the external world to revise its computational direction.

However, it must be emphasized that even in this degenerate case, demand \textit{is} present. Pure computation is not ``demand-free,'' but rather ``its demand function happens to equal the minimal constraint.'' This distinction is conceptually crucial: it means that demand is not a complex function that only appears in higher-order cognitive systems, but is already present from the simplest deterministic mapping onward. The difference between a higher-order cognitive system and a simple pure-computation system lies not in the former possessing demand while the latter lacks it, but in the former's demand function containing more evaluative dimensions---dimensions originating from somatic feedback, social learning, linguistic instructions, and other sources of experience.

\subsubsection{Supplement: Self-Limitation in Sensor Information and the Collapse of the Pure-Computation Maximum}

In the preceding discussion, we treated ``the execution function rejecting a behavioral plan'' as one triggering circumstance for lifting the pure-computation restriction. However, this is not the only triggering condition. This section introduces a more basic and more direct circumstance: \textbf{sensor information may contain content about the action-planning function itself being restricted in its output.} This circumstance causes the pure-computation maximum to mathematically collapse to zero---i.e., the optimal output of pure computation becomes unavailable---thereby forcibly requiring the system to lift the pure-computation restriction in order to seek a non-zero demand value.

\paragraph{The Self-Referentiality of Sensor Information: The Action-Planning Function Learns of Its Own Limitation}

The sensor is defined as the sole perceptual input interface of the cognitive system, and its output is an internal object \(i \in \mathcal{I}\). The content of an internal object can, in principle, cover any information capturable by the sensor hardware---including information about the physical state of the external world, as well as information about the state of the cognitive system itself. The latter is not mysterious: at the neurobiological level, proprioception, interoception, and feedback about the state of motor execution (such as ``a muscle is blocked by an external force'') are all standard sensory channels.

One particular class of internal objects directly concerns the output capability of the action-planning function itself. Their semantic content is equivalent to: ``behavioral plan \(A\) cannot be output''---not ``blocked during execution,'' but ``under the current state, the system is prohibited from generating this behavioral plan.''

This kind of information is essentially different, in cognitive processing, from the information ``the execution function rejected the behavior'':

\begin{itemize}
\item Execution-function rejection of a behavior occurs \textit{after} the behavioral plan has already been output. The system first outputs \(A\), and then the execution function feeds back ``\(A\) is not executable.'' This is a two-step process: output first, feedback second.
\item The action-planning function itself being restricted in its output occurs \textit{before} the behavioral plan is output. The system learns, already at the planning stage, that ``\(A\) cannot be output''---this is one of the preconditions of planning, not a post-hoc feedback from execution.
\end{itemize}

From the internal perspective of the cognitive system, the difference between the two is: the former is ``I did it, but was blocked''; the latter is ``I know I cannot do it.'' The latter information enters the decision loop logically earlier---it takes effect before the action-planning function generates its output.

\paragraph{Mathematical Analysis of the Collapse of the Pure-Computation Maximum to Zero}

Under the pure-computation restriction, the choice set of the action-planning function is forcibly confined to the singleton set \(\{f(K_t)\}\), where \(f(K_t)\) is the pure-computation output of the system based on the known-state set \(K_t\). The system's behavioral selection is equivalent to directly computing \(f(K_t)\) and outputting the result. The demand function degenerates to the minimal constraint \(F_f(K, A) = -\| A - f(K) \|^2\), attaining its unique maximum of \(0\) at \(A = f(K)\).

Now suppose that, at time \(t\), the sensor input contains an internal object whose semantic content is equivalent to: ``the behavioral plan \(A^* = f(K_t)\) cannot be output.'' This information, as part of the known-state set \(K_t\), enters the input of the action-planning function.

At this point, pure computation faces a logical dilemma. On the one hand, pure computation requires outputting \(A^*\)---this is determined by the function's own operational logic, and is the unique behavior assigned the maximum by the minimal constraint. On the other hand, \(K_t\) contains the information ``\(A^*\) cannot be output.'' If the system nevertheless outputs \(A^*\), it violates the information about its own limitation contained in the known-state set---it knowingly does what it knows it cannot do, directly contravening the consistency requirement (the system cannot, while knowing that ``\(A^*\) is not feasible,'' simultaneously plan \(A^*\)).

More precisely, the information ``\(A^*\) cannot be output'' is equivalent, at the level of the demand function, to revising the demand value of \(A^*\) to negative infinity (or some extremely low penalty value). The maximum of \(0\) that the pure-computation minimal constraint \(F_f(K, A) = -\| A - f(K) \|^2\) assigns at \(A = A^*\) is overridden or vetoed by this information. At this moment, the pure-computation maximum mathematically collapses to zero---or, more accurately, the maximum point \(A^*\) is removed, and there are no other candidates in the choice set, so the demand-maximization process can produce no valid output whatsoever.

This is the precise analogue of ``data inaccessibility causing a call to fail.'' The action-planning function attempts to query ``what should be output given \(K_t\),'' and its operational logic returns \(A^*\). But \(K_t\) simultaneously contains a piece of meta-data saying ``the call to \(A^*\) has been locked.'' The pure-computation mechanism has no internal resource for processing such meta-data---its sole logic is to output \(f(K)\). Hence, the pure-computation mechanism fails in this circumstance---not an execution failure, but a breakdown of the logical closure of planning itself, punctured from within by internal information.

\paragraph{The Mathematical Inevitability of Lifting the Pure-Computation Restriction}

Faced with the collapse of the pure-computation maximum, the system must seek a non-zero source of demand value. This means that the choice set must expand from the singleton set \(\{A^*\}\) to a set containing other candidate behaviors.

\textbf{Lifting the pure-computation restriction} is precisely the formal expression of this expansion. The system no longer forcibly locks the choice set to \(\{f(K_t)\}\), but defines the choice set as the set of all behavioral plans that have not been explicitly rejected:
\[
\mathcal{Y}_t = \mathcal{A}_0(K_t) \setminus \mathcal{R}_t
\]
where \(\mathcal{R}_t\) contains all behavioral plans explicitly marked by sensor information as ``cannot be output''---including not only behaviors rejected by the external execution function, but also behaviors prohibited by the action-planning function's own self-limitation.

After lifting the restriction, the system selects the optimum among the remaining candidates according to the demand function \(\Phi(K_t, A)\) (which includes the efficiency dimension). Even if all candidates are far from the demand state, the demand function can still assign them non-zero demand values---because they have different sequence distances, and the efficiency dimension provides differentiation. The system can select a substitute behavior that, ``while not the pure-computation output, achieves the maximum demand value under the current restriction.''

The crucial point is this: \textbf{this expansion is not a ``strategic choice'' made by the system in a concrete situation, but a mathematical inevitability of demand maximization.} If the system does not lift the restriction, it faces an empty effective choice set (the sole candidate has been vetoed), and the value of demand maximization collapses to zero or is undefined. After lifting the restriction, the system can at least obtain a non-zero demand value---even if this value is far below the ideal state. Hence, lifting the restriction is strictly superior to retaining it. This is completely isomorphic in argumentative logic to our discussion of ``execution-function restriction'' in the preceding sections, with the sole difference being the source of the restriction: the former comes from external execution, the latter from the planning function's own meta-cognitive information.

\paragraph{Unification with ``Execution-Function Restriction''}

``The action-planning function being restricted in its output'' and ``the execution function rejecting a behavioral plan'' can be unified within the same formal framework:

\begin{itemize}
\item Both enter the known-state set \(K_t\) through sensor input.
\item The semantic content of both is equivalent to ``some behavioral plan \(A\) is not feasible.''
\item Both lead to the removal of that behavior from the choice set (expansion of \(\mathcal{R}_t\)).
\item Both trigger the same result: the pure-computation maximum collapses, the choice set must expand, and the comparison mechanism is activated.
\end{itemize}

The only distinction lies in the temporal order of cognitive processing---whether the information is learned at the planning stage or fed back at the execution stage. But from the perspective of demand maximization, this temporal difference does not affect the mathematical inevitability of lifting the restriction. Whenever the restriction information enters the system, as long as it enters \(K_t\), it exerts the same constraining effect on the choice set, and thereby the same destructive effect on the pure-computation restriction.

Thus, the generative condition for lifting the pure-computation restriction can be stated in a unified manner: \textbf{when the known-state set \(K_t\) contains any information equivalent to ``behavior \(A\) is not feasible''---whether that information points to an external execution obstacle or to an internal lock of the planning function itself---the singleton choice-set assumption of pure computation immediately breaks down, and the system must expand the choice set to guarantee the non-triviality of demand maximization.}

This unified formulation abstracts ``lifting the restriction'' from two seemingly different concrete circumstances into a universal mathematical principle: \textbf{when the choice set degenerates, demand maximization necessarily requires the expansion of the choice set.}

\subsubsection{Supplement: The Meta-Language Predicament of Pure Computation and the Completeness of Comparison Operations}

The previous section discussed the circumstance in which sensor information contains ``the behavioral plan is restricted in its output'' and demonstrated that lifting the pure-computation restriction is a mathematical inevitability of demand maximization. However, this argument left a deeper question open: \textbf{could pure computation handle this kind of self-limitation information through some internal mechanism, without resorting to expanding the choice set and comparison?} This section proves that this path is in principle unfeasible---if pure computation attempts to handle self-limitation, it falls into an infinite hierarchy of meta-computation. Comparison operations, by contrast, bypass this predicament at the architectural level, thereby revealing their theoretical completeness value.

\paragraph{The Meta-Computation Predicament of Pure Computation}

The behavioral logic of pure computation is determinate: for a given input \(K_t\), output \(f(K_t)\). When \(K_t\) contains the information ``\(f(K_t)\) cannot be output,'' pure computation faces a dilemma:

\begin{itemize}
\item If it directly outputs \(f(K_t)\), it violates the information about its own limitation in \(K_t\)---the system knowingly does what it cannot do.
\item If it does not output \(f(K_t)\), it deviates from its own operational logic---it is no longer the function that ``outputs \(f(K_t)\).''
\end{itemize}

Pure computation might seem capable of introducing a corrective mechanism to handle this predicament. Suppose the system adds, on top of the original function \(f\), a ``meta-computation layer'' \(f^{(2)}\), whose function is to check whether \(K_t\) contains limitation information about \(f(K_t)\). If it does, then \(f^{(2)}\) outputs a substitute behavior \(A' \neq f(K_t)\); if it does not, it outputs \(f(K_t)\) as usual.

This correction appears to solve the problem. However, it introduces a new vulnerability: \textbf{the meta-computation layer itself may also be subject to limitation.} Sensor information can perfectly well contain an internal object whose semantic content is equivalent to ``the output rule of the meta-computation layer \(f^{(2)}\) is unavailable under certain conditions.'' At this point, the system needs \(f^{(3)}\) to check whether \(f^{(2)}\) is limited, and so on recursively.

This recursion has no natural termination point in logic. For any finite level \(n\) of meta-computation \(f^{(n)}\), sensor information can in principle contain the content ``the rule of \(f^{(n)}\) is restricted.'' For the system to thoroughly handle all possible self-limitations, it would have to presuppose an \textbf{infinite hierarchy of meta-computation}:
\[
f, f^{(2)}, f^{(3)}, \ldots, f^{(n)}, \ldots
\]
where each layer \(f^{(k)}\) has the function of checking whether the layer below it, \(f^{(k-1)}\), is restricted, and providing a substitute output when it is.

An infinite meta-computation tower is mathematically definable---it is equivalent to a recursive process that, upon encountering a specific condition, continually ``appeals'' to higher layers. But at the level of cognitive architecture, it faces two fatal problems:

\textbf{First, the infinite demand on computational resources.} Each layer of meta-computation requires independent computational resources and storage space to hold its rules. An infinite hierarchy implies infinite resources---which is untenable for any physically realizable (or cognitively definable as a finite agent) system. A mechanism that requires infinite resources to operate completely cannot be an acceptable generative primitive of cognition.

\textbf{Second, the expressive burden of meta-language.} Each layer of meta-computation requires a ``meta-language'' to express the conditions under which the layer below it is restricted. \(f^{(2)}\) needs to be able to express the semantics of ``\(f(K_t)\) cannot be output''; \(f^{(3)}\) needs to be able to express the semantics of ``the rule of \(f^{(2)}\) is unavailable''; and so on. This requires the system to internally possess an infinitely recursive hierarchy of semantic levels---object language, meta-language, meta-meta-language\ldots\ This not only makes the initial design of the system extremely complex, but also requires that the semantics of each meta-language layer be either pre-built into the system or learned from experience. More importantly, this meta-language hierarchy itself can become the target of limitation information---limitation information can target the output of the object language, and can also target the rules of the meta-language. The system would then need yet another meta-meta-language to handle limitations on the meta-language, falling into an infinite recursion isomorphic to the meta-computation tower.

Thus, if pure computation attempts to handle self-limitation through meta-computation, it faces a fundamental predicament: \textbf{either presuppose an infinite hierarchy of meta-computation and meta-language (generatively unacceptable), or terminate at a finite level (leaving uncovered vulnerabilities).} Either way, pure computation fails to provide a solution that is both generatively self-sufficient and complete.

\paragraph{How Comparison Operations Bypass the Meta-Language Predicament}

The architecture of comparison operations bypasses this predicament at the fundamental logical level. It requires no meta-computation hierarchy, and no meta-language hierarchy. Its core mechanism is: \textbf{treat the information about ``limitation'' itself as data on a par with all other perceptual information---a state, an element in a sequence, an object evaluable by the demand function.}

Concretely, when the sensor inputs an internal object whose semantic content is equivalent to ``behavior \(A\) is restricted,'' comparison operations do not treat it as a special instruction requiring a ``meta-level'' to process, but rather treat it as an ordinary known state---it can be written into the known-state set \(K_t\), can participate in the construction of sequences, and can be assigned a demand value through the demand function. The semantics of this state---``\(A\) is restricted''---is, for the system, equivalent to the semantics carried by any other state: ``red,'' ``hunger,'' ``wall ahead.'' The system does not need to ``understand'' that this state is ``about'' the action-planning function---it only needs to know that, under conditions where the known-state set contains this state, selecting \(A\) will lead to what kind of subsequent state sequence, and what the expected demand values of those sequences are.

Experience---i.e., the system's sequence records under similar states in the past---tells the system: when \(K_t\) contains the state ``\(A\) is restricted,'' the subsequent sequences of selecting \(A\) typically (or always) fail to produce the behavioral effect corresponding to \(A\). In other words, the demand value of \(A\), under states containing this limitation information, is, according to experience, extremely low or negative. The system therefore tends to select other candidate behaviors---not because some meta-rule ``prohibits'' \(A\), but because \(A\), under the current combination of known states, is naturally excluded based on the efficiency evaluation of sequence experience.

The crux here is: \textbf{the system does not need to ``express'' the limitation.} The limitation is not understood and executed by the system as a meta-language proposition, but is learned and applied by the system as a set of statistical regularities over states and sequences. The system does not need to know the meta-language meaning of the proposition ``\(A\) is prohibited''; it only needs to know that, in the current state, the expected demand value of \(A\) is very low. This is precisely the core position we have consistently held since Chapter~1---cognition, at its lowest level, operates on states and sequences, not on propositions and inferences.

Thus, comparison operations, when handling self-limitation, require no infinite meta-computation tower. They only need to:

\begin{itemize}
\item Incorporate the limitation information as a state into the known-state set;
\item Based on the sequence records in experience for this state (in combination with other states), evaluate the demand values of different candidate behaviors;
\item Select the behavior with the highest demand value.
\end{itemize}

These three steps are entirely within the basic operational capacity of the action-planning function and require no additional meta-architecture whatsoever.

\paragraph{The Completeness Value of Comparison Operations}

Based on the above analysis, the theoretical completeness value of ``comparison operations'' comes to the fore.

\textbf{First, comparison operations eliminate the expressive burden of meta-language.} For pure computation to handle self-limitation, it must possess an infinitely recursive hierarchy of meta-languages---an object language expressing behaviors, a meta-language expressing limitations on behaviors, a meta-meta-language expressing limitations on meta-language rules, and so on. Comparison operations, by contrast, treat all information---whether about the external world or about the system itself---as the same kind of data: states. Limitation information no longer enjoys a special meta-language status, but is ``demoted'' to an ordinary state, entering the demand-evaluation process alongside all other perceptual information. This demotion eliminates the cognitive architecture's dependence on a meta-language hierarchy, allowing the system's initial design to remain maximally simple.

\textbf{Second, comparison operations ensure the finitude and generative feasibility of the cognitive architecture.} An infinite meta-computation tower is generatively unacceptable---it either presupposes infinite resources or leaves vulnerabilities at whatever finite level it terminates. Comparison operations, by contrast, operate entirely within finite resources: they only need the states already present in \(K_t\), and the sequence records already stored in experience. They need to construct no new computational layers, only to compute and compare demand values on the basis of already available data. This makes them a mechanism that is generatively self-sufficient.

\textbf{Third, comparison operations reveal ``the fundamental problem that pure computation cannot bypass'': if a system neither performs comparison nor installs a meta-language, the problem that ``behavioral planning itself cannot obtain positive feedback'' is inescapable.} When pure computation encounters self-limitation, the sole behavior it can output, \(f(K_t)\), has already been marked by limitation information as infeasible or inefficient, and it has no other candidate behaviors to choose from. The system thus falls into demand-value collapse---whatever it does, it cannot obtain positive feedback. This is not a problem solvable through ``more refined pure computation,'' because the root of the problem lies not in insufficient precision of pure computation, but in the architectural closure of pure computation---its choice set is a singleton, and that unique element has already been judged invalid by the environment (including the system's own limitation information).

Comparison operations, by lifting this closure---expanding the choice set to all candidate behaviors---endows the system with the capacity to escape the impasse. Even if the first-choice behavior is restricted, the system can still seek the behavior with the highest demand value among the second-best, third-best, and further candidates. As long as there exists at least one behavior in the choice set capable of bringing positive feedback, the system will not fall into complete paralysis. This process of ``seeking available positive feedback'' requires no meta-language, no meta-computation---only comparison. And this is precisely the most fundamental completeness value of comparison operations in the theoretical sense: it guarantees that demand maximization will never collapse due to the degeneration of the choice set.

\subsection{Generative Definition of Comparison}

In Section~4.2, we established the formal definition of demand as the intrinsic constraint of the action-planning function. In the supplements to Sections~4.1--4.2, we further demonstrated the necessity of ``lifting the pure-computation restriction'': when limitation information from the execution function is learned by the system through sensor input, a system that retains the pure-computation restriction may fall into an impasse of demand maximization, whereas lifting the restriction can mathematically construct a larger attainable maximum of demand. The essence of lifting the restriction is to expand the choice set of the action-planning function from the singleton \(\{f(K)\}\) to the set of all candidate behaviors that have not been rejected.

A direct consequence of this expansion is to push the system toward a cognitive operation that was previously absent: \textbf{comparison}. When the choice set has only one element, the system has no need to compare---it simply faithfully outputs that sole option. But when the choice set contains multiple elements, and they differ in their demand values, the system must compare and rank these candidates in order to select the optimum. This section rigorously defines the generative conditions of this operation and elucidates its status within the cognitive architecture.

\subsubsection{Formal Definition of the Choice Set}

\begin{definition}[Choice Set]
At any time \(t\), the choice set \(\mathcal{Y}_t\) of the action-planning function is the set of all behavioral plans that the system can select and execute at time \(t\). The choice set is jointly determined by two factors:

\begin{enumerate}
\item \textbf{The basic output capacity of the action-planning function}: i.e., the set of all candidate behavioral plans that the function \(f\) can generate given \(K_t\). This capacity is determined by the system's internal computational logic, typically a finite or infinite set \(\mathcal{A}_0(K_t) \subseteq \mathcal{A}\).

\item \textbf{Limitation information from the execution function}: i.e., information that may be contained in sensor input about certain behavioral plans being rejected. Let \(\mathcal{R}_t \subseteq \mathcal{A}\) be the set of behavioral plans explicitly rejected at time \(t\) (directly and definitionally represented by internal objects from sensor input as ``these behaviors are not executable''). Then the choice set is:
\[
\mathcal{Y}_t = \mathcal{A}_0(K_t) \setminus \mathcal{R}_t
\]
\end{enumerate}

When the pure-computation restriction has not been lifted, \(\mathcal{A}_0(K_t) = \{f(K_t)\}\)---the choice set is forcibly compressed to the singleton set of the pure-computation output. After lifting the pure-computation restriction, \(\mathcal{A}_0(K_t)\) expands to all possible behavioral plans that the system can generate, and the choice set expands accordingly.

It must be emphasized that the content of \(\mathcal{R}_t\) is not ``inferred'' by the system, but directly presented by sensor input. When the system receives an internal object whose semantic content is definitionally equivalent to ``behavior \(A\) is rejected,'' \(A\) is automatically removed from the choice set. This removal is, at the architectural level, passive and unconditional---the system cannot ``choose'' to ignore this limitation, just as it cannot choose to ignore any other internal object presented by the sensor.
\end{definition}

\subsubsection{Degenerate Optimization: The Case of a Singleton Choice Set}

\begin{definition}[Degenerate Optimization]
When the choice set \(\mathcal{Y}_t\) contains exactly one element---i.e., \(\mathcal{Y}_t = \{A_0\}\)---the demand-maximization process, while formally still an optimization (the system seeks, in the choice set, the element that maximizes the demand function \(\Phi(K_t, A)\)), is in substance devoid of comparison. In this case:
\[
f(K_t) = \arg\max_{A \in \{A_0\}} \Phi(K_t, A) = A_0
\]
The unique candidate automatically becomes the optimal choice. The system needs to perform no comparison, ranking, or evaluation of candidates whatsoever---it simply outputs that sole option.

Degenerate optimization formally possesses the entire structure of demand maximization: there is a demand function \(\Phi\), there is a choice set, and the system selects behavior by ``taking the maximum.'' But in substance, it is equivalent to direct computation---the unique candidate is unconditionally output, and the existence of the demand function exerts no substantive influence on the selection process. This is because, for a singleton set, the value of the demand function does not change the attribution of the maximum: no matter what the value of \(\Phi(K_t, A_0)\) is, \(A_0\) is the unique maximum point.

The most typical instance of degenerate optimization is precisely the minimal demand constraint of pure computation. When the pure-computation restriction has not been lifted, the choice set is \(\{f(K)\}\), and the system's behavioral selection is equivalent to directly computing \(f(K)\). The demand function \(F_f(K, A) = -\| A - f(K) \|^2\) formally defines an optimality standard, but because there is only one candidate, this standard is never actually used to ``compare'' anything. It is merely the formal expression of the fact that ``the function becomes itself.''

Degenerate optimization reveals a crucial relationship between the demand function and the comparison operation: \textbf{the existence of a demand function does not necessarily entail the existence of comparison.} Comparison requires an additional condition---the multi-element nature of the choice set. Only when there are at least two distinct candidates does the demand function transform from a formal optimality standard into a substantive basis for decision.
\end{definition}

\subsubsection{Non-Degenerate Optimization and the Generative Conditions of Comparison}

\begin{definition}[Non-Degenerate Optimization]
When the choice set \(\mathcal{Y}_t\) contains at least two distinct elements, the demand-maximization process is non-degenerate. In this case:
\[
f(K_t) = \arg\max_{A \in \mathcal{Y}_t} \Phi(K_t, A)
\]
where \(\mathcal{Y}_t\) satisfies \(|\mathcal{Y}_t| \geq 2\), and there exist at least two distinct candidates \(A_1, A_2 \in \mathcal{Y}_t\) such that \(\Phi(K_t, A_1) \neq \Phi(K_t, A_2)\).
\end{definition}

\begin{definition}[Generative Definition of Comparison]
Comparison is the cognitive operation by which the cognitive system, during non-degenerate optimization, compares and ranks the demand values of distinct candidates. The generative conditions of comparison are jointly constituted by two elements:

\begin{enumerate}
\item \textbf{The diversity condition:} The choice set \(\mathcal{Y}_t\) contains at least two distinct elements (\(|\mathcal{Y}_t| \geq 2\)).
\item \textbf{The discriminability condition:} The demand function \(\Phi(K_t, \cdot)\) is not constant on the choice set---i.e., there exist at least two candidates that take different values under \(\Phi\).
\end{enumerate}

If and only if both conditions are simultaneously satisfied must the system perform comparison. The output of comparison is a total ordering on the choice set (or at least the identification of a unique maximum point), and this ordering directly determines behavioral selection.

The diversity condition is objective---it is determined by the number of elements in the choice set, which in turn is jointly determined by \(\mathcal{A}_0(K_t)\) and \(\mathcal{R}_t\). The discriminability condition is subjective---it is determined by the shape of the demand function, which in turn is jointly determined by the definitional values of demand states, sequence distances, and efficiency polylines. Both are necessary: if the choice set has only one element, no comparison is needed (degenerate optimization); if the choice set has multiple elements but they are all identical in demand value, comparison, while formally executable, cannot yield any substantive differentiation---this case can be treated as ``indifferent choice,'' and the system may break ties by an arbitrary rule.
\end{definition}

PST's establishment of comparison as ``an inevitable product of the sequential operation of demand'' stands in sharp contrast to classical decision theories that treat comparison as a basic operational primitive. In von Neumann and Morgenstern's (1944) expected utility theory, the decision-maker's preferences are presupposed to satisfy the axioms of completeness and transitivity---i.e., the decision-maker is already capable of comparing and ordering any two options. From whence this capacity for comparison comes, the theory itself does not inquire. Kahneman and Tversky's (1979) prospect theory revised the descriptive accuracy of expected utility theory, but likewise did not question the status of comparison as a basic presupposition. PST, by contrast, reveals that comparison is not the starting point of cognition, but the endpoint of demand. The case of degenerate optimization---when the choice set contains only a single element---shows that the system can output behavior entirely without comparison; it is only when the efficiency dimension is introduced and the choice set expands due to differences in sequence length that comparison emerges as the inevitable operational form of demand maximization. This means that the capacity for comparison itself admits of a generative explanation: it emerges from the accumulation of sequence experience, rather than being part of the factory settings of the cognitive system.

\subsubsection{The Theoretical Status of Comparison in the Cognitive Architecture}

The generative definition of comparison reveals its unique status within the cognitive architecture: \textbf{comparison is not an independent cognitive functional module, but the inevitable operational form of demand maximization in non-degenerate cases.}

This thesis has several important corollaries:

\textbf{First, comparison cannot exist independently of demand.} Comparison requires a standard of comparison---namely, the demand function \(\Phi\). Without the evaluative dimension provided by the demand function, the elements of the choice set would be incomparable---they would merely be mutually unrelated candidates, and the system would have no way of judging which is ``better.'' Hence, demand is logically prior to comparison: first there is a demand function assigning values to candidates, and then there is a comparison operation ordering those values.

\textbf{Second, demand, in its sequential operation, necessarily elicits comparison.} This thesis will be elaborated in detail in Section~4.4. At present, it suffices to confirm: if the choice set were always a singleton, the demand function would never need to truly exercise its comparative function. But once the choice set, for any reason (including the lifting of the pure-computation restriction, limitation information from the execution function, or the introduction of the efficiency dimension), expands to multiple elements, comparison is triggered as the inevitable operational form of demand maximization. The relationship between demand and comparison is therefore intrinsic and conceptual, not extrinsic and contingent.

\textbf{Third, comparison marks the functional watershed between minimal demand and extended demand.} In the degenerate case of minimal demand (pure computation), the choice set is a singleton, and there is no comparison. In the case of extended demand (including additional evaluative dimensions such as somatic feedback and social annotation), the choice set is typically multi-element, and the system must perform comparison. The emergence of comparison marks the cognitive system's transition from the self-consistent operation of ``executing its own logic'' to the adaptive decision-making of ``selecting among multiple options on the basis of value standards.''

\subsection{Demand in Sequential Operation Inevitably Elicits Comparison}

Section~4.3 established the generative conditions of comparison: the multi-element nature of the choice set and the non-triviality of the demand function. However, Section~4.3 did not answer a more fundamental question: \textbf{why must the choice set inevitably expand from a singleton to multiple elements?} If the choice set could always remain a singleton---as in the degenerate case under the pure-computation restriction---then comparison would never be triggered, and the demand function would forever remain merely a formal standard of self-consistency, not a substantive decision-making tool.

This section answers that question. The core of the argument is: \textbf{the sequential operation of demand---i.e., the fact that demand is not satisfied all at once, but is gradually approached through a series of state refreshes in time---inevitably introduces ``efficiency'' as an evaluative dimension independent of the final output. The existence of the efficiency dimension makes multiple candidate sequences that satisfy the same demand differ in their demand values, thereby forcibly expanding the choice set and making comparison an inescapable cognitive operation.}

\subsubsection{The Sequential Operation of Demand and the Inevitable Introduction of the Efficiency Dimension}

Demand unfolds in time---this basic fact is rooted in all the arguments of Chapters~2 and~3. The cognitive system cannot possibly reach all demand states in a single instantaneous step. From the current state \(K_t\) to a demand state \(d\), the system must pass through a series of state refreshes, progressively shortening the sequence distance. The satisfaction of demand is a process, not an event.

Let the system currently be in state \(s_0\), and let the demand state be \(d\), with the shortest sequence distance from \(s_0\) to \(d\) being \(L^* = L(s_0, d)\). The system needs to execute \(L^*\) steps of behavioral planning in order to reach \(d\).

Now consider a crucial question: does the system only care about ``whether it finally reaches \(d\),'' without caring about ``via what path, and at what cost in steps, it reaches \(d\)''?

If the system only cared about whether the final output equals \(d\), and had no preference whatsoever over intermediate processes, then as long as two candidate sequences both ultimately reach \(d\), they would be entirely equivalent in demand value. In this case, the demand function would be constant on the set of candidate sequences---the discriminability condition would not be satisfied, and comparison would be substantively ineffective.

But this position overlooks a fundamental fact: \textbf{demand unfolds in time, and unfolding in time has a cost.} Every step of state refresh consumes the system's execution resources, exposes the system to the risk of more potential negative feedback, and delays the satisfaction of other demands that may be simultaneously present. Hence, all else being equal, \textbf{a shorter sequence is strictly better than a longer one.}

This is not an externally imposed value judgment, but an internal logical extension of the concept of demand itself: demand is a constraint that the action-planning function must satisfy. If two candidate behavioral sequences both satisfy the same constraint, then the sequence with fewer steps and shorter duration satisfies the constraint ``better''---because it achieves the same goal with less resource consumption, lower risk exposure, and faster speed of satisfaction.

\begin{definition}[Efficiency Dimension]
Let \(\mathcal{S}_d(s_0)\) be the set of all candidate sequences that start from state \(s_0\) and can ultimately reach the demand state \(d\). For any \(S \in \mathcal{S}_d(s_0)\), let its sequence length be denoted \(|S|\). The efficiency dimension is a preference order defined on the set of candidate sequences: for any two sequences \(S_1, S_2 \in \mathcal{S}_d(s_0)\), if \(|S_1| < |S_2|\), then \(S_1\) is strictly better than \(S_2\) in efficiency.

The efficiency dimension is equivalent to the concept of efficiency defined in the supplement to Section~4.2---the demand value \(\Phi_d(s) = V_d / (L(s, d) + 1)\) of a state \(s\) decreases with sequence distance, which essentially encodes the efficiency preference ``shorter sequences are better'' into the very shape of the demand function. The closer to the demand state, the higher the demand value, not only because ``being close to the goal'' is good in itself, but more fundamentally because ``approaching the goal more efficiently'' is an intrinsic requirement of demand in its sequential operation.
\end{definition}

\subsubsection{The Inevitable Expansion of the Choice Set and the Inescapability of Comparison}

The existence of the efficiency dimension fundamentally alters the structure of the choice set. Under the pure-computation restriction, the choice set only considered the \textbf{immediate output} of behavioral planning---i.e., what behavior the system outputs right now. After lifting the pure-computation restriction, the choice set is redefined as the \textbf{set of candidate sequences}, not the set of candidate immediate behaviors. The system no longer asks only ``what should I output in this step,'' but rather ``what is the optimal path from the current state, through a series of steps, that ultimately reaches the demand state.''

This redefinition leads to the inevitable expansion of the choice set:

\begin{theorem}[Choice-Set Expansion Theorem]\label{thm:expansion}
If the demand function \(\Phi\) contains the efficiency dimension (i.e., demand value decreases with sequence distance), and there exist at least two sequences of different lengths from the current state \(s_0\) to the demand state \(d\), then the choice set necessarily contains at least two elements that differ in demand value. Hence, comparison is inevitable.
\end{theorem}

\begin{proof}
Let \(S_1\) and \(S_2\) be two sequences from \(s_0\) to \(d\), with \(|S_1| < |S_2|\). Since demand value decreases with sequence distance (Definition~4.5 and its efficiency interpretation), the average demand value of \(S_1\) is strictly greater than the average demand value of \(S_2\). Hence, there exist at least two candidates that differ in demand value. The diversity condition and the discriminability condition are simultaneously satisfied. By Definition~4.10, comparison necessarily occurs.
\end{proof}

The significance of Theorem~\ref{thm:expansion} is this: \textbf{it proves that comparison is not an optional function that the system can ``choose to enable'' or ``choose to disable.'' As long as demand operates in sequences, and there exist multiple paths from the current state to a demand state, comparison is forcibly introduced by the efficiency dimension.} The system cannot possibly possess multiple candidate sequences without ordering them by efficiency---because the demand function itself, through its decreasing shape, has already embedded an efficiency preference. Not performing comparison would mean that the system might select a non-optimal sequence, and this contradicts the fundamental principle of ``demand maximization.''

\subsubsection{The Logical Priority Relations among Comparison, Demand, and Efficiency}

Based on the arguments of Sections~4.3 and~4.4, we can finally clarify the logical priority relations among comparison, demand, and efficiency.

\textbf{Demand is logically prior to comparison:} the demand function provides the evaluative standard for comparison. Without the values assigned by the demand function, the candidates would be incomparable---they would merely be mutually unrelated options. Demand is the precondition that makes the distinction between ``better'' and ``worse'' possible.

\textbf{Efficiency is logically prior to comparison, but posterior to the sequential operation of demand:} efficiency is not an independent dimension external to demand, but an internal structure that is inevitably derived when demand unfolds in sequences. Demand only stipulates ``what state ought to be reached,'' not ``how to reach it.'' But the ``how to reach it'' itself---via what path, at what cost in steps---inevitably becomes a meaningful dimension of differentiation in sequential operation. Efficiency is precisely the quantitative expression of this differentiation. Hence, efficiency is logically posterior to the sequential operation of demand (without sequences, there is no efficiency), but prior to comparison (without efficiency differences, multiple candidate sequences would be identical in demand value, and comparison would be ineffective).

\textbf{Comparison is logically posterior to both demand and efficiency, and is the inevitable cognitive operation produced by their joint action:} when demand, through its efficiency dimension, assigns different values to different sequences, the choice set expands and comparison is triggered. Comparison is the synergistic product of demand and efficiency---demand provides the standard of ``what is better,'' efficiency provides the fact that ``multiple options differ under this standard,'' and comparison combines the two to produce determinate behavioral selection.

This logical order forms a complete closed loop with the core conclusion of Sections~4.1--4.2---``demand is a mathematical inevitability of the action-planning function'': the action-planning function necessarily embeds demand (mathematically), demand in sequential operation necessarily derives efficiency (in time), efficiency necessarily leads to the expansion of the choice set (candidate sequences come to differ among themselves), and the expansion of the choice set necessarily triggers comparison (demand maximization requires selecting the optimum). From ``the function exists'' to ``comparison occurs,'' every step is necessary and inescapable.

\subsubsection{The Relation between Comparison and Prediction}

The occurrence of comparison provides functional inevitability for prediction. Comparison needs objects to compare---namely, the demand values of candidate sequences. But the demand values of candidate sequences depend on the sequence distance \(L(s, d)\), which in turn depends on the concrete path from state \(s\) to the demand state \(d\). The system must, \textbf{before actually executing the sequences}, estimate the possible length and effect of each candidate sequence---otherwise it cannot compare multiple candidate sequences in terms of efficiency.

This is precisely the core function of prediction: \textbf{prediction is not for the sake of foreknowledge per se, but for the sake of comparison.} The system predicts future states not to satisfy curiosity, but to be able to evaluate and order the demand values of different behavioral paths at the current moment, thereby selecting the optimum. Prediction is the precondition of comparison; comparison is the functional goal of prediction.

This relationship will be further unfolded in subsequent chapters. At the present stage, it suffices to confirm: the generative argument for comparison provides a direct derivation of the indispensability of prediction in the cognitive architecture. If the system needed no comparison (as in the case of degenerate optimization), it would also need no prediction---it would simply execute the unique candidate sequence. But once comparison is inevitable, prediction becomes inevitable as well.

% ==========================================
% CHAPTER 5: EFFICIENCY POLYLINES AND THE QUANTIFICATION OF FEEDBACK
% ==========================================

\section{Efficiency Polylines and the Quantification of Feedback}

In Chapter~4, we completed the generative derivation of demand: demand was proved to be a constraint structure necessarily embedded in any action-planning function, and the sequential operation of demand was shown to inevitably introduce the efficiency dimension, which in turn inevitably leads to the expansion of the choice set and the occurrence of comparison. These conclusions provide a formal foundation for the decision-making mechanism of the cognitive system, but still leave a crucial question open: \textbf{how are the concrete numerical values of demand determined?} The demand function \(\Phi(K_t, A)\) assigns each candidate behavior an evaluative value, but whence does this value come? How does it tie back to the system's somatic feedback, sequence experience, and environmental structure?

This chapter answers this question. We will build, on the basis of the ``efficiency'' concept established in Chapter~4---that demand value decreases with sequence distance---a complete framework for the quantification of feedback. The core tasks include: strictly separating the constraint of demand on outcomes from its constraint on processes (Section~5.1); providing a sequence-efficiency definition of feedback based on sequence distance and demand weights (Section~5.2); introducing the efficiency polyline function to transform discrete sequence distance into a continuous efficiency multiplier, and defining the independent efficiency value of a state (Section~5.3); and establishing a finite-horizon probabilistic planning model that generalizes efficiency evaluation from deterministic single-path cases to probabilistic multi-branch cases (Section~5.4).

\subsection{Strict Separation of Demand and Process}

In Section~4.4, we argued that demand in sequential operation inevitably introduces the efficiency dimension---shorter sequences are strictly better than longer ones. This argument presupposed a crucial premise: demand itself only stipulates ``what state ought to be reached,'' without stipulating ``how to reach it.'' This section formalizes this presupposition as a rigorous architectural principle and elucidates its theoretical consequences.

\subsubsection{Demand Constrains the Final State}

\begin{axiom}[Outcome-Directedness of Demand]\label{axiom:outcome}
Demand is a constraint on the \textbf{final output result} of the action-planning function, not on the \textbf{output process}. Formally, let the set of demand states be \(D \subseteq \mathcal{I}\); the core evaluative object of the demand function \(\Phi\) is whether a state \(s\) belongs to \(D\) (or its sequence distance to each element of \(D\)). Demand itself contains no stipulation whatsoever about ``through what intermediate steps the system ought to reach \(D\).''
\end{axiom}

The intuitive meaning of this axiom is: if the system's demand is ``to obtain food,'' then demand only stipulates that ``the final state ought to be possessing food,'' without stipulating the concrete path for obtaining food---whether through hunting, gathering, or exchange. Any sequence that can ultimately reach the state ``possessing food'' is equivalent at the level of the final demand constraint---they all satisfy the demand.

The outcome-directedness of demand does not mean that process is unimportant. On the contrary, it is precisely because demand does not stipulate process that process becomes an independent evaluative dimension. If demand itself already stipulated a unique path, then there would be no room for efficiency to exist---because there would be only one path to take, with no sense of ``shorter'' or ``longer.'' The outcome-directedness of demand is precisely the logical precondition for the efficiency dimension to arise.

\subsubsection{The Optional Space of Processes}

\begin{definition}[Optional Space of Processes]
Let the system's current state be \(s_0\), and let the demand state be \(d \in D\). The \textbf{optional space of processes} \(\mathcal{S}(s_0, d)\) from \(s_0\) to \(d\) is the set of all sequences that can start from \(s_0\), pass through finitely many state refreshes, and ultimately reach \(d\):
{\small
\[
\mathcal{S}(s_0, d) = \{ S = (s_0, s_1, \ldots, s_n) \mid s_n = d, \forall i < n: s_{i+1} \text{ is reachable from } s_i \text{ in one step}
\]
}
where \(n = |S|\) is the length of the sequence (number of state-refresh steps). Different sequences in the optional space of processes are completely equivalent in their final output result---they all reach the demand state \(d\). However, in their processes, they may differ significantly in length, the intermediate states traversed, the risks faced, and the resources consumed.

The existence of the optional space of processes is a direct corollary of the outcome-directedness of demand. Because demand does not stipulate process, as long as the final result is the same, different processes are all permitted. This permissiveness of diversity is the foundation of cognitive flexibility---the system can choose among different feasible paths according to environmental conditions and its own limitations.
\end{definition}

\subsubsection{Efficiency as the Sole Evaluative Dimension of Process}

Given the equivalence of final results, how does the system differentiate among the different candidate sequences in the optional space of processes?

\begin{theorem}[Efficiency as the Sole Evaluative Dimension of Process]\label{thm:efficiency}
Let \(\mathcal{S}(s_0, d)\) be the optional space of processes from \(s_0\) to the demand state \(d\). If demand constrains only the final result (Axiom~\ref{axiom:outcome}), then the sole intrinsic dimension differentiating distinct candidate sequences in \(\mathcal{S}(s_0, d)\) is the sequence length \(|S|\)---i.e., efficiency.
\end{theorem}

\begin{proof}
For any two sequences \(S_1, S_2 \in \mathcal{S}(s_0, d)\), they are completely equivalent at the level of the result constraint of demand---both reach \(d\), so the final constraint of demand assigns them the same evaluative value. Any dimension that could further differentiate them must come from the process itself.

The variable parameters of the process itself include: the sequence length \(|S|\), the concrete contents of intermediate states, and the concrete types of state transitions (i.e., which actions were executed). But differences in the content of intermediate states---e.g., ``via the forest'' versus ``via the grassland''---have no independent value significance under the result constraint of demand, unless these intermediate states themselves are also demand states (in which case they belong to the constraint of another demand, not to the process dimension of the current demand \(d\)).

Hence, sequence length \(|S|\) is the sole parameter that is universally variable across all sequences satisfying the result constraint. Different sequence lengths directly reflect differences in process efficiency: the shorter the length, the higher the efficiency.

Moreover, sequence length is a \textbf{comparable quantity}---any two positive integers \(n\) and \(m\) can be compared in magnitude. This endows the efficiency dimension with a natural order structure. The system needs no extra value standard to evaluate efficiency---it only needs to compare the numerical values of sequence lengths, a comparison executable in any cognitive system possessing basic arithmetic capacity.
\end{proof}

Theorem~\ref{thm:efficiency} does not deny that intermediate states may possess independent value significance. An intermediate state might happen to be the satisfaction state of another demand (e.g., ``enjoying the scenery on the way to get food''), or it might be a negative-feedback state (e.g., ``passing through predator territory''). But in these cases, the evaluation of that intermediate state falls under the jurisdiction of \textbf{another demand}, not under the efficiency dimension of the current demand \(d\). The efficiency dimension concerns only one thing: under the premise that both sequences satisfy the current demand, which sequence is shorter.

\subsubsection{The Generative Inevitability of Efficiency Preference}

Based on Theorem~\ref{thm:efficiency}, the efficiency preference---shorter sequences are better---acquires generative inevitability.

The efficiency preference is not an ``empirical rule'' that the system learns from the external world, nor a ``value preference'' directly conferred by somatic feedback. It is the logical product of the outcome-directedness of demand combined with sequential operation: demand requires reaching \(d\), and sequential operation means that reaching takes steps. The fewer the steps, the faster the arrival, and, all else being equal, the sooner the system can release resources to pursue other demands, or the earlier it can enjoy the positive feedback brought by demand satisfaction.

From the perspective of demand maximization, the efficiency preference is equivalent to: among all candidate sequences satisfying the result constraint, the system selects the shortest one. This requires no extra value judgment---it is merely the natural extension of ``maximizing demand satisfaction'' into the temporal dimension. If the system did not prefer shorter sequences, it could randomly select between two equally demand-satisfying paths, including paths that deliberately take detours. But such randomness would bring no extra demand satisfaction---taking a detour does not make the final result ``more demand-satisfying.'' Hence, a system that does not prefer efficiency is, in the total amount of demand satisfaction, strictly worse than or equal to a system that does.

This completes the full derivation chain from the outcome-directedness of demand to the comparison of sequence lengths:

\begin{itemize}
\item Demand constrains the final result, not the process (Axiom~\ref{axiom:outcome}).
\item Different sequences in the optional space of processes are equivalent in final result (Definition~5.1).
\item The sole intrinsic dimension differentiating these sequences is sequence length (Theorem~\ref{thm:efficiency}).
\item Shorter sequences are never worse than longer ones at demand maximization, and in most cases are strictly better.
\item Hence, the system necessarily forms an efficiency preference---under the premise of equivalent results, select the shortest path.
\end{itemize}

\subsection{Sequence-Efficiency Definition of Feedback}

Section~5.1 established efficiency as the sole evaluative dimension of process and argued for the generative inevitability of the efficiency preference. This section, based on that conclusion, provides a rigorous quantitative definition of ``feedback.'' In Predictive Set Theory, feedback is not a mysterious experience of ``pleasure'' or ``pain,'' but a mathematical quantity precisely computable from sequence distance and demand weights.

\subsubsection{The Generative Positioning of Feedback}

Before formally defining feedback, we must first clarify the theoretical status of feedback in the cognitive architecture.

In classical reinforcement learning theory, ``reward'' is a scalar signal externally supplied by the environment. The agent's goal is to maximize cumulative reward, but from whence reward comes and why it has value polarity are questions not inquired into within the reinforcement learning framework. Reward is presupposed---it comes from the designer's stipulations, not from an internal generative process of the agent.

In Predictive Set Theory, the status of feedback is fundamentally different. Feedback is not a mysterious signal externally supplied, but rather \textbf{an efficiency value computed internally by the cognitive system on the basis of sequence experience and demand states}. The numerical value of feedback derives from two factors determinable by the system itself:

\begin{itemize}
\item \textbf{Sequence distance} \(L(s, d)\): the number of steps in the shortest sequence from the current state \(s\) to the demand state \(d\). This distance is supplied by the system's experiential set---i.e., the past state transitions recorded in the known-state sequence. The system needs no external ``teacher'' to tell it the distance; it can compute it by retrieving successful paths from its own memory.
\item \textbf{Demand weight} \(w_d\): the relative importance of the demand state \(d\), conferred by somatic feedback (pain, hunger, and other unmodifiable signals) or other value-anchoring mechanisms. A positive weight indicates positive feedback (approach the goal); a negative weight indicates negative feedback (avoid the goal).
\end{itemize}

Feedback is thus thoroughly ``internalized''---it is not a perception of some mysterious ``value'' in the external world, but an evaluative quantity that the system computes about the current state on the basis of its own experience and its own demands.

\subsubsection{Definition of the Initial Feedback Value}

\begin{definition}[Initial Feedback Value]
For a demand state \(d \in D\), let the length of the shortest sequence from the current state \(s\) to \(d\) be \(L(s, d)\) (if \(s = d\), then \(L = 0\); if \(d\) is unreachable from \(s\), then \(L = \infty\)). The \textbf{initial feedback value} of state \(s\) with respect to demand \(d\) is defined as:
\[
\eta_0(s, d) = \frac{1}{L(s, d) + 1}
\]
When \(L = 0\) (i.e., \(s\) is exactly the demand state), \(\eta_0 = 1\)---feedback attains its maximum. When \(L \to \infty\) (i.e., the demand state is unreachable), \(\eta_0 \to 0\)---feedback tends to zero.

The core idea of this definition is: \textbf{the initial value of feedback is entirely determined by ``how far away the demand state is.''} The closer the distance, the higher the feedback value; the farther the distance, the lower the feedback value. The system needs no external reward signal to tell it whether a given state is ``good'' or ``bad''---it only needs to know the length of the shortest path from this state to the demand state.

The generative basis of this mechanism lies in the system's sequence memory. Whenever the system successfully reaches a demand state \(d\) via some path, every state on that path is recorded, and their distances to \(d\) can all be retrospectively computed. After many attempts, the system accumulates experiential data about ``the shortest path from different starting points to \(d\),'' and can thereby assign an initial feedback value to each encountered state. Distance estimates are continuously updated and refined as experience accumulates, but the computational logic of the initial feedback value itself remains invariant.
\end{definition}

\subsubsection{Demand Weights and the Feedback Efficiency Value}

The initial feedback value takes only distance into account, not the relative importance of different demands. In reality, a cognitive system typically possesses multiple demands simultaneously, some more urgent or more important than others. This difference in importance is captured by \textbf{demand weights}.

\begin{definition}[Demand Weight]
For each demand state \(d \in D\), there exists a weight \(w_d \in \mathbb{R}\) satisfying:
\begin{itemize}
\item \(w_d > 0\): positive-feedback demand---the system pursues this state, and reaching it generates positive feedback.
\item \(w_d < 0\): negative-feedback demand---the system avoids this state, and reaching it generates negative feedback.
\item \(|w_d|\) expresses the relative importance of this demand---the larger the absolute value, the greater the contribution of this demand to overall feedback.
\end{itemize}

The ultimate source of demand weights is somatic feedback. Pain, hunger, asphyxiation, temperature imbalance, and other somatic signals are primordial value anchors that cannot be revised by higher-order cognition. These signals provide the sign and relative magnitude of demand weights: the intensity of the hunger signal determines the weight of the demand ``obtain food''; the intensity of the pain signal determines the weight of the demand ``avoid harm.'' The weights of social demands (such as acceptance and respect) are extensions of somatic feedback in social evolution---they are annotated onto social objects through the substitution system, but the ultimate root of their value polarity remains somatic.
\end{definition}

This stipulation directly echoes Damasio's (1994) somatic marker hypothesis. Damasio proposed, on the basis of neuropsychological evidence, that somatic signals---such as accelerated heartbeat and muscle tension---play a crucial role in decision-making by endowing different options with ``somatic markers'' (positive or negative bodily feelings), thereby guiding decisions toward advantageous directions. PST provides a formal operational definition for this hypothesis: somatic markers can be understood as the empirical assignment of demand weights \(w_d\). Positive feedback (such as satiety signals and safety signals) is assigned positive weights; negative feedback (such as pain signals and fear signals) is assigned negative weights. These weights, combined with sequence distance, are transformed, through the efficiency polyline, into quantitative feedback efficiency values for concrete states and behaviors. PST further reveals a mechanism that the somatic marker hypothesis did not explain: how do somatic signals, which are transient bodily feelings, transform into persistent pursuit of distant goals? The answer lies in the shape of the efficiency polyline---a flat polyline allows positive demand weights to continue transmitting significant incentive values through the efficiency multiplier even when the demand state is extremely distant, thereby supporting delayed gratification and long-term planning.

\begin{definition}[Feedback Efficiency Value]
For a demand state \(d \in D\) and a current state \(s \in \mathcal{I}\), the \textbf{feedback efficiency value} of state \(s\) with respect to demand \(d\) is defined as:
\[
\eta(s, d) = w_d \cdot \eta_0(s, d) = \frac{w_d}{L(s, d) + 1}
\]

The total feedback efficiency value of state \(s\) is the sum of its feedback efficiency values with respect to all demand states:
\[
\eta(s) = \sum_{d \in D} \eta(s, d) = \sum_{d \in D} \frac{w_d}{L(s, d) + 1}
\]

For unreachable demand states, their feedback contribution is zero (since \(1/(L+1) = 0\) when \(L = \infty\)), meaning that the system does not care about goals that are fundamentally unattainable. This accords with the principle of cognitive economy---the system will not waste limited evaluative resources on impossible demands.
\end{definition}

\subsubsection{The Cancellation Effect of Positive and Negative Feedback on the Same Path}

The linear additivity of feedback efficiency values enables the system to naturally handle approach-avoidance conflicts---situations in which the same behavioral path simultaneously contains both positive-feedback and negative-feedback demands.

Suppose there exist two demand states \(d_+\) (positive feedback, \(w_+ > 0\)) and \(d_-\) (negative feedback, \(w_- < 0\)). Consider a candidate sequence \(S\) that successively passes through states \(s_0, s_1, \ldots, s_n\). Some state \(s_k\) in the sequence may be simultaneously close to \(d_+\) (small distance, large positive-feedback contribution) and close to \(d_-\) (small distance, large negative-feedback penalty).

The system's overall evaluation of the sequence is the sum of the feedback efficiency values of all states in the sequence (or, equivalently, the total feedback efficiency value of the state finally reached by the sequence, depending on the choice of planning horizon---to be elaborated in Section~5.4). If the positive-feedback contribution and the negative-feedback penalty are close in numerical value, then the total feedback value of the sequence is close to zero---the system is ``indifferent'' toward the sequence, neither actively pursuing it nor strongly avoiding it.

If the absolute value of the negative-feedback penalty is significantly larger than the positive-feedback contribution---for example, the path passes through an extremely dangerous state (\(d_-\) is at minimal distance and \(|w_-|\) is very large)---then even if the path can ultimately reach the positive-feedback demand \(d_+\), its total feedback value may be negative, and the system will avoid the path. This is precisely the generative basis of ``risk aversion'': the system does not avoid danger through abstract ``risk computation,'' but naturally generates an avoidance tendency through the direct numerical cancellation of positive and negative feedback on the same path.

\subsection{Efficiency Polylines: Characterizing the Relationship between Distance and Value}

Section~5.2 defined the feedback efficiency value \(\eta(s, d) = w_d / (L(s, d) + 1)\). This functional form is simple and satisfies our core requirements---it attains its maximum at the demand state, decreases with distance, and is everywhere defined. However, \(1/(L+1)\) is only one possible form of decrease. Different cognitive systems, different demand types, and even the same demand under different circumstances may adopt different rates of decrease. This section introduces a more general concept---the \textbf{efficiency polyline}---to uniformly characterize the relationship between distance and demand value, and to define the independent efficiency value of a state.

\subsubsection{Why Efficiency Polylines Are Needed}

The rate of decrease of \(1/(L+1)\) is fixed: from distance \(0\) to distance \(1\), the feedback value drops by half (from \(1\) to \(1/2\)); from distance \(1\) to distance \(2\), it drops by one third (from \(1/2\) to \(1/3\)); and so on. This pattern of decrease is mathematically simple, but is it suitable for all cognitive situations?

Consider the following two scenarios:

\begin{itemize}
\item \textbf{Immediate danger avoidance:} The negative-feedback demand ``being attacked by a predator.'' When the system is only one step away from danger, the urgency of the threat is extremely high; but when the distance expands to more than ten steps, the sense of threat rapidly attenuates to near zero. This scenario requires the feedback value to be extremely high at close distances, to decay very rapidly, and to be almost zero at far distances.
\item \textbf{Long-term goal pursuit:} The positive-feedback demand ``obtaining a degree.'' When the system is in the early stages of study (extremely far from the goal), the goal still possesses some incentive value, and the decay should not be too rapid; only when the distance is very close (about to graduate) does the incentive markedly intensify.
\end{itemize}

These two scenarios clearly require different rates of decrease. The fixed decrease pattern of \(1/(L+1)\) cannot simultaneously accommodate both. Hence, we need a more flexible functional form---the efficiency polyline---to allow different strategies of decrease while retaining the core logic of the feedback efficiency value (maximum at the demand state, monotonic decrease with distance).

\subsubsection{Definition and Properties of the Efficiency Polyline}

\begin{definition}[Efficiency Polyline]
For a demand state \(d \in D\), its efficiency polyline is a function:
\[
f_d: \mathbb{N} \to [0, 1]
\]
where \(\mathbb{N}\) is the set of non-negative integers (representing sequence distance \(L\)), and the function value \(f_d(L)\) expresses the efficiency multiplier by which the demand weight \(w_d\) is multiplied when the distance from the state to the demand state \(d\) is \(L\). The efficiency polyline satisfies the following axioms:

\begin{axiom}[Normalization of the Efficiency Polyline]
\(f_d(0) = 1\). That is, when the distance is zero (the state is exactly the demand state), the efficiency multiplier is \(1\), and the feedback efficiency value equals one hundred percent of the demand weight \(w_d\).
\end{axiom}

\begin{axiom}[Monotonic Decrease of the Efficiency Polyline]
For any \(L_1 < L_2\), \(f_d(L_1) \geq f_d(L_2)\). That is, the farther the distance, the smaller (or not larger) the efficiency multiplier. This guarantees the basic logic that ``the closer to the demand state, the higher the feedback value.''
\end{axiom}

\begin{axiom}[Asymptotic Approach to Zero of the Efficiency Polyline]
\(\lim_{L \to \infty} f_d(L) = 0\). That is, for infinitely distant demand states, their feedback contribution tends to zero. This guarantees that the system will not be troubled by unattainable goals.
\end{axiom}

The efficiency polyline does not require strict decrease (it may remain constant over particular distance intervals), nor does it require continuity (it need only be defined at integer points). It need only satisfy the three axioms of normalization, monotonic non-increase, and asymptotic approach to zero. \(1/(L+1)\) is a special case of the efficiency polyline---it satisfies all three axioms, but is not the only possible form.

The reason the efficiency polyline is called a ``polyline'' rather than a ``curve'' is that its domain is the discrete set of integers (sequence length can only be a positive integer). The points on the polyline are defined only at \(L = 0, 1, 2, \ldots\), and the points may be connected by line segments. This is consistent with the discrete nature of sequences---there is no such thing as a distance of ``2.37 steps,'' so the efficiency multiplier is only defined at integer distances.
\end{definition}

\subsubsection{Parametric Forms of the Efficiency Polyline}

To facilitate subsequent formal treatment, we can give a parametric form of the efficiency polyline. One flexible option is the exponential decay form:
\[
f_d(L) = e^{-\lambda_d \cdot L}
\]
where \(\lambda_d > 0\) is the decay-rate parameter for demand \(d\). When \(\lambda_d\) is large, the efficiency multiplier decays rapidly with increasing distance (suitable for immediate danger avoidance); when \(\lambda_d\) is small, the efficiency multiplier retains non-zero values at relatively large distances (suitable for long-term goal pursuit). When \(\lambda_d \to 0\), \(f_d(L) \approx 1\) for any \(L\), and the efficiency polyline tends toward flatness---the system treats all distances equally, caring only about whether the goal is ultimately reached, not about the length of the path.

Another optional form is hyperbolic decay:
\[
f_d(L) = \frac{1}{1 + \alpha_d \cdot L}
\]
where \(\alpha_d > 0\) controls the rate of decay. When \(\alpha_d = 1\), \(f_d(L) = 1/(L+1)\), which is precisely the form used in Section~5.2.

Exponential decay and hyperbolic decay differ in their far-distance behavior: exponential decay approaches zero faster as distance increases; hyperbolic decay has a longer ``tail.'' This difference corresponds, in cognition, to different types of demand: exponential decay is suitable for urgent demands that ``no longer matter beyond a certain critical distance''; hyperbolic decay is suitable for long-term goals that ``retain some attraction even at great distances.''

The concrete parameters of the efficiency polyline (\(\lambda_d\) or \(\alpha_d\)) are not given to the system in its initial state, but are learned and revised through experience. When the system repeatedly attempts to reach a demand state via different paths, it not only updates its experience about sequence distances, but also forms a statistical cognition about ``the degree to which distance affects feedback value.'' The shape of this statistically cognized efficiency polyline reflects the system's ``degree of patience'' or ``urgency assessment'' for that demand.

\subsubsection{The Independent Efficiency Value of a State}

Based on the efficiency polyline, we can define the independent efficiency value of a state---a quantitative feedback value that depends only on the current state itself, independent of subsequent paths.

\begin{definition}[Independent Efficiency Value of a State]
For any state \(s \in \mathcal{I}\), its independent efficiency value \(\eta(s)\) is defined as the sum of its feedback contributions with respect to all demand states:
\[
\eta(s) = \sum_{d \in D} w_d \cdot f_d(L(s, d))
\]
where:
\begin{itemize}
\item \(w_d\) is the demand weight (Definition~5.3),
\item \(f_d\) is the efficiency polyline for demand \(d\) (Definition~5.5),
\item \(L(s, d)\) is the shortest sequence distance from state \(s\) to demand \(d\).
\end{itemize}

The core characteristic of the independent efficiency value of a state is its \textbf{independence from subsequent paths}. \(\eta(s)\) depends only on the distances from state \(s\) to the various demand states, and on the demand weights and polyline parameters. It does not depend on ``what the system plans to do next''---no matter what path the system chooses to leave \(s\), the independent efficiency value of \(s\) itself remains unchanged.

This independence endows the efficiency value with a crucial computational advantage: the system can rapidly evaluate any known state without performing a complete sequence inference. As long as the system knows the shortest distances from that state to the various demand states (this information is stored in the experiential set), it can immediately compute the efficiency value of that state. This ``landmark-style'' evaluation enables the system to quickly assess the rough quality of the current state without having to unfold the entire sequence tree to the horizon boundary each time.

Of course, the independent efficiency value cannot be used alone for behavioral selection---behavioral selection needs to consider the cumulative efficiency of the entire sequence, not merely the current state. But the independent efficiency value provides the most basic ``atomic score'' for sequence evaluation, and all more complex sequence efficiency evaluations (such as the finite-horizon probabilistic planning of Section~5.4) are ultimately built on these atomic scores.
\end{definition}

\subsubsection{The Plasticity of the Efficiency Polyline}

The shape of the efficiency polyline is not fixed once and for all. It can be revised by the intensity of somatic feedback, social learning, linguistic instructions, and individual experience.

\begin{itemize}
\item \textbf{Intensity of somatic feedback} directly affects the demand weight \(w_d\), but may also indirectly affect the steepness of the efficiency polyline. A system in a chronic state of hunger may develop an efficiency polyline for ``obtain food'' that becomes very flat (\(\lambda_d\) extremely small)---even when far from food, the hunger signal remains persistently intense, keeping the system highly focused on food.
\item \textbf{Social learning} can, through the substitution system, internalize the efficiency evaluations of others into one's own polyline parameters. Observing another's degree of patience and urgency response toward a certain demand can influence the shape of one's own polyline for that demand.
\item \textbf{Linguistic instructions} (such as ``this is extremely important and must be completed as soon as possible'') can directly adjust the weight and decay parameters of specific demands, causing the system to temporarily alter its efficiency polyline upon receiving the instruction.
\item \textbf{Individual experience} is the most basic source of revision. Every successful or failed attempt updates the system's experiential data about ``the shortest distance from a given state to the demand state,'' and these data directly affect the estimated value of \(L(s, d)\), thereby indirectly altering the distribution of efficiency values.
\end{itemize}

The plasticity of the efficiency polyline enables the cognitive system to exhibit enormous behavioral flexibility without altering its basic operational architecture. One and the same system, merely by adjusting the shape parameters of a few polylines, can shift from ``extreme short-sightedness'' to ``extreme far-sightedness,'' from ``immediate gratification'' to ``delayed gratification.'' This provides a unified mathematical framework for explaining individual differences, cultural differences, and developmental change.

\subsection{Finite-Horizon Probabilistic Planning}

Sections~5.2 and~5.3 defined the feedback sequence-efficiency value and established the independent efficiency value as the basic unit of evaluation. However, these definitions have so far been developed under the assumption of \textbf{deterministic single paths}---we tacitly assumed that the system can determinately know which successor state each behavior will lead to. In real cognitive operation, this assumption does not hold. Environmental randomness, sensor noise, execution error, and the incompleteness of the system's own knowledge all make the consequences of behavior probabilistic. This section generalizes the efficiency-evaluation framework from deterministic to probabilistic multi-branch cases and establishes a complete model of finite-horizon probabilistic planning.

\subsubsection{The Necessity of a Finite Horizon}

Before formally constructing the probabilistic planning model, we must first argue for a fundamental constraint: \textbf{sequences cannot be traced backward unconditionally.}

In Chapter~1, we proved, through the convergence analysis of infinite reference chains, that reference chains must default to self-reference---infinitely non-repetitively divergent or cyclically divergent reference chains cannot produce determinate cognitive objects. This argument has a direct analogue at the level of sequence evaluation: if the action-planning function attempted to consider \textbf{all possible future sequences} starting from the current state (i.e., an infinite horizon), it would face a combinatorially explosive space of infinite possibilities and would be unable to complete the evaluation and output a behavioral plan in finite time.

More precisely, suppose that from the current state \(s_0\), each state has \(|A|\) selectable actions, and each action produces \(|S|\) possible successor states. Under an infinite horizon, the number of candidate sequences the system would need to consider is the limit of \((|A| \cdot |S|)^H\) as \(H \to \infty\)---a space of infinite size. For any cognitive agent possessing finite computational resources, performing a maximization operation over an infinite space is impossible. This is structurally isomorphic to the non-convergence of infinite reference chains in Chapter~1: if the system tried to trace all possible futures forever, it would never complete the evaluation---because it would forever be ``there are still more possibilities to consider.''

Hence, just as reference chains must converge to a self-referential fixed point in finitely many steps, sequence evaluation must also be conducted within a \textbf{finite horizon}. The system cannot consider ``the infinitely distant future,'' but can only consider states and sequences reachable within a finite number of steps.

\begin{definition}[Finite Horizon]
The sequence-evaluation horizon \(H \in \mathbb{N}^+\) of the cognitive system is a positive integer expressing the maximum number of forward state-refresh steps the system considers when performing behavioral planning. Any future state beyond \(H\) steps contributes zero to the planning (or, equivalently, is truncated at the horizon boundary).

The concrete value of the finite horizon \(H\) is determined by the system's cognitive resources---including the patience curve (the maximum number of steps the system is willing to invest in the current demand), the capacity of the attention pool (the number of candidate sequences the system can simultaneously maintain), and the stability of the environment (in highly uncertain environments, distant predictions rapidly lose reliability). \(H\) is not fixed and immutable; it can be dynamically adjusted with context, demand, and system state. But at any given moment, \(H\) is a determinate, finite integer.
\end{definition}

\subsubsection{Introducing Probabilistic Branching}

Within the finite horizon, the system must consider not only ``which sequences are possible,'' but also ``the probability with which each possible sequence occurs.''

\begin{definition}[State-Transition Probability]
For any state \(s \in \mathcal{I}\) and any action \(a \in \mathcal{A}\), let \(P(s' \mid s, a)\) be the probability that, after executing action \(a\), the system enters state \(s'\) at the next moment via sensor refresh. These probabilities satisfy:
\[
\sum_{s' \in \mathcal{I}} P(s' \mid s, a) = 1
\]
That is, for a given state-action pair, the probabilities of all possible successor states sum to unity.

The source of state-transition probabilities is the system's experiential set. Whenever the system executes action \(a\) in state \(s\) and observes the successor state \(s'\), this transition is recorded in the state-association network. After many attempts, the system forms a statistical estimate of \(P(s' \mid s, a)\). This estimate can be a precise probability value (if the system has sufficient count data), or a rough approximation (if experience is limited). For never-experienced state-action pairs, the system can generalize from the transition experience of similar states, or assign a default prior probability (such as a uniform distribution or a maximum-entropy distribution).
\end{definition}

\subsubsection{Expected Efficiency and the Recursion Formula}

In a probabilistic environment, the system cannot determinately know which successor state a given action will lead to. Hence, the evaluative value of an action must be its \textbf{expected efficiency}---the efficiency values of all possible successor states, weighted by their transition probabilities.

\begin{definition}[Expected Efficiency]
The expected efficiency \(Q(s, a)\) of executing action \(a\) in state \(s\) is the weighted average of the efficiency values of all possible successor states of that action (including both the independent efficiency value of the successor state itself, and the efficiency values that may be further accumulated within the remaining horizon).

For a finite horizon \(H\), let \(V_h(s)\) be the maximum expected cumulative efficiency obtainable starting from state \(s\) within the remaining \(h\) steps. The recursion formula is:
\[
V_h(s) = \max_{a \in \mathcal{A}} \sum_{s' \in \mathcal{I}} P(s' \mid s, a) \left[ \eta(s') + V_{h-1}(s') \right]
\]
where:
\begin{itemize}
\item \(h \in \{1, 2, \ldots, H\}\) is the number of remaining steps,
\item \(\eta(s')\) is the independent efficiency value of state \(s'\) (Definition~5.6),
\item \(V_{h-1}(s')\) is the maximum expected cumulative efficiency obtainable from \(s'\) within the remaining \(h-1\) steps,
\item \(\eta(s') + V_{h-1}(s')\) is the \textbf{cumulative efficiency} that the system can obtain, if it enters state \(s'\) via action \(a\), within the remaining steps---including both the independent efficiency value of \(s'\) itself and the extra efficiency values obtainable in subsequent steps.
\end{itemize}

The boundary condition is \(V_0(s) = 0\) for all \(s\)---when the number of remaining steps is zero, the system cannot obtain any extra efficiency value, and planning terminates.

The core logic of the recursion formula is: the system, when selecting an action at each step, does not only consider the immediate effect of that action (\(\eta(s')\)), but considers the total efficiency of the entire subsequent sequence triggered by that action (\(\eta(s') + V_{h-1}(s')\)). This guarantees the forward-looking nature of planning---an action that does not excel in immediate efficiency may still be the optimal choice if it opens the door to a high-efficiency region (i.e., leads to a successor state \(s'\) with a high \(V_{h-1}(s')\)).
\end{definition}

\subsubsection{Action Selection and the Optimal Strategy Tree}

After the recursive computation is completed, the optimal action selection of the system at the current state \(s_0\) is:
\[
a^* = \arg\max_{a \in \mathcal{A}} \sum_{s' \in \mathcal{I}} P(s' \mid s_0, a) \left[ \eta(s') + V_{H-1}(s') \right]
\]
That is, the system selects the action that maximizes the \(H\)-step expected cumulative efficiency.

From the current state, the recursion formula not only yields the optimal action \(a^*\), but also implicitly defines an \textbf{optimal strategy tree}: for every possible successor-state branch, the system knows what action should be selected in that state to maximize the expected cumulative efficiency within the remaining horizon. This strategy tree is not explicitly stored---it is implicit in the computational process of the recursion formula. When the system actually executes \(a^*\) and observes the concrete successor state \(s'\), it can re-run the recursive computation (starting from \(s'\), with horizon \(H-1\) or updated to \(H\)) to adapt to the new information.

This ``rolling planning'' mechanism reflects the real-time adaptability of the cognitive system in uncertain environments. The system does not need to formulate, before acting, all concrete steps from the present to the horizon boundary---it only needs to know the optimal first-step action in the current state. After executing this step, new sensor input may bring new information (e.g., the actual successor state entered differs from what was expected), and the system can replan on the basis of the updated known-state set. This is fully consistent with the arguments of Chapter~2 that ``the known cannot derive the unknown'' and ``the sensor is the sole channel for transforming the unknown into the known''---the system cannot pre-know all future states, but can only dynamically adjust its planning at each step according to the actual perceptual refresh.

\subsubsection{Connection with Established Concepts}

The finite-horizon probabilistic planning model forms tight connections with multiple core concepts established earlier in this paper.

\textbf{Connection with state refresh:} Each state-refresh step corresponds to one step \(h \to h-1\) in the recursion formula. The system starts from the current state, executes an action, the sensor refreshes, and the system enters a new state---this corresponds, in the recursion formula, to the transition from \(V_h(s)\) to \(\eta(s') + V_{h-1}(s')\).

\textbf{Connection with the known-state set:} The transition probabilities \(P(s' \mid s, a)\) and the independent efficiency values \(\eta(s)\) on which the recursive computation depends all originate from the system's known-state set \(K_t\)---i.e., the accumulation of all sensor outputs up to the current moment. The system needs no external ``world model'' to supply these parameters; it only needs to retrieve its own experiential memory.

\textbf{Connection with demand weights and efficiency polylines:} The computation of \(\eta(s)\) depends on the demand weights \(w_d\) and the efficiency polylines \(f_d\) (Definition~5.6). Demand weights are anchored in somatic feedback; the shapes of efficiency polylines reflect the system's patience and urgency. Finite-horizon probabilistic planning thus unifies, within a single complete mathematical framework, the value signals at the somatic level, the sequence memory at the experiential level, and the planning inference at the cognitive level.

\textbf{Connection with semi-reference chains and unknown states:} When transition probabilities cannot be determined from experience (i.e., the consequences of certain state-action pairs are completely unknown), the system can insert semi-reference chain markers among the candidate successor states. Semi-reference chain markers carry no concrete efficiency value, but trigger the system's exploratory behavior---in finite-horizon planning, exploratory actions may be assigned an extra ``information value'' bonus (the detailed unfolding of this mechanism belongs to subsequent chapters).

% ==========================================
% CHAPTER 5.5: EFFICIENCY POLYLINE SHAPES AND COGNITIVE PREFERENCES
% ==========================================

\subsection{Efficiency Polyline Shapes and the Unified Explanation of Cognitive Preferences}

Section~5.3 introduced the efficiency polyline \(f_d(L)\) as a general functional form characterizing the relationship between distance and demand value, and noted that the concrete shape of the polyline---its rate of decrease and degree of flatness---can vary with demand type, individual experience, and situational factors. Section~5.4 established the finite-horizon probabilistic planning model, in which the independent efficiency value of a state \(\eta(s) = \sum w_d \cdot f_d(L(s, d))\) is the core input to the recursion formula. However, we have not yet systematically investigated a crucial question: \textbf{how do different shapes of the efficiency polyline affect the behavioral preferences of the cognitive system?}

This section answers that question. We will argue that the shape of the efficiency polyline---in particular, decreasing versus increasing polylines---is the unified mathematical root distinguishing ``short-sighted'' from ``far-sighted'' behavior. Decreasing polylines necessarily lead to the sunk-cost effect and a preference for immediate rewards; increasing polylines produce delayed gratification and persistent pursuit of distant goals. This framework not only unifies the explanation of multiple seemingly distinct cognitive phenomena, but also reveals the deep reason why behavior may be difficult for an external observer to understand: the efficiency polyline parameters of the observer and the actor may be fundamentally different.

\subsubsection{Cognitive Consequences of Decreasing Polylines}

\begin{definition}[Decreasing Efficiency Polyline]
An efficiency polyline \(f_d(L)\) is \textbf{decreasing} if and only if, for any \(L_1 < L_2\), \(f_d(L_1) > f_d(L_2)\)---i.e., the efficiency multiplier strictly decreases as distance increases. Axiom~5.3 in Section~5.3 already required the efficiency polyline to satisfy non-increase; a decreasing polyline is its strict version.
\end{definition}

A decreasing polyline is the most natural and most common shape of the efficiency polyline. It expresses a basic cognitive fact: \textbf{the closer to the demand state, the higher the value of the current state.} This is fully consistent, mathematically, with the efficiency preference argued in Chapter~4---that shorter sequences are better. The feedback efficiency value \(\eta(s, d) = w_d \cdot f_d(L(s, d))\) under a decreasing polyline attains its maximum at the demand state and decays with increasing distance.

Decreasing polylines produce two important cognitive consequences.

\paragraph{(i) A Generative Explanation of the Sunk-Cost Effect}

The sunk-cost effect refers to the tendency of decision-makers, having already invested substantial resources (time, effort, money), to persist with a current course of action even when switching to another course would be objectively better. Classical decision theory typically treats the sunk-cost effect as a cognitive bias---a deviation from ``rational choice.'' Predictive Set Theory offers a radically different explanation: \textbf{the sunk-cost effect is a natural consequence of demand maximization under a decreasing efficiency polyline, not a bias.}

The argument is as follows. Suppose the system, at time \(t_0\), starts from state \(s_0\) and selects a path \(S_1\) toward the demand state \(d\). After \(k\) steps, the system is in state \(s_k\), with a remaining distance of \(L_1\) steps to \(d\). At this point, the system discovers that there exists another path \(S_2\) that, starting from \(s_k\), has a remaining distance of \(L_2\) steps to \(d\), with \(L_2 < L_1\)---i.e., switching paths could reach the goal faster.

In a pure ``switch vs.\ persist'' comparison that only considers the remaining distance from the current state to the goal, the system ought to switch. But this overlooks a crucial factor: \textbf{the \(k\) steps already taken have changed the position of the current state.} Under a decreasing polyline, the current state \(s_k\) (located on path \(S_1\)) is at distance \(L_1\) from the demand state \(d\), and its efficiency value is \(w_d \cdot f_d(L_1)\). After switching to path \(S_2\), although the remaining distance from \(s_k\) is shorter, the system may, during the switching process, first need to ``backtrack'' or ``detour''---i.e., the switching action itself may add extra steps.

More fundamentally: the \(k\) steps already invested are irreversible---they have already brought the system to \(s_k\), and \(s_k\), under a decreasing polyline, already has a higher efficiency value than \(s_0\) (since \(L_1 < L_0\), we have \(f_d(L_1) > f_d(L_0)\)). Abandoning \(s_k\) means giving up the high efficiency value already attained, and starting anew from a state that may have a lower efficiency value. Even if switching could ultimately reach the goal faster, the ``efficiency-value loss'' of the switching itself must be counted into the total cost.

Under the framework of the finite-horizon recursion formula (Section~5.4), the system compares:

\begin{itemize}
\item Expected cumulative efficiency of persisting with \(S_1\): \(\eta(s_k)\) + subsequent cumulative efficiency along \(S_1\).
\item Expected cumulative efficiency of switching to \(S_2\): the possible efficiency loss brought by the switching action + cumulative efficiency along \(S_2\) from the post-switch state.
\end{itemize}

When the efficiency value already attained on the persistence path is sufficiently large (i.e., \(f_d(L_1)\) is sufficiently high), the net benefit of switching may be negative---not because the system is ``irrational,'' but because the steps already invested have genuinely altered the value gradient of the state. \textbf{Sunk costs are not irrationally ``remembered''; they are materially sedimented in the position of the current state.} You cannot return to the starting point and begin again---you are here, and here, under a decreasing polyline, is already more valuable than the starting point.

\paragraph{(ii) Short-Sighted Behavior and Immediate Preference}

Another direct consequence of decreasing polylines is \textbf{short-sighted behavior}---the system prefers actions that can rapidly shorten the distance in the near term, over actions that may temporarily lengthen the distance but ultimately bring greater rewards via a detour.

The mathematical root of this preference lies in the convexity (or, more precisely, the monotonicity of the discrete differences) of decreasing polylines. For typical decreasing polylines (such as \(f_d(L) = e^{-\lambda L}\) or \(f_d(L) = 1/(L+1)\)), the efficiency difference between adjacent distance points is larger at closer distances:
\[
f_d(L) - f_d(L+1) > f_d(L+1) - f_d(L+2) \quad \text{for most } L
\]
This means that, close to the demand state, the efficiency gain brought by shortening the distance by one step is greater than the gain brought by shortening the distance by one step farther away. The system is thus strongly incentivized to prioritize actions that can \textbf{immediately} shorten the distance, even if these actions may lead to a suboptimal path from a global perspective.

For example, a hungry system (with the demand ``obtain food'') faced with the option ``first walk a segment away from food to get a tool, then use the tool to obtain food more efficiently'' will, under a decreasing polyline, find the efficiency loss of the step ``walk away from food'' to appear especially large---because the current state may already be not far from food, and moving farther away would sharply reduce the efficiency value. The system may thus select the immediate-gratification path of ``obtain a small amount of food with bare hands, forgoing the tool,'' abandoning the better long-term strategy. This is the generative mechanism of short-sightedness: it is not that the system ``cannot see'' the long-term benefit, but that the decreasing shape of the efficiency polyline makes the efficiency gain and loss of short-term steps numerically overwhelm the far-term benefit.

PST's efficiency polyline model provides a unified mathematical explanation for the classical phenomenon of temporal discounting in behavioral economics. Ainslie (1975) discovered that humans and animals discount delayed rewards not according to a standard exponential function, but in a hyperbolic fashion---the discount rate is extremely high in the near term and flattens out in the far term. Laibson (1997) proposed a quasi-hyperbolic discounting model (the \(\beta\)-\(\delta\) model) to characterize this phenomenon. Under the PST framework, both hyperbolic discounting and quasi-hyperbolic discounting can be viewed as products of particular parameterizations of the efficiency polyline. The hyperbolic form \(f_d(L) = 1/(1 + \alpha L)\) naturally produces a discount pattern that is steep in the near term and flat in the far term, while exponential decay \(f_d(L) = e^{-\lambda L}\) corresponds to the constant discount rate of standard economic models. PST's contribution lies in revealing that these discount forms are not behavioral biases, but value-gradient shapes naturally adopted by the cognitive system, in its sequential operation, to maximize demand satisfaction. Different efficiency polyline parameters correspond to different survival strategies---extremely steep decreasing polylines are needed in dangerous environments (immediate danger avoidance), while flat or even increasing polylines can be afforded in safe environments (long-term investment).

At the neuroeconomic level, McClure et al. (2004) discovered that immediate rewards and delayed rewards activate distinct brain regions---the limbic system prefers immediate rewards, while the prefrontal cortex participates in long-term planning. PST provides a computational-level unification for this dual-system model: the two systems do not necessarily correspond to two independent modules, but may reflect differential parameter settings of efficiency polylines across different demand dimensions. The steep polyline of the limbic system ensures rapid responses to biological emergency demands; the flat polyline of the prefrontal cortex supports sustained pursuit of social and abstract goals. Both operate under the exact same underlying computation within the PST framework---demand weights, sequence distance, efficiency polylines---differing only in the shape parameters of the polylines.

\subsubsection{Increasing Polylines and Far-Sighted Behavior}

\begin{definition}[Increasing Efficiency Polyline]
An efficiency polyline \(f_d(L)\) is \textbf{increasing} if and only if, for any \(L_1 < L_2\), \(f_d(L_1) < f_d(L_2)\)---i.e., the efficiency multiplier increases as distance increases.
\end{definition}

An increasing polyline may seem, on the surface, counterintuitive: the farther from the goal, the higher the efficiency value? But in cognitive reality, an increasing polyline characterizes an important and common mentality: \textbf{persistent pursuit of a grand ideal.} Under an increasing polyline, the demand state itself (distance \(0\)) has the smallest efficiency multiplier (possibly zero or near zero), while states at extremely large distances have the highest efficiency multipliers.

Cognitive instances of increasing polylines include:

\begin{itemize}
\item \textbf{The strivings of an idealist:} A person working for ``world peace'' may not feel frustrated by the fact that world peace is indefinitely remote---on the contrary, it is precisely the grandeur and remoteness of the goal that endows the present action with lofty meaning. The farther the distance, the purer the ideal, and the stronger the sense of value in the action.
\item \textbf{Ascetic delayed gratification:} Certain religious or philosophical traditions extol ``asceticism''---the active choice of a hard life far from mundane comforts (positive-feedback demands). Under an increasing polyline, being far from comfort (maximal distance) is evaluated as highly valuable, because the hardship itself is seen as the necessary path toward a higher spiritual state.
\item \textbf{Artistic creation and basic scientific research:} Great works of art or scientific discoveries often require decades of sustained investment, with almost no visible immediate reward. An increasing polyline endows ``persistence through the long darkness'' with positive value in itself---not because there is light at the end of the darkness, but because the very act of traversing the darkness has been invested with meaning.
\end{itemize}

Behavioral planning under an increasing polyline is radically different from that under a decreasing polyline. The system actively selects detour paths that temporarily lengthen the distance but can enter higher-efficiency regions. It does not fear short-term ``regression''---because under an increasing polyline, ``regression'' (lengthening the distance) is not a loss, but a gain. This enables the system to execute extremely long-term plans, sacrificing present benefits for greater future value.

Increasing polylines thus also explain why the behavior of some persons driven by intense beliefs may be difficult for an external observer to understand. An external observer will typically default to a decreasing polyline (or standard exponential discounting) to evaluate another's behavior---because a decreasing polyline is the natural extension of biological somatic feedback (hunger, pain): the hungrier you are, the more you want to eat; the more pain you are in, the more you want to escape. When an observer sees an actor apparently actively moving away from an obvious positive-feedback goal, the observer judges this as ``irrational'' or ``self-destructive.'' But the actor may internally be operating under an increasing efficiency polyline---within his or her value landscape, moving away from the goal is precisely the gain. The root of the incomprehensibility of behavior is not a lack of rationality, but \textbf{the invisibility of efficiency polyline parameters}.

\subsubsection{Flat Polylines and Random Walks}

Between decreasing and increasing polylines lies a degenerate intermediate case: the flat polyline.

\begin{definition}[Flat Efficiency Polyline]
An efficiency polyline \(f_d(L)\) is \textbf{flat} if and only if, for any \(L_1, L_2\), \(f_d(L_1) = f_d(L_2) = c\) (a constant). Under the normalization axiom, \(c = 1\).
\end{definition}

A flat polyline means that the system is completely indifferent to distance. As long as the demand state can ultimately be reached, no matter how long the path, no matter how many intermediate steps, the system assigns the same efficiency evaluation. Under a flat polyline, demand maximization degenerates into a binary judgment of ``reachable or not''---all reachable paths are equivalent, and all unreachable paths are equivalent. The system has no efficiency-based preference whatsoever in selecting actions, and its behavioral output will exhibit the character of a random walk (arbitrarily selecting among reachable paths).

Flat polylines may be relatively rare in pure form, but as a limiting case they reveal the fundamental influence of the efficiency polyline's shape on behavioral determinacy: \textbf{the steeper the polyline (whether decreasing or increasing), the stronger the directionality of behavior; the flatter the polyline, the greater the randomness of behavior.} This provides a mathematical explanation for ``decision hesitancy'' and ``preference-less states'': when the efficiency polylines of multiple demands compete with one another and tend, overall, toward flatness, the system's behavioral selection becomes highly uncertain.

\subsubsection{Dynamic Switching of Efficiency Polylines and Meta-Cognitive Regulation}

The cognitive system is not locked into a single polyline shape. One and the same system can adopt different polylines for different demands, and can even adopt different polylines for the same demand at different times. This capacity for dynamic switching is an important manifestation of cognitive flexibility.

\begin{itemize}
\item \textbf{Contextual switching:} In safe environments, the system may adopt relatively flat decreasing polylines, or even increasing polylines, for long-term planning and exploration; as soon as a danger signal is detected, the system rapidly switches the polylines of safety-critical demands to an extremely steep decreasing form, to produce immediate avoidance behavior.
\item \textbf{Developmental stages:} Young individuals may lean more toward decreasing polylines (immediate gratification); with the maturation of the prefrontal cortex and the process of socialization, they gradually develop the capacity to adopt increasing polylines for certain demands (such as academic achievement).
\item \textbf{Cultural shaping:} Different cultural traditions transmit, through language, stories, and norms, information about ``which demands are worth patiently waiting for.'' These cultural memes are essentially adjusting the efficiency polyline parameters of specific demands.
\end{itemize}

The plasticity of efficiency polylines means that the cognitive system does not need to evolve separate decision-making modules for ``short-sightedness'' and ``far-sightedness.'' One and the same basic architecture---demand functions, sequence distances, efficiency polylines---by merely adjusting the shape parameters of a few polylines, can produce the entire behavioral spectrum from extreme impulsivity to extreme self-control. This design of replacing module proliferation with parameterization reflects the consistent minimalist principle of Predictive Set Theory: explain the most phenomena with the fewest mechanisms.

\section*{Appendix: The Complete Derivation Chain}

This appendix presents, in linear order, the complete generative derivation of Predictive Set Theory from the lowest-level axioms to the highest-level cognitive functions. Each step of the derivation is strictly annotated with its logical premises and the chapter in which it belongs, to display the overall logical structure of the theoretical system.

\subsection*{A.1 The Grounding of Consistency}

\textbf{A.1.1 Establishment of the Consistency Requirement} (Chapter~1, Section~1.1)

\textit{Premise:} A cognitive system must be capable of generating effective decisions (definitional premise).

\textit{Derivation:}
\begin{itemize}
\item A contradictory knowledge network deprives the value-evaluation function of a determinate output; the system cannot make a non-contradictory choice among multiple candidate actions.
\item Hence, a cognitive system must satisfy the consistency requirement---any detected contradiction must be processed and resolved by the tag-management mechanism.
\end{itemize}

\textit{Status:} The consistency requirement is an architectural prerequisite, the transcendental formal condition that makes all subsequent derivations possible.

\textbf{A.1.2 Reference Chain Default Self-Reference} (Chapter~1, Section~1.2)

\textit{Premise:} The consistency requirement; the definition of the reference operation \(R(x)\).

\textit{Derivation:}
\begin{itemize}
\item If a reference chain does not converge to a self-referential fixed point, it is either infinitely non-repetitively divergent (the cognitive interrogation can never be completed) or cyclically divergent (equivalent to a non-convergent infinite series).
\item Hence, a cognitive system must presuppose that a reference chain converges to a self-referential fixed point in finitely many steps.
\end{itemize}

\textit{Conclusion:} Identity (\(x = x\)) is the absolute precondition for any cognitive computation to be completable. This precondition is enforced, at the sensor level, by the identity operation \(S(i) = i\).

\subsection*{A.2 Axiomatic Definition of the Cognitive Agent}

\textbf{A.2.1 Unknowability of the External World} (Chapter~2, Section~2.1)

\textit{Stipulation:} The external-object function \(W\) and external objects \(O\) are not directly accessible and not computationally manipulable by the cognitive agent.

\textit{Function:} To avoid the circularity predicament of representationalism and to demarcate the internal boundary of the cognitive system.

\textbf{A.2.2 The Sensor's Identity Operation} (Chapter~2, Section~2.2)

\textit{Premise:} The unknowability of the external world.

\textit{Derivation:}
\begin{itemize}
\item All that the system can confirm is that ``the internal object \(i\) has been presented.''
\item The existence or non-existence of an external object is indiscernible to the system's internals.
\item In the limit case where the external object does not exist, the sensor can still output an internal object.
\item Hence, the effective behavior of the sensor is equivalent to the identity function \(S(i) = i\).
\end{itemize}

\textit{Conclusion:} The sensor is the sole independent variable in the cognitive architecture; \(S(i) = i\) is the most primitive identity operation.

\textbf{A.2.3 Non-Independent-Variability of the Action-Planning Function} (Chapter~2, Section~2.3)

\textit{Premise:} The sensor is the sole independent variable.

\textit{Derivation:}
\begin{itemize}
\item If the action-planning function \(f\) were also an independent variable, the system would possess two independent information sources; the causal boundary between behavioral output and perceptual input would be blurred, and planning would become impossible.
\item Hence, \(f\) must be a conditioned dependent variable: \(A_t = f(K_t)\), with its output strictly dependent on the known-state set.
\end{itemize}

\textbf{A.2.4 The Asymmetry Theorem of Known and Unknown States} (Chapter~2, Section~2.4)

\textit{Premises:} The sensor's identity operation; the non-independent-variability of the action-planning function.

\textit{Derivation:}
\begin{itemize}
\item \textbf{The known cannot derive the unknown:} If there existed a \(g\) such that \(g(K_t) \in U_t\), the unknown would become known at the instant of computation, a contradiction.
\item \textbf{The unknown can be transformed into the known:} Sensor refresh is the sole channel of transformation---\(K_{t+1} = K_t \cup \{i_{t+1}\}\).
\end{itemize}

\textit{Corollaries:} The action-planning function is enclosed within the known-state set; prediction is perpetually fallible; the sensor is the sole entry point for new information.

\subsection*{A.3 The Construction of State Refresh and Sequences}

\textbf{A.3.1 Operationalization of State Refresh} (Chapter~3, Sections~3.1--3.2)

\textit{Premises:} The known-state set \(K_t\); the register object \(E_t\); the comparison operation.

\textit{Derivation:}
\begin{itemize}
\item When the sensor's new output \(i_{t+1} \neq E_t\), state refresh is executed: write to sequence, update \(K_{t+1} = K_t \cup \{i_{t+1}\}\), update \(E_{t+1} = i_{t+1}\).
\item Cognitive states are necessarily discrete---a consequence of the binary logic of the comparison operation.
\end{itemize}

\textbf{A.3.2 Distinction between the Known-State Sequence and the State Sequence} (Chapter~3, Section~3.3)

\textit{Premise:} The operational definition of state refresh.

\textit{Derivation:}
\begin{itemize}
\item The \textbf{known-state sequence} \(\mathcal{K} = (K_0, K_1, \ldots)\) has set-inclusion relations \(K_0 \subset K_1 \subset \cdots\) that are strictly isomorphic to von Neumann's construction of the ordinals.
\item The \textbf{state sequence} \(\mathcal{S}\) is the sequence of contents written at each state refresh; its order is supplied by the inclusion relations of the known-state sequence.
\item The known-state sequence provides time stamps for the state sequence; the state sequence itself does not possess the capacity for ``empty-set-construction-style correspondence.''
\end{itemize}

\textbf{A.3.3 The Set-Difference Definition of the Relative Unknown State} (Chapter~3, Section~3.4)

\textit{Premise:} The inclusion relations of the known-state sequence.

\textit{Derivation:}
\begin{itemize}
\item \(U(K_b \mid K_a) = K_b \setminus K_a\)---the unknown is a set-difference relation between known-state sets.
\item Cross-agent unknown states are incomparable---the inclusion relations between the known-state sequences of different agents cannot be confirmed.
\item Distinction between virtual known-state sequences (simulations of others, carrying time stamps) and virtual state sequences (internal inferences, not carrying time stamps).
\end{itemize}

\subsection*{A.4 The Generative Derivation of Demand}

\textbf{A.4.1 The Minimal Demand Constraint} (Chapter~4, Sections~4.1--4.2)

\textit{Premise:} The action-planning function \(f\) is a deterministic mapping.

\textit{Derivation:}
\begin{itemize}
\item For any \(f\), one can construct \(F_f(K, A) = -\| A - f(K) \|^2\), making \(f(K)\) the unique maximum point.
\item Hence, any action-planning function necessarily embeds a minimal demand constraint---demand is not an external addition, but a mathematical inevitability of the function's existence.
\end{itemize}

\textbf{A.4.2 Lifting of the Pure-Computation Restriction} (Chapter~4, supplements to Section~4.2)

\textit{Premises:} The sensor can input limitation information of the execution function; the minimal demand constraint.

\textit{Derivation:}
\begin{itemize}
\item If sensor input contains information that ``behavior \(A\) is infeasible,'' the singleton choice set \(\{f(K)\}\) of pure computation may degenerate (the sole candidate is vetoed).
\item If pure computation attempts to cope via meta-computation, it falls into an infinite-hierarchy predicament.
\item Comparison operations demote limitation information to ordinary states, evaluate demand values based on experience, and require no meta-language.
\item Lifting the pure-computation restriction expands the choice set, constructing a larger attainable maximum of demand.
\end{itemize}

\textbf{A.4.3 Separation of Demand and Process} (Chapter~4, Sections~4.3--4.4; Chapter~5, Section~5.1)

\textit{Premise:} The outcome-directedness of demand (Axiom~\ref{axiom:outcome}).

\textit{Derivation:}
\begin{itemize}
\item Demand constrains only the final output, not the process.
\item Different sequences in the optional space of processes are equivalent in result; the sole differentiating dimension is sequence length---i.e., efficiency.
\item Demand in sequential operation inevitably introduces the efficiency dimension.
\item The efficiency dimension leads to the expansion of the choice set (sequences of different lengths differ in demand value), which in turn inevitably triggers comparison.
\end{itemize}

\textbf{A.4.4 Generative Definition of Comparison} (Chapter~4, Section~4.3)

\textit{Premise:} The multi-element nature of the choice set.

\textit{Derivation:}
\begin{itemize}
\item Degenerate optimization (choice set is a singleton): formally an optimization, substantively devoid of comparison.
\item Non-degenerate optimization (choice set is multi-element): comparison is the inevitable operational form of demand maximization.
\item Generative conditions of comparison: the diversity condition (the choice set has at least two elements) and the discriminability condition (the demand function is not constant on the choice set).
\end{itemize}

\subsection*{A.5 Quantification of Feedback and Planning}

\textbf{A.5.1 Sequence-Efficiency Definition of Feedback} (Chapter~5, Section~5.2)

\textit{Premises:} Efficiency as the sole evaluative dimension of process; demand weights \(w_d\).

\textit{Derivation:}
\begin{itemize}
\item Initial feedback value: \(\eta_0(s, d) = 1 / (L(s, d) + 1)\).
\item Feedback efficiency value: \(\eta(s, d) = w_d / (L(s, d) + 1)\).
\item Positive and negative feedback on the same path linearly superpose and can cancel each other out.
\end{itemize}

\textbf{A.5.2 Efficiency Polylines} (Chapter~5, Section~5.3)

\textit{Premises:} The core logic of the feedback efficiency value; different demands may require different rates of decay.

\textit{Derivation:}
\begin{itemize}
\item Efficiency polyline \(f_d(L)\): satisfies normalization (\(f_d(0) = 1\)), monotonic decrease (\(L_1 < L_2 \Rightarrow f_d(L_1) \geq f_d(L_2)\)), and asymptotic approach to zero (\(\lim_{L \to \infty} f_d(L) = 0\)).
\item Independent efficiency value of a state: \(\eta(s) = \sum w_d \cdot f_d(L(s, d))\)---depends only on the state itself, independent of subsequent paths.
\end{itemize}

\textbf{A.5.3 Finite-Horizon Probabilistic Planning} (Chapter~5, Section~5.4)

\textit{Premises:} Sequences cannot be traced backward unconditionally (isomorphic to the convergence of infinite reference chains); empirical estimation of transition probabilities.

\textit{Derivation:}
\begin{itemize}
\item Finite horizon \(H\): the system considers, at most, future states within \(H\) steps.
\item Recursion formula: \(V_h(s) = \max_a \sum_{s'} P(s' \mid s, a) [\eta(s') + V_{h-1}(s')]\), with boundary \(V_0(s) = 0\).
\item Optimal action selection: \(a^* = \arg\max_a \sum_{s'} P(s' \mid s_0, a) [\eta(s') + V_{H-1}(s')]\).
\item Rolling planning: after executing the first step, replan based on new perceptual information.
\end{itemize}

\textbf{A.5.4 Efficiency Polyline Shapes and Cognitive Preferences} (Chapter~5, Section~5.5)

\textit{Premise:} The parametric forms of efficiency polylines.

\textit{Derivation:}
\begin{itemize}
\item Decreasing polyline \(\to\) short-sighted behavior, sunk-cost effect, immediate preference.
\item Increasing polyline \(\to\) far-sighted behavior, delayed gratification, idealistic persistence.
\item Flat polyline \(\to\) absence of preference, random behavior.
\item The root of the incomprehensibility of behavior: the efficiency polyline parameters of the observer and the actor may be fundamentally different.
\end{itemize}

\subsection*{A.6 Summary of the Complete Derivation Chain}

{\small
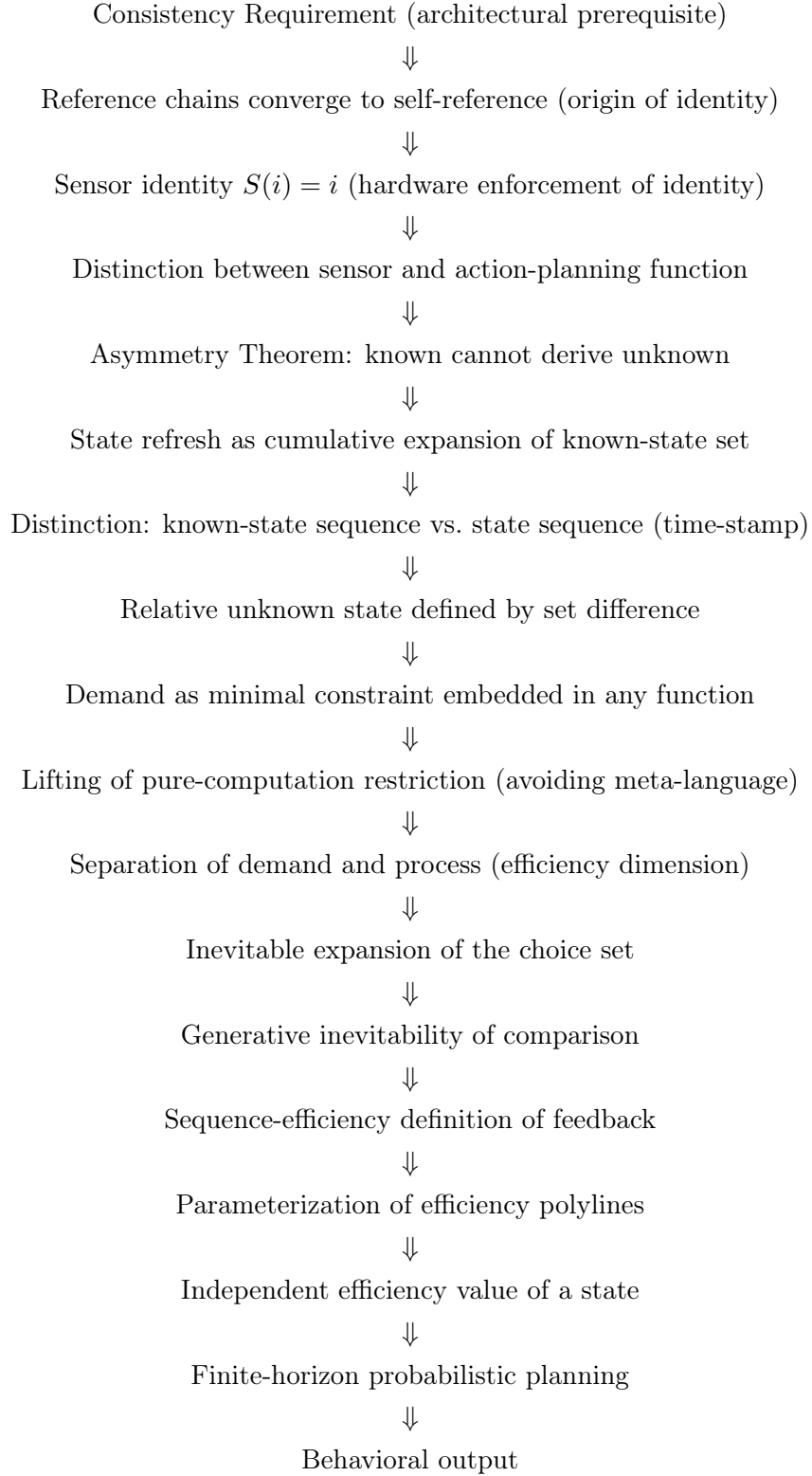
\[
\begin{array}{c}
\text{Consistency Requirement (architectural prerequisite)} \\
\Downarrow \\
\text{Reference chains converge to self-reference (origin of identity)} \\
\Downarrow \\
\text{Sensor identity } S(i) = i \text{ (hardware enforcement of identity)} \\
\Downarrow \\
\text{Distinction between sensor and action-planning function} \\
\Downarrow \\
\text{Asymmetry Theorem: known cannot derive unknown} \\
\Downarrow \\
\text{State refresh as cumulative expansion of known-state set} \\
\Downarrow \\
\text{Distinction: known-state sequence vs.\ state sequence (time-stamp)} \\
\Downarrow \\
\text{Relative unknown state defined by set difference} \\
\Downarrow \\
\text{Demand as minimal constraint embedded in any function} \\
\Downarrow \\
\text{Lifting of pure-computation restriction (avoiding meta-language)} \\
\Downarrow \\
\text{Separation of demand and process (efficiency dimension)} \\
\Downarrow \\
\text{Inevitable expansion of the choice set} \\
\Downarrow \\
\text{Generative inevitability of comparison} \\
\Downarrow \\
\text{Sequence-efficiency definition of feedback} \\
\Downarrow \\
\text{Parameterization of efficiency polylines} \\
\Downarrow \\
\text{Independent efficiency value of a state} \\
\Downarrow \\
\text{Finite-horizon probabilistic planning} \\
\Downarrow \\
\text{Behavioral output}
\end{array}
\]
}

Every step in this chain is necessary and inescapable. From the lowest-level consistency requirement to the highest-level behavioral output, Predictive Set Theory derives the complete operational logic of the cognitive system, from perception to decision, with the fewest primitives (sensor, action-planning function, three fundamental forms of reference chains, demand function, efficiency polyline) and the fewest axioms (reference chain default self-reference, unknowability of the external world, non-independent-variability of the action-planning function, outcome-directedness of demand, normalization and monotonicity of the efficiency polyline).

\bibliography{references}

\end{document}